\documentclass{pkumelon}

\usepackage[utf8]{inputenc}
\usepackage{url}
\usepackage{amsfonts}
\usepackage{nicefrac}
\usepackage{xspace}
\usepackage{amsmath}
\usepackage{amsthm}
\usepackage{amssymb}
\usepackage{multicol}
\usepackage{makecell}
\usepackage{enumitem}
\usepackage{wrapfig}
\usepackage{dsfont}
\usepackage{epstopdf}
\usepackage{colortbl}
\usepackage{algorithm}
\usepackage{threeparttable}
\usepackage{mathtools}
\usepackage{float}
\usepackage{mathrsfs}
\usepackage{algpseudocode}
\usepackage{tabularray}
\usepackage{quiver}
\usepackage{diagbox}
\usepackage{needspace}

\renewcommand{\titlefont}{\fontsize{17}{20}\selectfont}
\renewcommand\affiliationformat[2][]{{\affiliationfont {\sffamily $^{#1}$#2}}}


\makeatletter
\renewcommand\authorformat[2][]{%
  \authorfont {\sffamily
    \mbox{\ifstrempty{#1}{#2}{#2$^{#1}$}}%
  }%
}
\makeatother

\newtcolorbox{limitationbox}{
    enhanced,
    breakable,
    colback=boxbg,
    frame hidden,
    boxrule=0pt,
    borderline west={1.2pt}{0pt}{reaslab},
    sharp corners,
    left=6pt,
    right=2pt,
    top=3pt,
    bottom=3pt,
    before skip=5pt,
    after skip=5pt
}
\newenvironment{limbox}{\begin{limitationbox}}{\end{limitationbox}}

\restylefloat{algorithm} 
\theoremstyle{plain}
\newtheorem{assumption}{Assumption}

\newtheorem{theorem}{Theorem}
\newtheorem{remark}{Remark}
\newtheorem{lemma}{Lemma}
\newtheorem{example}{Example}

\newtheorem{corollary}{Corollary}

\newcommand{\R}{\mathbb{R}}
\newcommand{\E}{\mathbb{E}}
\DeclareMathOperator{\tr}{tr}
\DeclareMathOperator{\argmin}{argmin}

\newcommand{\Phiobj}{\Phi}

\newcommand{\oursname}{\textnormal{\textbf{BROS}}}
\newcommand{\ours}{\oursname\ }
\newcommand{\ourslast}{\oursname}

\title{BROS: Bias-Corrected Randomized Subspaces for Memory-Efficient Single-Loop Bilevel Optimization}

\author[1,*]{Hengrui Zhang}
\author[2,*]{Boao Kong}
\author[2,*]{Engao Zhang}
\author[2,\P]{Kun Yuan}

\affiliation[1]{Sichuan University}
\affiliation[2]{Peking University}
\contribution[*]{Equal Contribution}
\contribution[\P]{Corresponding author}
\checkdata[Emails]{\email{2022141210007@stu.scu.edu.cn}, \email{kongboao@stu.pku.edu.cn}, \email{2300012403@stu.pku.edu.cn}, \email{kunyuan@pku.edu.cn}}

\abstract{
Stochastic bilevel optimization (SBO) has become a standard framework for hyperparameter learning, data reweighting, representation learning, and data-mixture optimization in deep learning. Existing exact single-loop SBO methods and memory-efficient surrogate SBO methods either create severe memory pressure for large lower-level neural networks or lack competitive convergence guarantees under standard assumptions. In this paper, we propose \ourslast, a memory-efficient single-loop SBO method with the same convergence rate order as exact single-loop SBO methods. \ours performs lower and auxiliary updates in randomized subspaces with a Rademacher bi-probe correction that recovers an unbiased Hessian-action estimator. We prove that \ours preserves the $\mathcal O(\varepsilon^{-2})$ sample complexity of MA-SOBA for finding an $\varepsilon$-stationary point under only standard assumptions. Experiments on hyper-data cleaning, data-mixture learning, hyper-representation learning, and ViT sample reweighting show that \ours reduces peak memory by up to $44.9\%$ while closely matching full-space baseline performance.
}

\begin{document}

\maketitle

\section{Introduction}
\label{sec:intro}

Stochastic bilevel optimization (SBO) has become an increasingly important framework in deep learning, with applications including hyper-cleaning and noisy-label reweighting~\citep{hao2024bilevel,yang2024tuning,tu2023learning,zheng2021meta}, hyperparameter optimization and meta-learning~\citep{franceschi2018bilevel,shaban2019truncated,chen2024optimal}, hyper-representation learning~\citep{huang2024optimal,yang2025first}, and data reweighting or data-selection for large-scale training~\citep{fan2023doge,xie2023doremi,shen2024seal,pan2025scalebio,wang2026draw}. Across these applications, the upper-level variable can take a variety of forms. The lower-level problem, however, consistently trains a neural network with multilayer matrix parameters. Based on this observation, this paper studies a more general multilayer matrix-parameterized SBO model, which covers the vector formulation used in previous works~\citep{ghadimi2018approximation,dagreou2022framework,chen2024optimal,chu2024spaba}:
\begin{subequations} \label{eq:intro-problem}
\begin{align}
  &\min_{x\in\mathbb{R}^{d_x}}\; \Phi(x) := f\bigl(x,\boldsymbol{Y}^*(x)\bigr), \quad \boldsymbol{Y}^*(x) \in \argmin_{\boldsymbol{Y}\in\mathcal{Y}} g(x,\boldsymbol{Y}), \label{eq:multi_layer_bo}\\
  &\text{where}\;\; \mathcal{Y}:=\Bigl\{\boldsymbol{Y}=(Y_1,\ldots,Y_L)\,:\, Y_\ell\in\mathbb R^{m_\ell\times n_\ell},\ \ell\in[L]\Bigr\}. \label{eq:domain_y}
\end{align}
\end{subequations}
Here $x$ denotes the upper-level variable, and $\boldsymbol{Y}$ collects the trainable parameters of the lower-level network. We work in the stochastic setting where the objectives are population risks:
\begin{equation*}
f(x,\boldsymbol Y)=\mathbb E_{\omega_f\sim\mathcal D_f}[\ell_f(x,\boldsymbol Y;\omega_f)], \quad g(x,\boldsymbol Y)=\mathbb E_{\omega_g\sim\mathcal D_g}[\ell_g(x,\boldsymbol Y;\omega_g)],
\end{equation*}
and the algorithm accesses these objectives through stochastic oracles.

\textbf{Full-space hypergradient methods. }A large body of algorithms has been developed for SBO. A common route is to estimate the hypergradient $\nabla \Phi(x)$ either by approximating the inverse Hessian or by maintaining an auxiliary linear-system variable updated through Hessian-vector products (HVPs) and Jacobian-vector products (JVPs)~\citep{ghadimi2018approximation,khanduri2021near}. This route includes early stochastic bilevel methods~\citep{ghadimi2018approximation,hong2023two,ji2021bilevel} and recent single-loop methods such as SOBA/SABA~\citep{dagreou2022framework} and MA-SOBA~\citep{chen2024optimal}, which serve as the main baselines for our theoretical convergence rate comparison. Other variance-reduced or momentum-based baselines such as SPABA~\citep{chu2024spaba}, SUSTAIN~\citep{khanduri2021near}, and MRBO~\citep{yang2021provably} improve sample complexity but require finite-sum structure or stronger oracle assumptions, so they are not the primary convergence-rate baselines for our comparison here. These methods provide a clean hypergradient baseline, but face a severe limitation:

\begin{limbox}
\textbf{Limitation 1. }\emph{Full-space hypergradient methods remain memory-heavy for large networks.}
\end{limbox}

Although these methods need not explicitly form the full Hessian or Jacobian, their recursions still rely on full-space second-order information and auxiliary linear-system variables. When the lower-level variable is a large neural network with hundreds of millions or billions of parameters, the corresponding Hessian and Jacobian matrices are impossible to form or store, and even HVP-based full-space auxiliary recursions can impose prohibitive memory and computational costs.

\textbf{Memory-efficient surrogate methods. } Another way to reduce the cost is to move away from the standard hypergradient update. Penalty methods~\citep{mehra2021penalty,kwon2023penalty} replace the original bilevel problem with a penalty reformulation. Value-function-based fully first-order methods~\citep{kwon2023fully} avoid HVP and JVP. FdeHBO~\citep{yang2023achieving} uses finite differences to avoid explicit HVP and JVP. Zeroth-order methods~\citep{shirkavand2025bilevel} further reduce the need for full derivative information. These approaches can substantially reduce or avoid explicit Hessian/Jacobian-vector computations, but this scalability typically comes with one of three costs: no exact recovery of the original bilevel solution, slower stochastic complexity, or additional structural/regularity assumptions:

\begin{limbox}
\textbf{Limitation 2. }\emph{Surrogate methods theoretically lose the exact hypergradient baseline.}
\end{limbox}

Penalty methods optimize a finite-penalty surrogate rather than the original problem, with exact recovery requiring a limiting penalty regime and additional regularity conditions~\citep{mehra2021penalty,kwon2023penalty}. Value-function-based fully first-order methods typically have worse stochastic complexity orders than standard single-loop hypergradient methods~\citep{kwon2023fully}. Finite-difference methods such as FdeHBO~\citep{yang2023achieving} rely on an auxiliary projection and higher-order smoothness assumptions to control finite-difference errors. Zeroth-order bilevel methods often require many more function evaluations and stronger structural assumptions, such as low effective rank~\citep{shirkavand2025bilevel}. Thus, these approaches do not answer the question we study here: \emph{can one reduce memory usage while preserving the standard single-loop hypergradient update and its $\mathcal O(\varepsilon^{-2})$ sample complexity for finding an $\varepsilon$-stationary point?}

\textbf{Randomized-subspace large-model training. } Recent progress in single-level large-model training suggests a complementary path. Several studies have observed and exploited low-rank or low-dimensional structure in large-model gradients and updates, leading to memory-efficient methods that store and optimize only subspace variables. For example, LoRA~\citep{hu2022lora} constrains trainable updates to low-rank adapters, GaLore~\citep{zhao2024galore} projects gradient matrices into low-rank subspaces to reduce optimizer-state memory, and RSO~\citep{chen2025memory} decomposes LLM training into a sequence of lower-dimensional subproblems. Importantly for our setting, RSO shows that randomized subspace training can reduce activation and gradient memory while retaining the same convergence order as full-space training in single-level optimization~\citep{chen2025memory}. Motivated by this, we ask whether the lower-level and auxiliary-level training problem in SBO can also be solved in randomized subspaces.

Unfortunately, applying randomized subspace methods to SBO is highly nontrivial: unlike first-order RSO in single-level optimization, exact bilevel methods must compress and recover second-order information for the hypergradient. One may think of using an RSN-like strategy~\citep{gower2019rsn}: compute Hessian information inside the sampled subspace and lift it back to the original full space. This is sufficient for randomized Newton methods, whose analysis treats the sketched Newton step itself as a descent direction. However, in exact bilevel optimization, the Hessian determines the auxiliary linear-system used in the implicit-gradient formula. A biased lifted Hessian makes the auxiliary variable track the wrong system and biases the hypergradient.

To overcome this, we propose \ours (\textbf{B}ilevel optimization via \textbf{R}and\textbf{O}mized \textbf{S}ubspace), a memory-efficient single-loop stochastic bilevel method for the multilayer matrix-parameterized problem \eqref{eq:intro-problem}. \ours introduces a Rademacher-probe-based correction that recovers an unbiased hypergradient estimator while retaining the memory advantage of randomized subspace training and the convergence rate order of exact single-loop methods. The core contributions are summarized as follows:

\textbf{C1. Memory-efficient randomized-subspace bilevel recursion.} We develop a single-loop stochastic bilevel method that decomposes the lower-level neural-network training recursion into randomized lower-dimensional matrix subproblems. Since each subproblem operates in a reduced subspace, \ours lowers the dimensionality and memory costs of the lower gradients, auxiliary directions, activations, and HVP/JVP computations required by the lower-level recursion.

\textbf{C2. Unbiased hypergradient recovery with MA-SOBA-order convergence.} We identify the projection-induced bias in the lifted auxiliary Hessian action and remove it using a layerwise Rademacher bi-probe correction. The resulting estimator is conditionally unbiased for the full multilayer Hessian action without imposing additional structural assumptions. Under standard assumptions in SBO, this correction allows \ours to match the $\mathcal O(\varepsilon^{-2})$ sample complexity of the MA-SOBA backbone for finding an $\varepsilon$-stationary point.

\textbf{C3. Empirical validation on neural and language-model bilevel tasks.} We evaluate \ours on MNIST hyper-data cleaning, data-mixture learning with $280$M GPT-style proxy/main models, CIFAR-10 hyper-representation learning, and CIFAR-100 ViT sample reweighting. \ours closely matches full-space MA-SOBA on hyper-data cleaning and data-mixture learning, while reducing proxy-training peak memory by about $45\%$ relative to MA-SOBA and about $27\%$ relative to Penalty.

Table~\ref{tab:intro-positioning} summarizes the positioning of \ours by separating the oracle route, assumption regime, convergence order, and whether the lower-level and auxiliary computations are full-space. A detailed discussion of related work is provided in Appendix~\ref{app:related_work}.

\begin{table*}[t!]
\centering
\small
\setlength{\tabcolsep}{3pt}
\renewcommand{\arraystretch}{1.14}
\caption{
Rate- and assumption-oriented comparison of representative stochastic bilevel algorithms.
$N_\varepsilon$ denotes the oracle/iteration complexity for obtaining
$\mathbb E\|\nabla\Phi(\bar x)\|^2\le \varepsilon$, following each paper's own oracle model and stationarity convention.
}
\label{tab:intro-positioning}
\begin{threeparttable}
\begin{tabular}{lcccccccc}
\toprule
Algorithm
& S.L.$^\dagger$
& Oracle$^\ddagger$
& Assmp.$^\triangleright$
& E.Assmp.$^\triangleleft$
& $N_\varepsilon$
& Std. $\varepsilon^{-2}$?$^\blacklozenge$
& U.S.$^\heartsuit$\\
\midrule

BSA~\citep{ghadimi2018approximation}
& $\times$
& H/JVP
& Std.
& --
& $\widetilde{\mathcal O}(\varepsilon^{-3})$
& $\times$
& Full\\

SOBA~\citep{dagreou2022framework}
& $\checkmark$
& H/JVP
& Str.
& HO
& $\mathcal O(\varepsilon^{-2})$
& HO only
& Full\\

SUSTAIN/MRBO~\citep{khanduri2021near,yang2021provably}
& $\checkmark$
& H/JVP
& Str.
& MS
& $\widetilde{\mathcal O}(\varepsilon^{-3/2})$
& MS only
& Full\\

F$^2$SA~\citep{kwon2023fully}
& $\checkmark$
& FO
& Diff.
& FO
& $\widetilde{\mathcal O}(\varepsilon^{-7/2})$
& $\times$
& Full\\

F$^3$SA~\citep{kwon2023fully}
& $\checkmark$
& FO
& Diff.
& FO
& $\widetilde{\mathcal O}(\varepsilon^{-5/2})$
& $\times$
& Full\\

FdeHBO~\citep{yang2023achieving}
& $\checkmark$
& FO
& Diff.
& FD+Proj.
& $\widetilde{\mathcal O}(\varepsilon^{-3/2})$
& Diff.
& Full\\

SPABA-E~\citep{chu2024spaba}
& $\checkmark$
& H/JVP
& Str.
& MS
& $\mathcal O(\varepsilon^{-3/2})$
& MS only
& Full \\

SPABA-FS~\citep{chu2024spaba}
& $\checkmark$
& H/JVP
& Std.
& --
& $\mathcal O((n+m)^{1/2}\varepsilon^{-1})$
& FS
& Full\\

\midrule
\rowcolor{boxbg}
\textbf{\ourslast}
& $\checkmark$
& H/JVP
& \textbf{Std.}
& RP
& $\mathbf{\mathcal O(\varepsilon^{-2})}$
& $\checkmark$
& \textbf{Proj.}\\

\midrule
MA-SOBA~\citep{chen2024optimal}
& $\checkmark$
& H/JVP
& Std.
& --
& $\mathcal O(\varepsilon^{-2})$
& $\checkmark$
& Full\\

\bottomrule
\end{tabular}
\footnotesize
\begin{tablenotes}
\footnotesize
\item[$^\dagger$]The algorithm is a single-loop method.
\item[$^\ddagger$]Derivative oracle. H/JVP = Hessian-/Jacobian-vector products, FO = first-order-only.
\item[$^\triangleright$]Coarse assumption regime for comparability with the standard SBO baseline. Std. = standard nonconvex--strongly-convex stochastic bilevel setting with smooth $f$, strongly convex and second-order smooth $g$, and unbiased bounded-variance stochastic gradient/HVP/JVP oracles; Str. = stronger assumptions than Std.; Diff. = different oracle model or reformulation.
\item[$^\triangleleft$]Specific extra condition, oracle restriction, or mechanism explaining the entry in 'Assmp.'. HO = high-order smoothness, MS = mean-squared smoothness of stochastic oracles, FD+Proj. = finite-difference approximation with auxiliary projection, RP = randomized-projection moment conditions with no block-diagonal lower-Hessian assumption.
\item[$^\blacklozenge$]The method attains the standard $\mathcal O(\varepsilon^{-2})$ complexity under Std. assumption. ``HO only'', ``MS only'', ``Diff.'', and ``FS'' mean that the displayed rate requires, respectively, 
additional high-order smoothness, mean-squared smoothness of stochastic oracles, a different oracle model/reformulation, or a finite-sum setting.
\item[$^\heartsuit$] Lower/Auxiliary update space. Full = full trainable space, Proj. = projected/randomized subspace. 
\end{tablenotes}
\end{threeparttable}
\end{table*}

\textbf{Notations. } For an integer $L$, let $[L]:=\{1,\ldots,L\}$. All vector norms are Euclidean norms, and all matrix norms are Frobenius norms unless otherwise specified. For random variables, $\E[\cdot]$ and $\E[\cdot\mid\mathcal F]$ denote expectation and conditional expectation, and $\mathcal F_k$ denotes the filtration before drawing the projector at iteration $k$. We use bold uppercase letters, such as $\boldsymbol U,\boldsymbol V,\boldsymbol Y,\boldsymbol Z$, for multilayer matrix variables with the same layerwise structure as $\boldsymbol Y=(Y_1,\ldots,Y_L)$.

The $\ell$-th block of $\boldsymbol U$ is denoted by $U_\ell$. For two block variables with compatible dimensions, $\langle\cdot,\cdot\rangle$ and $\|\cdot\|$ denote the product Frobenius inner product and norm. For the lower-level objective $g$, write $\mathcal H(x,\boldsymbol Y):=\nabla^2_{\boldsymbol Y\boldsymbol Y}g(x,\boldsymbol Y)$ and $\mathcal J(x,\boldsymbol Y):=\nabla^2_{x\boldsymbol Y}g(x,\boldsymbol Y)$. For a block linear operator $\mathcal A:\mathcal Y\to\mathcal Y$, $\mathcal A_{\ell t}$ denotes its $(\ell,t)$-block operator from block $t$ to block $\ell$. At iteration $k$, we write $\mathcal H^k:=\mathcal H(x^k,\boldsymbol Y^k)$ and $\mathcal H_{\ell t}^k$ for the $(\ell,t)$-block of $\mathcal H^k$. When multilayer projectors appear, $P_\ell\in\mathbb R^{m_\ell\times r_\ell}$ with $r_\ell\le m_\ell$, and products such as $\boldsymbol P\boldsymbol B$ and $\boldsymbol P^\top\boldsymbol U$ are understood blockwise.

\section{Stochastic bilevel optimization via randomized subspace (\ourslast)}
\label{sec:algorithm}

This section develops \ours through a simplified single-layer case with $L=1$, i.e., the case where the lower-level variable $Y$ is a single-layer matrix in $\mathbb R^{m\times n}$. We make this simplification only for clarity of exposition, the canonical multi-layer derivation is deferred to Appendix~~\ref{app:multilayer_correction}.

\subsection{Hypergradient computation and single-loop bilevel updates}
\label{subsec:target_single}

When solving the upper-level problem, the core is the hypergradient $\nabla\Phiobj(x)$~\citep{ghadimi2018approximation}. Under Assumption~\ref{as:A1}, the lower-level solution $Y^*(x)$ is unique, and the implicit-gradient formula gives:
\begin{equation}
\label{eq:hypergrad}
\nabla\Phiobj(x)=\nabla_x f(x,Y^*(x))-\mathcal J(x,Y^*(x))\!\left[\mathcal H(x,Y^*(x))^{-1}[\nabla_Y f(x,Y^*(x))]\right].
\end{equation}
Evaluating this expression directly is expensive because it involves the inverse of the lower Hessian operator. Following the standard exact-hypergradient route, we introduce an auxiliary variable
\begin{equation*}
Z^*(x):=\mathcal H(x,Y^*(x))^{-1}[\nabla_Y f(x,Y^*(x))].
\end{equation*}
Equivalently, $Z^*(x)$ is the solution to the quadratic auxiliary problem
\begin{equation}
\label{eq:auxiliary_quadratic_single}
Z^*(x)=\argmin_{Z\in\mathbb R^{m\times n}}\left\{\frac{1}{2}\langle Z,\mathcal H(x,Y^*(x))[Z]\rangle-\langle Z,\nabla_Y f(x,Y^*(x))\rangle\right\}.
\end{equation}
Once $Z^*(x)$ is obtained, the hypergradient reduces to $\nabla_x f(x,Y^*(x))-\mathcal J(x,Y^*(x))[Z^*(x)]$.

Following this idea, solving the bilevel problem essentially involves three coupled subproblems. Define $h(x,Y,Z):=\frac{1}{2}\langle Z,\mathcal H(x,Y)[Z]\rangle-\langle Z,\nabla_Y f(x,Y)\rangle$. Then
\begin{subequations}
\label{eq:three_single_problems}
\begin{align}
x^*&\in\argmin_{x\in\mathbb R^{d_x}} f(x,Y^*(x)), &&\text{(upper-level)}\label{eq:single_upper_problem}\\
Y^*(x)&=\argmin_{Y\in\mathbb R^{m\times n}} g(x,Y), &&\text{(lower-level)}\label{eq:single_lower_problem}\\
Z^*(x)&=\argmin_{Z\in\mathbb R^{m\times n}} h(x,Y^*(x),Z). &&\text{(auxiliary-level)}\label{eq:single_aux_problem}
\end{align}
\end{subequations}
Given $x$, the exact-hypergradient procedure first obtains $Y^*(x)$ by solving the lower-level problem, then solves the auxiliary-level problem with $Y^*(x)$ fixed to obtain $Z^*(x)$, and finally uses $Z^*(x)$ to form the hypergradient for the upper-level update. The auxiliary formulation in~\eqref{eq:three_single_problems} therefore leads to a three-level computational template. Rather than solving the lower-level and auxiliary-level problems to high accuracy in nested loops, recent exact single-loop methods perform one stochastic update for each level per iteration, yielding a coupled single-loop recursion over $(x,Y,Z)$~\citep{dagreou2022framework,chen2024optimal}.

\textbf{Memory bottleneck in large-scale lower-level problems. }In the full-space single-loop template, the lower update uses a stochastic lower gradient, the auxiliary update uses HVPs to solve the quadratic system for $Z$, and the upper update substitutes the auxiliary variable into the hypergradient formula. This creates a substantial memory bottleneck. For a lower-level block $Y\in\mathbb R^{m\times n}$, the full-space recursion must maintain the lower variable $Y^k$ and the auxiliary variable $Z^k$, and the derivative evaluations introduce additional $O(mn)$ memory for lower gradients, auxiliary residuals, HVP inputs and outputs, and JVP-related quantities. Moreover, HVP/JVP computations are implemented through higher-order automatic differentiation, which can retain multiple forward/backward tapes and temporary tensors during one iteration. Thus, the peak trainable-side memory can be several times larger than the memory needed to store the lower model itself, creating a large overhead for high-dimensional neural networks and potentially translating into substantially higher accelerator-memory requirements and hardware provisioning costs.

\subsection{Randomized subspace sketching for bilevel recursions}
\label{subsec:projection_single}

To reduce the memory overhead described above, a natural idea is to combine randomized subspace optimization (RSO)~\citep{he2024subspace,chen2025memory} for the lower-level update with randomized subspace Newton (RSN) sketching~\citep{gower2019rsn} for the auxiliary update.

Specifically, at iteration $k$, we sample a scaled Haar projector $P^k\in\mathbb R^{m\times r}$, i.e., $P^k=\sqrt{m/r}\,O_r^k$ where $O_r^k$ contains the first $r$ columns of a Haar-uniform orthogonal matrix. This normalization gives $(P^k)^\top P^k=(m/r)I_r$ and $\E[P^k(P^k)^\top\mid\mathcal F_k]=I_m$. We then search for the update to $Y^k$ inside the subspace induced by $P^k$:
\begin{equation}
\label{eq:rso_lower_subproblem}
B^k= \argmin_{B\in\mathbb R^{r\times n}}\tilde g_k(x^k,B), \quad \tilde g_k(x,B):=g(x,Y^k+P^kB).
\end{equation}
After solving this reduced problem, the full-space lower variable is updated by $Y^{k+1}=Y^k+P^kB^k$.

For the auxiliary variable, the analogous sketching step computes Hessian and Jacobian actions only with respect to the reduced variable $B$. For $W\in\mathbb R^{r\times n}$, define the projected HVP and JVP by
\begin{equation}
\label{eq:projected_hvp_jvp_main}
\tilde{\mathcal H}^{k+1}[W]:=\nabla_{BB}^2\tilde g_k(x^k,0)[W], \quad \tilde{\mathcal J}^{k+1}[W]:=\nabla_{xB}^2\tilde g_k(x^k,0)[W].
\end{equation}
The vanilla sketch-then-lift auxiliary update would then replace the full-space HVP and JVP by
\begin{equation}
\label{eq:naive_projected_hvp}
\widehat{\mathcal H}_{\rm lift}^{k+1}[Z^k]:=P^k\tilde{\mathcal H}^{k+1}[(P^k)^\top Z^k], \quad \widehat{\mathcal J}_{\rm lift}^{k+1}[Z^k]:=\tilde{\mathcal J}^{k+1}[(P^k)^\top Z^k].
\end{equation}
This construction replaces the $O(mn)$ full-space HVP/JVP directions and their associated temporary memory with projected $O(rn)$ derivative directions. Thus, before the final lifting step, the lower-gradient, HVP, and JVP computations are carried out in the randomized subspace, directly reducing the memory overhead identified in Section~\ref{subsec:target_single}.

\textbf{Key limitation for vanilla sketch-then-lift. }The key limitation is that the vanilla sketch-then-lift estimator does not, in general, recover the full-space HVP required by the exact hypergradient formula. First, the lifted Hessian action in \eqref{eq:naive_projected_hvp} is generally a biased estimate of the full-space Hessian action: even when $P^k$ is sampled from a Haar distribution and satisfies $\E[P^k(P^k)^\top\mid\mathcal F_k]=I_m$, there is \emph{no} guarantee for $\E[\widehat{\mathcal H}_{\rm lift}^{k+1}[Z^k]\mid\mathcal F_k]=\mathcal H^k[Z^k]$. The auxiliary recursion therefore tracks an inexact linear system, which distorts the hypergradient and can lead to a nonstationary limit. Appendix~\ref{app:counterexample} gives a deterministic quadratic counterexample for this issue. Second, any recovery mechanism for the above-mentioned distortion must preserve the subspace memory advantage. In particular, it \emph{cannot} rely on forming full Hessian/Jacobian matrices, computing full-space HVPs/JVPs, or incurring full-space derivative memory.

\subsection{Rademacher-probe recovery of auxiliary HVPs}
\label{subsec:correction_single}
To address this bias without sacrificing the subspace memory advantage, \ours introduces a Rademacher-probe correction that targets the full-space HVP using only randomized-subspace derivative queries. Define $\tilde Z^k:=(P^k)^\top Z^k$ and $A^{k+1}:=P^k\tilde{\mathcal H}^{k+1}[\tilde Z^k]$. The Rademacher probes are random sign vectors and are not used as additional optimization directions; instead, they estimate the trace-like and transpose-like nuisance components induced by the vanilla lifted HVP. Specifically, \ours samples independent probes $u^{k+1}\in\{-1,1\}^{n}$ and $\xi^{k+1}\in\{-1,1\}^{r}$, and queries one additional subspace HVP $M^{k+1}:=\tilde{\mathcal H}^{k+1}[\xi^{k+1}(u^{k+1})^\top]$, which is used to form the correction terms
\begin{equation}
\label{eq:single_biprobe_estimators}
\begin{aligned}
\widehat C^{k+1}&:=Z^k\!\Bigl((M^{k+1})^\top\xi^{k+1}(u^{k+1})^\top\Bigr), \quad \widehat B_{\sharp}^{k+1}&:=P^k\!\Bigl(\xi^{k+1}\bigl((M^{k+1})^\top(\tilde Z^ku^{k+1})\bigr)^\top\Bigr).
\end{aligned}
\end{equation}
For Haar projectors~\citep{chen2025memory} $P^k$ with $r\ge2$, the corrected HVP estimator used by \ours is
{\small
\begin{equation}
\label{eq:single_unbiased_estimator}
\widehat{\mathcal H[Z]}^{k+1}:=\frac{r(m-1)(m+2)}{m(mr+m-2)}A^{k+1}-\frac{m-r}{mr+m-2}\widehat C^{k+1}+\frac{r(m-1)(m-r)}{m(r-1)(mr+m-2)}\bigl(A^{k+1}-\widehat B_{\sharp}^{k+1}\bigr).
\end{equation}}

For later use, we denote the corrected estimator in~\eqref{eq:single_biprobe_estimators}--\eqref{eq:single_unbiased_estimator} by
\[
\mathrm{CorrHVP}(P^k,Z^k,\tilde{\mathcal H}^{k+1};u^{k+1},\xi^{k+1})
:=\widehat{\mathcal H[Z]}^{k+1}.
\]
This subroutine provides the debiased auxiliary Hessian-action direction used in place of the vanilla lifted HVP in the $Z$-update.

For the multilayer setting, \ours applies this correction layerwise to the diagonal Hessian blocks and directly lifts the off-diagonal projected HVPs. The detailed blockwise construction process is deferred to Appendix~\ref{app:multilayer_correction}. In Algorithm~\ref{alg:rso_masoba}, $E_t(W)$ inserts $W$ into the $t$-th projected block and zeros elsewhere, and $\mathrm{CorrHVP}_\ell$ denotes the layer-$\ell$ version of $\mathrm{CorrHVP}$ with layerwise dimensions, projector, and probes. Algorithm~\ref{alg:rso_masoba} presents the resulting multilayer single-loop recursion: after constructing the corrected auxiliary HVP estimator, it performs one coupled update for the upper-level variable, lower-level variable, auxiliary variable, and the moving-average hypergradient estimator $h^k$. The moving-average recursion for $h^k$ is a standard variance-reduction device in single-loop SBO algorithms, used to stabilize stochastic hypergradient evaluation~\cite{chen2024optimal,kong2025decentralized}.

\begin{algorithm}[H]
\caption{\ourslast}
\label{alg:rso_masoba}
\small
\begin{algorithmic}[1]
\Require stepsizes $\{\alpha_k,\beta_k,\gamma_k,\theta_k\}$, initial $(x^0,\boldsymbol Y^0,\boldsymbol Z^0,h^0=0)$
\For{$k=0,1,\ldots,K-1$}
    \State Sample projectors $\boldsymbol P^k$ and Rademacher probes $\{u_\ell^{k+1},\xi_\ell^{k+1}\}_{\ell=1}^L$
    \State Obtain projected stochastic oracles $(\tilde U_x^{k+1},\tilde{\boldsymbol U}_Y^{k+1},\tilde{\boldsymbol V}^{k+1},\tilde{\mathcal H}^{k+1},\tilde{\mathcal J}^{k+1})$
    \For{$\ell=1,\ldots,L$}
        \State $\widehat{\mathcal H_{\ell\ell}[Z_\ell]}^{k+1}\gets \mathrm{CorrHVP}_\ell\!\left(P_\ell^k,Z_\ell^k,\tilde{\mathcal H}^{k+1};u_\ell^{k+1},\xi_\ell^{k+1}\right)$ \Comment{diagonal correction by \eqref{eq:single_unbiased_estimator}}
        \State $(\widehat{\mathcal H[\boldsymbol Z]}^{k+1})_\ell\gets \widehat{\mathcal H_{\ell\ell}[Z_\ell]}^{k+1}\hspace{-1em}+\underset{t\ne\ell}{\sum}P_\ell^k\bigl(\tilde{\mathcal H}^{k+1}[E_t((P_t^k)^\top Z_t^k)]\bigr)_\ell$ \Comment{off-diagonal blocks lifting}
    \EndFor
    \State $x^{k+1}\gets x^k-\alpha_k h^k$ \Comment{upper-level update}
    \State $\boldsymbol Y^{k+1}\gets \boldsymbol Y^k-\beta_k\boldsymbol P^k\tilde{\boldsymbol V}^{k+1}$ \Comment{lower-level update}
    \State $\boldsymbol Z^{k+1}\gets \boldsymbol Z^k-\gamma_k\bigl(\widehat{\mathcal H[\boldsymbol Z]}^{k+1}-\boldsymbol P^k\tilde{\boldsymbol U}_Y^{k+1}\bigr)$ \Comment{auxiliary-level update}
    \State $h^{k+1}\gets (1-\theta_k)h^k+\theta_k\bigl(\tilde U_x^{k+1}-\tilde{\mathcal J}^{k+1}[\boldsymbol P^{k\top}\boldsymbol Z^k]\bigr)$ \Comment{moving average update}
\EndFor
\end{algorithmic}
\end{algorithm}

\textbf{Compatibility with momentum and adaptive optimizers.}
Algorithm~\ref{alg:rso_masoba} presents the SGD-style instantiation used in our convergence proof, where the lower- and auxiliary-level variables are updated by plain stochastic steps along the projected lower direction and the corrected auxiliary direction, respectively. The \ours correction is orthogonal to the choice of optimizer: it only constructs the stochastic directions used by these two updates. Therefore, the same projected lower direction and auxiliary direction can be plugged into momentum SGD~\cite{sutskever2013importance} or Adam-family optimizers~\cite{kingma2014adam,guo2021novel} by replacing the plain SGD updates for $\boldsymbol Y$ and $\boldsymbol Z$ with the corresponding optimizer steps. Appendix~\ref{app:optimizer_compatibility} gives a more detailed discussion of this compatibility. The convergence theorem below is stated for the SGD-style recursion; analyzing momentum or adaptive variants is left for future work.

\section{Convergence analysis for \ourslast}
\label{sec:convergence}

\subsection{Assumptions}

The first assumption is standard for hypergradient evaluation in bilevel analysis~\citep{chen2024optimal,chen2023decentralized,kong2025decentralized}.

\Needspace{7\baselineskip}
\begin{limbox}
\begin{assumption}[Smoothness, boundedness, and lower-level curvature]
\label{as:A1}
There exist positive constants $\mu_g,L_{f,0},L_{f,1},L_{g,1},L_{g,2}>0$ such that:
\begin{enumerate}[leftmargin=1.5em,label=\arabic*.,topsep=1pt,itemsep=1pt,parsep=0pt,partopsep=0pt]
    \item $\nabla f$ and $\nabla g$ are Lipschitz continuous in $(x,\boldsymbol Y)$ with constants $L_{f,1}$ and $L_{g,1}$, and $\nabla^2 g$ is $L_{g,2}$-Lipschitz continuous as a linear operator on the product space $\mathbb R^{d_x}\times\mathcal Y$;
    \item $\|\nabla_{\boldsymbol Y} f(x,\boldsymbol Y^*(x))\|\le L_{f,0}$ for all $x$;
    \item for every fixed $x$, $g(x,\cdot)$ is $\mu_g$-strongly convex on $\mathcal Y$.
\end{enumerate}
Let $L_{\max}:=\max\{L_{f,0},L_{f,1},L_{g,1},L_{g,2}\}$ and $\kappa:=L_{\max}/\mu_g$.
\end{assumption}
\end{limbox}

The next assumption is the standard SBO oracle condition, adapted to the projected stochastic oracles used by \ourslast. A complete filtration-level statement is given in Appendix~\ref{app:full_proofs}.
\Needspace{7\baselineskip}
\begin{limbox}
\begin{assumption}[Projected stochastic oracles]
\label{as:A2}
At iteration $k$, let $\mathcal F_k$ be the filtration before drawing $\boldsymbol P^k$ and let $\mathcal F_k^+$ additionally include $\boldsymbol P^k$. For the projected local objectives induced by $\boldsymbol Y^k+\boldsymbol P^k\boldsymbol B$, the stochastic oracles in Algorithm~\ref{alg:rso_masoba} satisfy:
\begin{enumerate}[leftmargin=1.5em,label=\arabic*.,topsep=1pt,itemsep=1pt,parsep=0pt,partopsep=0pt]
    \item the projected gradient oracles of $f$ and $g$ are conditionally unbiased given $\mathcal F_k^+$ and have uniformly bounded conditional second moments;
    \item the projected HVP and JVP oracles are conditionally unbiased given $\mathcal F_k^+$, with conditional second moments bounded by a constant times the squared norm of the queried direction;
    \item multiple queries to the same sampled HVP/JVP operator within one iteration are allowed; the Rademacher probes are conditionally independent of the stochastic oracle samples given $\mathcal F_k^+$.
\end{enumerate}
\end{assumption}
\end{limbox}

The final assumption specifies the random subspaces used by \ourslast. It matches the standard isotropic projector assumption used in RSO~\citep{chen2025memory}, and its moment identities are used in the diagonal-block bias correction.

\Needspace{7\baselineskip}
\begin{limbox}
\begin{assumption}[Random projectors and probes]
\label{as:A3}
Conditional on $\mathcal F_k$, the layerwise projectors $\{P_\ell^k\}_{\ell=1}^L$ are independent, with $P_\ell^k=\sqrt{m_\ell/r_\ell}\,O_{\ell,r_\ell}^k$, where $O_{\ell,r_\ell}^k$ contains the first $r_\ell$ columns of a Haar-uniform matrix $O_\ell^k$ in the orthogonal group $O(m_\ell)$. We assume $r_\ell\ge2$ for all $\ell$. The Rademacher probes in Algorithm~\ref{alg:rso_masoba} are independent across layers and independent of the stochastic oracle samples conditional on $\mathcal F_k^+$.
\end{assumption}
\end{limbox}

\paragraph{Projector sampling in implementation}
The Haar projector in Assumption~\ref{as:A3} can be sampled without forming a full $m_\ell\times m_\ell$ orthogonal matrix. In implementation, for each layer we draw a Gaussian matrix $G_\ell^k\in\mathbb R^{m_\ell\times r_\ell}$, compute its thin QR factorization $G_\ell^k=Q_\ell^kR_\ell^k$ with the standard positive-diagonal convention for $R_\ell^k$, and set $P_\ell^k=\sqrt{m_\ell/r_\ell}\,Q_\ell^k$. The resulting $Q_\ell^k$ is Haar-uniform on the Stiefel manifold, equivalently distributed as the first $r_\ell$ columns of a Haar-uniform matrix.

\subsection{Convergence results for \ours}
\label{subsec:main_convergence}
In this subsection, we state the main guarantee supporting the core claims of \ourslast. We first present the following lemma to state that \ours restores the same conditional expectations that a full-space single-loop method would use, while keeping the variance controlled. A detailed version with explicit constants is proved in Appendix~\ref{app:full_proofs}.
\Needspace{9\baselineskip}
\begin{limbox}
\begin{lemma}[Moment preservation of \ours directions]
\label{lem:bros_moment_preservation}
Suppose Assumptions~\ref{as:A1}--\ref{as:A3} hold. Define the full-space bilevel directions
{\small
\begin{equation*}
\begin{aligned}
D_Y^k\hspace{-0.3em}:=\hspace{-0.3em}\nabla_{\boldsymbol Y}g(x^k,\boldsymbol Y^k), \quad D_Z^k\hspace{-0.3em}:=\hspace{-0.3em}\mathcal H(x^k,\boldsymbol Y^k)[\boldsymbol Z^k]\hspace{-0.2em}-\hspace{-0.3em}\nabla_{\boldsymbol Y}f(x^k,\boldsymbol Y^k), \quad D_x^k\hspace{-0.3em}:=\hspace{-0.3em}\nabla_x f(x^k,\boldsymbol Y^k)\hspace{-0.2em}-\hspace{-0.3em}\mathcal J(x^k,\boldsymbol Y^k)[\boldsymbol Z^k].
\end{aligned}
\end{equation*}
}
For the analysis, denote the three stochastic directions generated by Algorithm~\ref{alg:rso_masoba} as $\widehat{\boldsymbol V}^{k+1}\hspace{-0.8em}:=\boldsymbol P^k\tilde{\boldsymbol V}^{k+1}$, $\widehat{\boldsymbol S}^{k+1}\hspace{-0.8em}:=\widehat{\mathcal H[\boldsymbol Z]}^{k+1}\hspace{-1.0em}-\boldsymbol P^k\tilde{\boldsymbol U}_Y^{k+1}$, and $\widehat W^{k+1}\hspace{-0.3em}:=\tilde U_x^{k+1}\hspace{-0.3em}-\tilde{\mathcal J}^{k+1}[\boldsymbol P^{k\top}\boldsymbol Z^k]$. Then
\begin{equation}
\label{eq:main_moment_preservation}
\E[\widehat{\boldsymbol V}^{k+1}\mid\mathcal F_k]=D_Y^k,\quad \E[\widehat{\boldsymbol S}^{k+1}\mid\mathcal F_k]=D_Z^k,\quad \E[\widehat W^{k+1}\mid\mathcal F_k]=D_x^k.
\end{equation}

Moreover, there exists a finite constant $C_m>0$, independent of $K$, such that
{\small
\begin{equation}
\label{eq:main_moment_bound}
\E\!\left[\|\widehat{\boldsymbol V}^{k+1}\hspace{-1.0em}-D_Y^k\|^2
\hspace{-0.4em}+\hspace{-0.2em}\|\widehat{\boldsymbol S}^{k+1}\hspace{-1.0em}-D_Z^k\|^2
\hspace{-0.4em}+\hspace{-0.2em}\|\widehat W^{k+1}\hspace{-0.6em}-D_x^k\|^2\hspace{-0.3em}\mid\hspace{-0.2em}\mathcal F_k\right] \hspace{-0.3em}\le\hspace{-0.2em} C_m\Bigl(\hspace{-0.1em}1+\|\boldsymbol Y^k\hspace{-0.3em}-\hspace{-0.1em}\boldsymbol Y^*(x^k)\|^2+\|\boldsymbol Z^k\hspace{-0.3em}-\hspace{-0.1em}\boldsymbol Z^*(x^k)\|^2\Bigr).
\end{equation}
}
\end{lemma}
\end{limbox}
Consequently, Lemma~\ref{lem:bros_moment_preservation} allows us to view the projected recursions of \ours as full-space single-loop bilevel updates with unbiased directions and controlled projection-induced variance, which leads to the following convergence guarantee.
\Needspace{9\baselineskip}
\begin{limbox}
\begin{theorem}[Main guarantee of \ourslast]
\label{thm:bros_main}
Suppose Assumptions~\ref{as:A1}--\ref{as:A3} hold, and assume also that $\Phi^*:=\inf_x \Phi(x)>-\infty$. Consider Algorithm~\ref{alg:rso_masoba}. Let $\omega_P:=\max_{1\le \ell\le L}(m_\ell/r_\ell)$ be the worst layerwise full-rank to subspace-rank ratio. Then there exist positive coupling constants $c_1,c_2,c_3$ and finite constants $C_0,C_1,\bar\alpha>0$, such that the following holds: If $\alpha_k\equiv \alpha$, $\beta_k\equiv c_1\alpha$, $\gamma_k\equiv c_2\alpha$, $\theta_k\equiv c_3\alpha$, with $0<\alpha\le (\bar\alpha/(1+\omega_P^3))$, we have (see detailed proof in Appendix~\ref{app:full_proofs}):
\begin{equation}
\label{eq:bros_stationarity_rate}
\frac{1}{K}\sum_{k=0}^{K-1}\E\|\nabla\Phi(x^k)\|^2 \le 2\left( \frac{C_0(1+\omega_P^2)}{K\alpha} +C_1(1+\omega_P^4)\alpha \right).
\end{equation}
Choosing $\alpha=\Theta(K^{-1/2})$ subject to the above upper bound yields
\begin{equation*}
\begin{aligned}
\frac{1}{K}\sum_{k=0}^{K-1}\E\|\nabla\Phi(x^k)\|^2=\mathcal O\!\left((1+\omega_P^4)K^{-1/2}\right), \quad \text{and} \quad \E\|\nabla\Phi(x^R)\|^2=\mathcal O\!\left((1+\omega_P^4)K^{-1/2}\right),
\end{aligned}
\end{equation*}
where $R\sim\mathrm{Unif}\{0,\ldots,K-1\}$ is independent of the algorithmic randomness. Hence, for fixed projection ratio $\omega_P$, \ours attains the standard $\mathcal O(\varepsilon^{-2})$ sample complexity for finding an $\varepsilon$-stationary point. With the projection factor displayed, it suffices to take $K=\mathcal O((1+\omega_P^4)^2\varepsilon^{-2})$.
\end{theorem}
\end{limbox}

\textbf{Full-space-order rate and projection-ratio dependence.}
Theorem~\ref{thm:bros_main} separates the standard single-loop SBO order from the price of projection: the dependence on $K$ matches the $\mathcal O(K^{-1/2})$ stationarity rate of MA-SOBA with full-space hypergradient evaluation~\citep{chen2024optimal}, while the constants depend on $\omega_P=\max_\ell m_\ell/r_\ell$. Larger subspace ranks decrease $\omega_P$, improving both the transient term $C_0(1+\omega_P^2)/(K\alpha)$ and the stochastic-error term $C_1(1+\omega_P^4)\alpha$, and allowing a larger stepsize through $\alpha\le \bar\alpha/(1+\omega_P^3)$. Conversely, smaller ranks save memory and projected-oracle cost but inflate the variance of lifted stochastic directions, so the same stationarity level may require more iterations or a smaller stepsize. In the full-rank case $r_\ell=m_\ell$ for all layers, $\omega_P=1$, so the result recovers the usual full-space single-loop stochastic bilevel rate up to absolute constants.

\textbf{Role of the bias correction.}
The above MA-SOBA-order guarantee relies on the \ours correction because it centers the auxiliary linear-system recursion at the full-space Hessian action, which is required by Lemma~\ref{lem:bros_moment_preservation} to target the true hypergradient $\nabla\Phi(x)$. Appendix~\ref{app:counterexample} gives a deterministic quadratic counterexample showing that, without this correction, the naive projected auxiliary HVP can converge to a biased auxiliary fixed point and drive the upper update to a nonstationary point, even with no stochastic objective noise and exact lower-level solutions.

% \textbf{Projection-ratio dependence.}
% The factor $\omega_P=\max_\ell m_\ell/r_\ell$ makes the rank dependence explicit. Increasing the subspace ranks decreases $\omega_P$, which improves the bound by reducing both the transient term $C_0(1+\omega_P^2)/(K\alpha)$ and the stochastic-error term $C_1(1+\omega_P^4)\alpha$, and also allows a larger stepsize through $\alpha\le \bar\alpha/(1+\omega_P^3)$. Conversely, smaller ranks save memory and projected-oracle cost but inflate the variance of lifted stochastic directions, so the same stationarity level may require more iterations or a smaller stepsize. In the full-rank degenerate case $r_\ell=m_\ell$ for all layers, $\omega_P=1$ and $P_\ell P_\ell^\top=I_{m_\ell}$ almost surely, so there is no projection-induced variance inflation and the result reduces to the usual full-space single-loop stochastic bilevel rate up to absolute constants.

% \textbf{Necessity of the bias correction}
% The \ours correction is necessary because it centers the auxiliary linear-system recursion at the full-space Hessian action, which is required by Lemma~\ref{lem:bros_moment_preservation} to target the true hypergradient $\nabla\Phi(x)$. Appendix~\ref{app:counterexample} gives a deterministic quadratic counterexample showing that, without this correction, the naive projected auxiliary HVP can converge to a biased auxiliary fixed point and drive the upper update to a nonstationary point, even with no stochastic objective noise and exact lower-level solutions.

\section{Memory analysis}
\label{sec:memory_efficiency}
The memory savings of \ours stem directly from the randomized-subspace computations in Algorithm~\ref{alg:rso_masoba}. Under the peak-memory proxy considered here, the projected lower- and auxiliary-level computations reduce both the trainable-side activation footprint and the online lower/auxiliary directions from full-size tensors to subspace-size tensors. The auxiliary HVP/JVP queries are also evaluated in projected coordinates, reducing the backend-dependent memory pressure of second-order autodifferentiation; accordingly, we report this component by query scale rather than folding it into a backend-specific closed-form constant. Notably, RSO reports that activations can dominate trainable-side memory in large-batch regimes~\citep{chen2025memory}, making \ours's activation savings important when the lower-level task trains neural networks.

We analyze one trainable Transformer decoder block with a SwiGLU feed-forward network~\cite{shazeer2020glu} and standard multi-head attention~\cite{vaswani2017attention}, i.e., without grouped-query attention~\cite{ainslie2023gqa} (non-GQA). The component-wise derivation is provided in Appendix~\ref{subsec:bros_memory_non_gqa}. Let $n$ denote the hidden size, $s$ the sequence length, $b$ the micro-batch size, and $h$ the number of attention heads. We count scalar memory slots, omit upper-level variables and datatype constants, and use a peak-memory proxy that includes persistent lower/auxiliary states, saved activations, and online gradient/direction tensors. Under this accounting, the counted peak-memory proxy of \ours is
\begin{equation*}
M_{\rm peak}^{\text{\ourslast}}=24n^2\;(\text{state})+\frac{28}{3}bsn\;(\text{hidden})+2bhs^2\;(\text{attn})+4bsr\;(\text{proj-act})+\frac{62}{3}rn\;(\text{grad/dir}).
\end{equation*}
Table~\ref{tab:memory_constants_bros} gives the memory comparison between \ours and MA-SOBA~\citep{chen2024optimal}, FdeHBO~\citep{yang2023achieving}, and Penalty~\citep{mehra2021penalty}. The complete derivation is provided in Appendix~\ref{subsec:baseline_memory_non_gqa}.

\begin{table}[t!]
\centering
\caption{Non-GQA peak-memory proxy for one trainable decoder block. Counts use scalar slots, exclude upper variables, and report H/JVP query scale.}
\label{tab:memory_constants_bros}
\small
\setlength{\tabcolsep}{6.0pt}
\renewcommand{\arraystretch}{1.12}
\begin{tabular}{lccccc}
\toprule
Method & States & Activ. & Dirs. & H/JVP scale & Peak proxy \\
\midrule
MA-SOBA~\citep{chen2024optimal}
& $24n^2$
& $15bsn+2bhs^2$
& $24n^2$
& full H/JVP
& $48n^2+15bsn+2bhs^2$ \\
FdeHBO~\citep{yang2023achieving}
& $48n^2$
& $15bsn+2bhs^2$
& $24n^2$
& full FD dirs.
& $72n^2+15bsn+2bhs^2$ \\
Penalty~\citep{mehra2021penalty}
& $12n^2$
& $15bsn+2bhs^2$
& $12n^2$
& no aux. H/JVP
& $24n^2+15bsn+2bhs^2$ \\
\rowcolor{boxbg}
\ourslast
& $24n^2$
& $\frac{28}{3}bsn+2bhs^2+4bsr$
& $\frac{62}{3}rn$
& proj. H/JVP
& $24n^2+\frac{28}{3}bsn+2bhs^2+4bsr+\frac{62}{3}rn$ \\
\bottomrule
\end{tabular}
\end{table}

Let $\rho:=r/n$ and $\tau:=bs/n$. Compared with MA-SOBA in Table~\ref{tab:memory_constants_bros}, \ours reduces both the activation term and the online gradient/direction term for any nontrivial subspace ratio $\rho<1$; for example, at $(\tau,\rho)=(1,0.25)$, the displayed peak-memory proxy is reduced by about $37.3\%$. Penalty is the strongest memory-oriented surrogate baseline in the table, since it avoids the auxiliary linear system and its HVP/JVP queries. Even against Penalty, \ours can be more memory efficient in activation-dominated regimes; for example, at $(\tau,\rho)=(2,0.25)$, it saves about $7.7\%$ in the displayed proxy. Unlike Penalty, however, \ours preserves the exact single-loop hypergradient structure and the MA-SOBA-order convergence guarantee. These numbers do not include the additional peak-memory reduction from replacing full-space HVP/JVP autodiff queries by projected ones.

\section{Experiments}
We evaluate \ours on four bilevel learning tasks: MNIST hyper-data cleaning, data-mixture learning for language-model training, CIFAR-10 hyper-representation learning, and CIFAR-100 ViT sample reweighting. Across all experiments, \ours uses randomized subspace updates with a projection ratio $r/n=0.25$ unless otherwise specified.

\textbf{Hyper-data cleaning.} We follow the convex MNIST data-cleaning setting in \citep{shaban2019truncated}. The training labels are corrupted by $50\%$ random label noise, while clean validation and test sets are used for the upper-level objective and final evaluation. We compare Penalty~\citep{mehra2021penalty}, FdeHBO~\citep{yang2023achieving}, MA-SOBA~\citep{chen2024optimal}, ZOFO~\citep{shirkavand2025bilevel}, and \ourslast. Figure~\ref{fig:hyper_data_cleaning} shows that \ours reaches $84.30\%$ final test accuracy, closely matching MA-SOBA at $84.46\%$ and outperforming Penalty, FdeHBO, and ZOFO. The ablation further shows that projection ratios from $r/n = 0.10$ to $r/n = 0.50$ retain strong performance, while an extremely small ratio $r/n=0.01$ degrades accuracy slightly.

\textbf{Data-mixture learning.} We also evaluate data-mixture learning on $17$ copyright-free Pile domains. A $280$M GPT-style proxy model first learns the data-mixture weights, and a $280$M main model is then trained from scratch with the learned mixture. Both proxy and main models use $16$ layers, hidden size $1024$, and $16$ heads with the EleutherAI/gpt-neox-20b~\cite{black2022gpt} tokenizer. Figure~\ref{fig:data_mixture_main} reports the final per-domain evaluation. \ours achieves average loss $2.8109$, essentially matching MA-SOBA ($2.8098$) and ZOFO ($2.8114$), while avoiding the unstable domain concentration observed for DoReMi~\cite{xie2023doremi}.

\begin{figure}[t!]
\centering
\makebox[\textwidth][c]{%
\vtop{\hbox{\begin{minipage}{0.64\textwidth}
\centering
\vspace{-4.5mm}
\begin{minipage}[t]{0.485\linewidth}
\centering
\vspace{0pt}
\includegraphics[width=\linewidth]{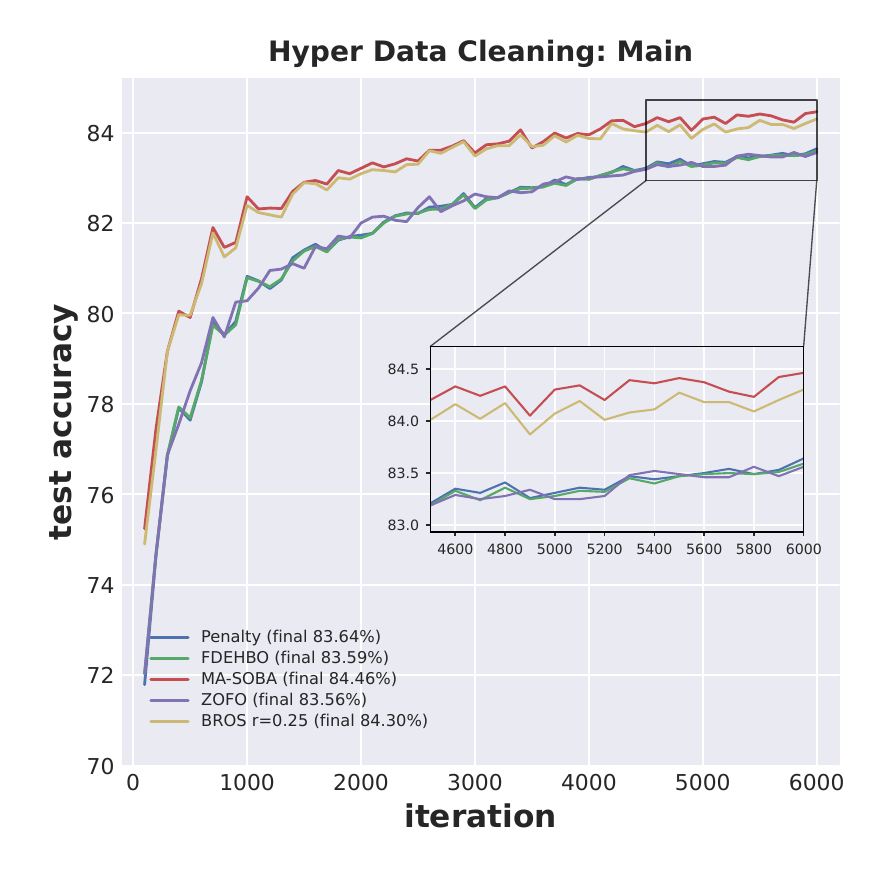}
\end{minipage}%
\hfill
\begin{minipage}[t]{0.485\linewidth}
\centering
\vspace{0pt}
\includegraphics[width=\linewidth]{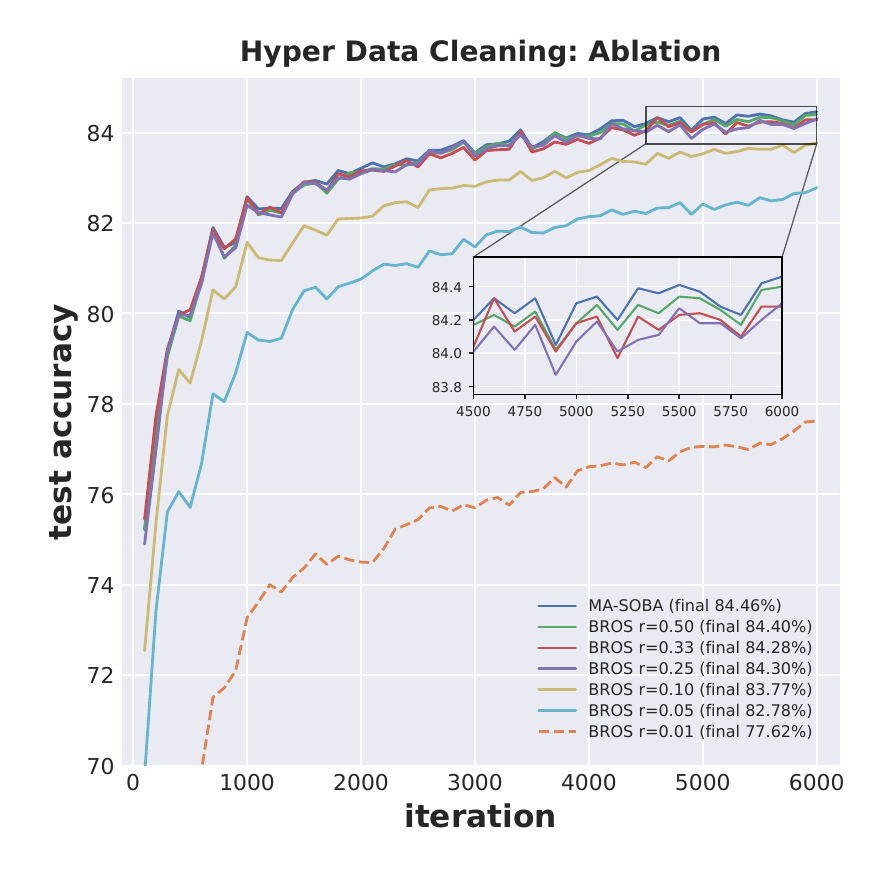}
\end{minipage}
\vspace{-4.5mm}
\caption{\textbf{Hyper-data cleaning on MNIST with $50\%$ label noise.} \textbf{Left:} Main comparison among Penalty, FdeHBO, MA-SOBA, ZOFO, and \ours with projection ratio $r/n=0.25$. \textbf{Right:} Projection-ratio ablation for \ourslast.}
\label{fig:hyper_data_cleaning}
\end{minipage}}}%
\hspace{0.02\textwidth}%
\vtop{\hbox{\begin{minipage}{0.32\textwidth}
\centering
\vspace{-4.5mm}
\includegraphics[width=0.97\linewidth]{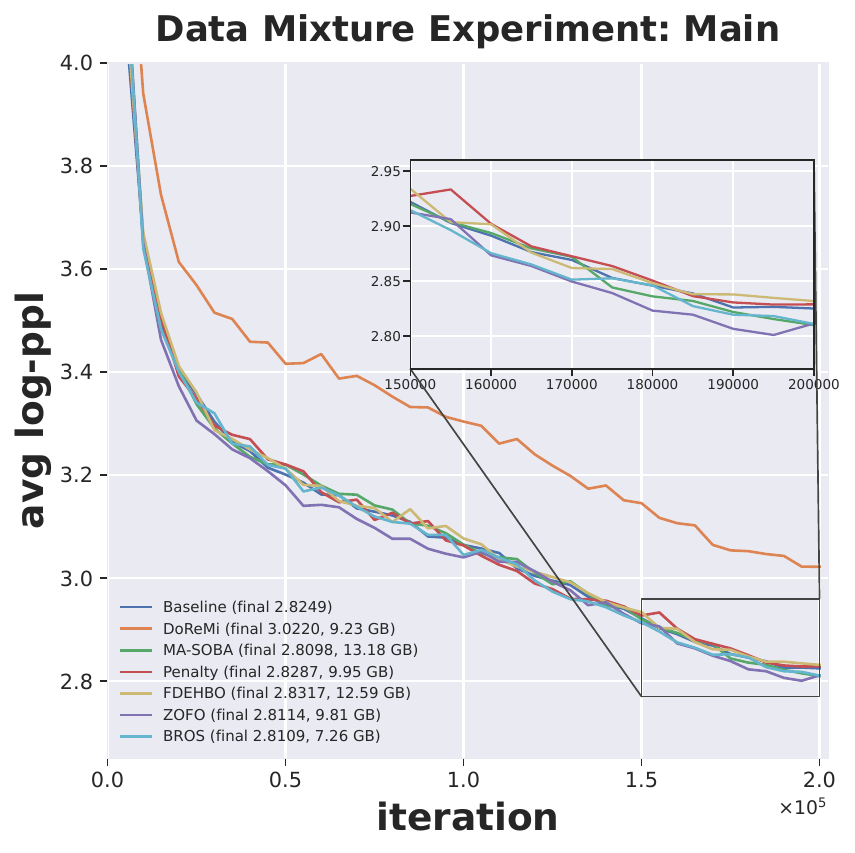}
\vspace{-4.5mm}
\caption{\textbf{Data-mixture learning.} Main-model evaluation after training with proxy-learned data mixtures.}
\label{fig:data_mixture_main}
\end{minipage}}}%
}
\end{figure}

The proxy-memory profiling results are included in the legend of Figure~\ref{fig:data_mixture_main}. \ours reduces peak memory by about $45\%$ relative to MA-SOBA and about $27\%$ relative to Penalty, while retaining a substantially higher step rate than ZOFO.

\textbf{Hyper-representation learning.} We further test \ours on CIFAR-10 hyper-representation learning, where the upper-level objective evaluates the learned representation and the lower-level problem trains predictors under that representation. We compare with the same bilevel baselines as above and report the full accuracy curves and projection-ratio ablation in Figure~\ref{fig:hyper_representation_appendix} of Appendix~\ref{app:additional_experiments}.

\textbf{ViT sample reweighting.} We also evaluate a larger CIFAR-100 sample-reweighting task with $20\%$ label noise, where the lower model is a small ViT trained from scratch and the upper level selects sample weights using a clean validation split. Detailed settings and the validation, test, and peak-memory comparison are given in Table~\ref{tab:vit_appendix} of Appendix~\ref{app:additional_experiments}.

\section{Conclusion and future work}
We propose \ourslast, a memory-efficient randomized-subspace method for stochastic bilevel optimization with multilayer matrix lower-level variables. By decomposing the lower and auxiliary recursions into randomized subspace updates and correcting the projection-induced auxiliary-HVP bias with Rademacher bi-probes, \ours preserves an unbiased hypergradient estimator and matches the $\mathcal O(\varepsilon^{-2})$ convergence order of MA-SOBA while reducing trainable-side memory. Future work includes further reducing the remaining activation memory in large architectures and extending the convergence analysis beyond the current SGD optimizer.

\newpage

{
\small

\bibliography{reference}
\bibliographystyle{icml2026}
}

%%%%%%%%%%%%%%%%%%%%%%%%%%%%%%%%%%%%%%%%%%%%%%%%%%%%%%%%%%%%
\newpage
\appendix
\begin{center}
{\sffamily\bfseries\fontsize{15}{18}\selectfont
    Appendix for ``BROS: Bias-Corrected Randomized Subspaces for Memory-Efficient Single-Loop Bilevel Optimization''\par}
\end{center}
\vspace{5mm}

\section{Related work}
\label{sec:related_work}
\label{app:related_work}

\textbf{Stochastic bilevel optimization. }
Bilevel optimization is widely used in hyperparameter optimization, meta-learning, data reweighting, and representation learning~\citep{franceschi2018bilevel,rajeswaran2019meta,bertinetto2018meta,ren2018learning,huang2024optimal}. The main technical difficulty is to estimate the hypergradient, since the upper-level gradient depends on how the lower-level solution changes with the upper-level variable. Early stochastic methods use double-loop updates, truncated backpropagation, approximate implicit differentiation, or two-timescale stochastic approximation~\citep{ghadimi2018approximation,hong2023two,ji2021bilevel,arbel2021amortized,chen2021tighter}. More recent single-loop methods update the lower variable, an auxiliary variable, and the upper variable in one loop. Representative examples include the single-timescale method~\citep{ghadimi2020single}, SUSTAIN~\citep{khanduri2021near}, SOBA/SABA~\citep{dagreou2022framework}, MA-SOBA~\citep{chen2024optimal}, and SPABA~\citep{chu2024spaba}. Beyond the single-machine setting, decentralized and federated bilevel methods extend the same nested optimization paradigm to networked or client-server systems, where hypergradient estimation interacts with communication topology, client participation, data heterogeneity, and consensus errors~\citep{yang2022decentralized,chen2023decentralized,tarzanagh2022fednest,li2023communication,yang2023simfbo,kong2025decentralized,zhu2024sparkle,yang2025first,he2024distributed}. These methods are the closest baseline for our work, but they still form both the lower update and the auxiliary update in the full parameter space.

\textbf{Bilevel methods with alternative hypergradient estimators or formulations. }
Another line of work studies bilevel optimization by changing how the hypergradient is computed, by adopting a different oracle model, or by reformulating the original bilevel problem. Penalty methods~\citep{mehra2021penalty,kwon2023penalty} replace the original bilevel problem with a penalty reformulation. Fully first-order methods~\citep{li2022fully,kwon2023fully} avoid Hessian- and Jacobian-vector products. FdeHBO~\citep{yang2023achieving} uses finite differences to avoid explicit Hessian- and Jacobian-vector products. Bilevel-ZOFO~\citep{shirkavand2025bilevel} combines zeroth-order outer updates with parameter-efficient inner updates. These methods reduce the burden of full hypergradient computation, but they no longer follow the same standard single-loop hypergradient update as methods such as MA-SOBA~\citep{chen2024optimal}: they either optimize a surrogate problem, use a different oracle, or rely on weaker hypergradient information.

\textbf{Memory-efficient subspace and low-rank training. }
Subspace and projection methods are a common way to reduce memory in single-level neural-network optimization. A broad class of methods forms updates in lower-dimensional spaces, either by projecting gradients or optimizer states~\citep{zhao2024galore,zhang2025breaking,he2024subspace,xiao2025coap,zhu2025apollo}, or by restricting learning to sampled low-dimensional subspaces, as in randomized subspace optimization (RSO)~\citep{chen2025memory}. Related parameter-efficient tuning methods also reduce the number of trainable degrees of freedom through low-dimensional modules, adapters, prompts, or low-rank updates~\citep{hu2022lora,lialin2023relora,liu2024dora,han2024sltrain,miles2024velora,houlsby2019parameter,liu2025cola,kong2025cr,mo2025parameter}. Subspace ideas also appear in bilevel optimization, but with a different goal: LancBiO~\citep{yang2024lancbio} uses Krylov/Lanczos subspaces to approximate Hessian-inverse-vector products by solving small projected linear systems. This reduces the computational cost of the auxiliary linear solve, but it is not designed to reduce trainable-side memory; maintaining Krylov bases and projected matrices can introduce additional storage. These works motivate our study, but they do not consider randomized subspace projection within single-loop bilevel hypergradient methods. In particular, they do not address how such projections affect the auxiliary Hessian-vector computation and the resulting hypergradient estimate.

\textbf{Neural bilevel applications. }
Our matrix-block setting is motivated by neural bilevel problems in which the lower-level variable is itself a multilayer model. Typical examples include data hyper-cleaning~\citep{shaban2019truncated}, sample reweighting and label correction~\citep{ren2018learning,zheng2021meta,tu2023learning,hao2024bilevel,yang2024tuning}, hyper-representation learning~\citep{franceschi2018bilevel,huang2024optimal,yang2025first}, and meta-learning~\citep{rajeswaran2019meta,bertinetto2018meta}. Recent data-selection and data-mixture methods for larger models were studied in~\citep{fan2023doge,xie2023doremi,shen2024seal,pan2025scalebio,wang2026draw}. These applications motivate our problem setting, but our contribution is algorithmic: we study how to reduce the memory used by stochastic bilevel optimization while keeping the hypergradient estimator unbiased.

\section{Additional algorithm discussion}
\label{app:additional_algorithm_discussion}

\subsection{Multilayer derivation of the \ours recursion}
\label{app:multilayer_algorithm_derivation}
\label{app:multilayer_correction}

Section~\ref{sec:algorithm} introduces the main idea through the single-layer case. We now spell out the corresponding multilayer notation used in Algorithm~\ref{alg:rso_masoba}. Recall that $\boldsymbol Y=(Y_1,\ldots,Y_L)$ with $Y_\ell\in\mathbb R^{m_\ell\times n_\ell}$. At iteration $k$, \ours samples independent scaled Haar projectors $P_\ell^k\in\mathbb R^{m_\ell\times r_\ell}$ and writes $\boldsymbol P^k=(P_1^k,\ldots,P_L^k)$. The reduced search variable is $\boldsymbol B=(B_1,\ldots,B_L)\in\mathcal B:=\prod_{\ell=1}^L\mathbb R^{r_\ell\times n_\ell}$, and blockwise multiplication is defined by $(\boldsymbol P^k\boldsymbol B)_\ell=P_\ell^kB_\ell$ and $(\boldsymbol P^{k\top}\boldsymbol U)_\ell=(P_\ell^k)^\top U_\ell$.

The multilayer RSO lower-level step is obtained from the local projected objective
\begin{equation*}
\tilde g_k(x,\boldsymbol B):=g(x,\boldsymbol Y^k+\boldsymbol P^k\boldsymbol B).
\end{equation*}
Thus the lower direction is formed in the reduced product space $\mathcal B$ and then lifted back to the matrix-block space as $\boldsymbol P^k\tilde{\boldsymbol V}^{k+1}$, where $\tilde{\boldsymbol V}^{k+1}$ estimates $\nabla_{\boldsymbol B}\tilde g_k(x^k,\mathbf 0)$. Similarly, $\tilde{\boldsymbol U}_Y^{k+1}$ estimates $\nabla_{\boldsymbol B}\tilde f_k(x^k,\mathbf 0)$, $\tilde U_x^{k+1}$ estimates $\nabla_x\tilde f_k(x^k,\mathbf 0)$, and the projected linear operators $\tilde{\mathcal H}^{k+1}:\mathcal B\to\mathcal B$ and $\tilde{\mathcal J}^{k+1}:\mathcal B\to\mathbb R^{d_x}$ estimate the projected HVP and JVP.

The auxiliary HVP is assembled blockwise. Let $E_t(W)\in\mathcal B$ denote the reduced block variable that inserts $W\in\mathbb R^{r_t\times n_t}$ into block $t$ and zeros elsewhere. For a target layer $\ell$, the full Hessian action decomposes as $(\mathcal H^k[\boldsymbol Z^k])_\ell=\sum_{t=1}^L\mathcal H_{\ell t}^k[Z_t^k]$. The off-diagonal contribution from block $t\ne\ell$ is lifted directly from the projected HVP:
\begin{equation*}
A_\ell^{(t),k+1}:=P_\ell^k\bigl(\tilde{\mathcal H}^{k+1}[E_t((P_t^k)^\top Z_t^k)]\bigr)_\ell.
\end{equation*}
For the diagonal contribution, the one-sided row projection creates the same self-block distortion as in the single-layer case. Therefore \ours applies the single-layer correction in~\eqref{eq:single_unbiased_estimator} layerwise:
\begin{equation*}
\widehat{\mathcal H_{\ell\ell}[Z_\ell]}^{k+1}:=\mathrm{CorrHVP}_\ell\!\left(P_\ell^k,Z_\ell^k,\tilde{\mathcal H}^{k+1};u_\ell^{k+1},\xi_\ell^{k+1}\right).
\end{equation*}
The full multilayer estimator used in Algorithm~\ref{alg:rso_masoba} is then
\begin{equation*}
(\widehat{\mathcal H[\boldsymbol Z]}^{k+1})_\ell:=\widehat{\mathcal H_{\ell\ell}[Z_\ell]}^{k+1}+\sum_{t\ne\ell}A_\ell^{(t),k+1}, \quad \ell=1,\ldots,L.
\end{equation*}
This estimator is plugged into the auxiliary residual $\widehat{\mathcal H[\boldsymbol Z]}^{k+1}-\boldsymbol P^k\tilde{\boldsymbol U}_Y^{k+1}$, while the projected mixed-Jacobian term $\tilde{\mathcal J}^{k+1}[\boldsymbol P^{k\top}\boldsymbol Z^k]$ enters the moving-average hypergradient estimator. The resulting coupled update is exactly the recursion stated in Algorithm~\ref{alg:rso_masoba}. The moment-preservation proof for this blockwise construction is given in Lemma~\ref{lem:unbiased_lifted_oracles}.

\subsection{Compatibility with momentum and adaptive optimizers}
\label{app:optimizer_compatibility}

\ours is an estimator-level construction rather than a complete prescription for the optimizer used to integrate the lower and auxiliary directions. At each iteration it produces two directions in the original matrix-block space:
\begin{equation*}
G_Y^{k+1}:=\boldsymbol P^k\tilde{\boldsymbol V}^{k+1}, \quad G_Z^{k+1}:=\widehat{\mathcal H[\boldsymbol Z]}^{k+1}-\boldsymbol P^k\tilde{\boldsymbol U}_Y^{k+1}.
\end{equation*}
Algorithm~\ref{alg:rso_masoba} uses the plain stochastic steps $\boldsymbol Y^{k+1}=\boldsymbol Y^k-\beta_kG_Y^{k+1}$ and $\boldsymbol Z^{k+1}=\boldsymbol Z^k-\gamma_kG_Z^{k+1}$ because this is the recursion analyzed in Section~\ref{sec:convergence}. In implementation, these two directions can instead be passed to a momentum or adaptive optimizer.

For a momentum SGD-like implementation, maintain velocity variables $\boldsymbol M_Y^k$ and $\boldsymbol M_Z^k$ with the same block structure as $\boldsymbol Y$ and $\boldsymbol Z$. Given momentum parameters $\rho_Y,\rho_Z\in[0,1)$, the lower and auxiliary updates can be written as
\begin{equation*}
\boldsymbol M_Y^{k+1}:=\rho_Y\boldsymbol M_Y^k+(1-\rho_Y)G_Y^{k+1}, \quad \boldsymbol Y^{k+1}:=\boldsymbol Y^k-\beta_k\boldsymbol M_Y^{k+1},
\end{equation*}
\begin{equation*}
\boldsymbol M_Z^{k+1}:=\rho_Z\boldsymbol M_Z^k+(1-\rho_Z)G_Z^{k+1}, \quad \boldsymbol Z^{k+1}:=\boldsymbol Z^k-\gamma_k\boldsymbol M_Z^{k+1}.
\end{equation*}
Thus the only difference from Algorithm~\ref{alg:rso_masoba} is that the projected lower direction and the corrected auxiliary direction are accumulated through momentum before updating the persistent variables.

For an Adam-like implementation, use $G_Y^{k+1}$ and $G_Z^{k+1}$ as gradient-like inputs to first- and second-moment recursions. For $\boldsymbol Y$, with moment parameters $\rho_1,\rho_2\in[0,1)$, numerical constant $\epsilon_{\rm A}>0$, and decoupled weight decay $\lambda_Y\ge0$, one may use
\begin{equation*}
\boldsymbol m_Y^{k+1}:=\rho_1\boldsymbol m_Y^k+(1-\rho_1)G_Y^{k+1}, \quad \boldsymbol q_Y^{k+1}:=\rho_2\boldsymbol q_Y^k+(1-\rho_2)(G_Y^{k+1}\odot G_Y^{k+1}),
\end{equation*}
\begin{equation*}
\widehat{\boldsymbol m}_Y^{k+1}:=\frac{\boldsymbol m_Y^{k+1}}{1-\rho_1^{k+1}}, \quad \widehat{\boldsymbol q}_Y^{k+1}:=\frac{\boldsymbol q_Y^{k+1}}{1-\rho_2^{k+1}}, \quad \boldsymbol Y^{k+1}:=(1-\beta_k\lambda_Y)\boldsymbol Y^k-\beta_k\frac{\widehat{\boldsymbol m}_Y^{k+1}}{\sqrt{\widehat{\boldsymbol q}_Y^{k+1}}+\epsilon_{\rm A}}.
\end{equation*}
The auxiliary variable is updated analogously:
\begin{equation*}
\boldsymbol m_Z^{k+1}:=\rho_1\boldsymbol m_Z^k+(1-\rho_1)G_Z^{k+1}, \quad \boldsymbol q_Z^{k+1}:=\rho_2\boldsymbol q_Z^k+(1-\rho_2)(G_Z^{k+1}\odot G_Z^{k+1}),
\end{equation*}
\begin{equation*}
\widehat{\boldsymbol m}_Z^{k+1}:=\frac{\boldsymbol m_Z^{k+1}}{1-\rho_1^{k+1}}, \quad \widehat{\boldsymbol q}_Z^{k+1}:=\frac{\boldsymbol q_Z^{k+1}}{1-\rho_2^{k+1}}, \quad \boldsymbol Z^{k+1}:=(1-\gamma_k\lambda_Z)\boldsymbol Z^k-\gamma_k\frac{\widehat{\boldsymbol m}_Z^{k+1}}{\sqrt{\widehat{\boldsymbol q}_Z^{k+1}}+\epsilon_{\rm A}}.
\end{equation*}
All products, square roots, and divisions in these Adam-like formulas are understood elementwise within each matrix block. These choices do not change the Rademacher bi-probe formula or the blockwise corrected HVP estimator; they only change how the resulting directions are accumulated across iterations.

There are two implementation conventions. The most direct convention maintains optimizer states in the original matrix-block space and feeds the lifted directions $G_Y^{k+1}$ and $G_Z^{k+1}$ to the optimizer. This is simple and compatible with standard optimizer implementations. A more memory-oriented convention applies momentum or adaptive updates to the reduced variables inside the current projected subproblem. In that case, the optimizer state is tied to the sampled projector and must be reset, refreshed, or transported when the projector changes. This projected-state convention is natural for subspace training, but its exact behavior depends on the projector refresh schedule and the optimizer-state transport rule.

The convergence theorem in the main text is stated for the SGD-style recursion. The key unbiasedness property of the \ours directions remains an estimator property, but a convergence proof for momentum or adaptive variants would need additional assumptions controlling the optimizer states or preconditioners. We therefore treat these variants as implementation-compatible extensions and leave their nonasymptotic analysis to future work.

% Appendix (technical proofs)
\section{Theoretical analysis}
\label{app:full_proofs}

\subsection{Assumptions}
\begin{assumption}\label{as:A1-1}
There exist constants $\mu_g, L_{f, 0}, L_{f, 1}, L_{g, 1}, L_{g, 2}$ such that:
\begin{enumerate}
\item $\nabla f$ and $\nabla g$ are $L_{f,1}$- and $L_{g,1}$-Lipschitz continuous in $(x,\boldsymbol{Y})$, using Euclidean norm for $x$ and the product Frobenius norm for $\boldsymbol{Y}$;
\item $\nabla^2 g$ is $L_{g,2}$-Lipschitz continuous as a block Hessian operator on $\R^{d_x}\times\mathcal Y$;
\item $\|\nabla_{\boldsymbol{Y}} f(x,\boldsymbol{Y}^*(x))\| \leq L_{f,0}$ for all $x\in\R^{d_x}$;
\item for every fixed $x$, $\boldsymbol{Y}\mapsto g(x,\boldsymbol{Y})$ is $\mu_g$-strongly convex in the product Frobenius norm.
\end{enumerate}
Moreover, define $L_{\max}:=\max \{L_{f,0}, L_{f,1}, L_{g,1}, L_{g,2}\}$ and $\kappa:=L_{\max} / \mu_g$.
\end{assumption}

\begin{assumption}[Stochastic projected oracles and filtration]\label{as:A2-1}
For any $k \geq 0$, let
\begin{equation*}
\mathcal F_k
\!=\sigma\{x^0,\ldots,x^k,\ \boldsymbol{Y}^0,\ldots,\boldsymbol{Y}^k,\ \boldsymbol{Z}^0,\ldots,\boldsymbol{Z}^k,\ h^0,\ldots,h^k,\ \boldsymbol{P}^0,\ldots,\boldsymbol{P}^{k-1}\}, \quad
\mathcal F_k^{+}
\!=\sigma(\mathcal F_k \cup\{\boldsymbol{P}^k\}).
\end{equation*}
Define the projected functions on $\mathcal B:=\prod_{\ell=1}^L \R^{r_\ell\times n_\ell}$ by
\begin{equation*}
\tilde f_k(x,\boldsymbol{B}):=f(x,\boldsymbol{Y}^k+\boldsymbol{P}^k\boldsymbol{B}),\quad \tilde g_k(x,\boldsymbol{B}):=g(x,\boldsymbol{Y}^k+\boldsymbol{P}^k\boldsymbol{B}),
\end{equation*}
where $(\boldsymbol{P}^k\boldsymbol{B})_\ell:=P_\ell^k B_\ell$ and $(\boldsymbol{P}^{k\top}\boldsymbol{U})_\ell:=(P_\ell^k)^\top U_\ell$. The stochastic oracles $\tilde U_x^{k+1}\in\R^{d_x}, \tilde{\boldsymbol U}_Y^{k+1}\in\mathcal B, \tilde{\boldsymbol V}^{k+1}\in\mathcal B,$ and the stochastic linear operators $\tilde{\mathcal H}^{k+1}:\mathcal B\to\mathcal B, \tilde{\mathcal J}^{k+1}:\mathcal B\to\R^{d_x}$ are conditionally unbiased given $\mathcal F_k^+$:
\begin{equation*}
\begin{aligned}
\E[\tilde U_x^{k+1}\mid \mathcal F_k^+]&=\nabla_x \tilde f_k(x^k,0),\\
\E[\tilde{\boldsymbol U}_Y^{k+1}\mid \mathcal F_k^+]&=\nabla_{\boldsymbol B} \tilde f_k(x^k,\mathbf 0),\\
\E[\tilde{\boldsymbol V}^{k+1}\mid \mathcal F_k^+]&=\nabla_{\boldsymbol B} \tilde g_k(x^k,\mathbf 0),\\
\E[\tilde{\mathcal H}^{k+1}[W]\mid \mathcal F_k^+]
\!&=\nabla_{\boldsymbol B\boldsymbol B}^2\tilde g_k(x^k,\mathbf 0)[W], \quad \forall W\in\mathcal B,\\
\E[\tilde{\mathcal J}^{k+1}[W]\mid \mathcal F_k^+]
\!&=\nabla_{x\boldsymbol B}^2\tilde g_k(x^k,\mathbf 0)[W], \quad \forall W\in\mathcal B.
\end{aligned}
\end{equation*}
Moreover, there exist constants $\tilde{\sigma}_{f,1}, \tilde{\sigma}_{f,2}, \tilde{\sigma}_{g,1}, \tilde{\sigma}_{g,2}\ge 0$ such that
\begin{equation*}
\begin{aligned}
\E [ \| \tilde U_x^{k+1} - \nabla_x \tilde f_k(x^k, 0) \|^2 \mid \mathcal F_k^+ ] &\leq \tilde{\sigma}_{f,1}^2,\\
\E [ \| \tilde{\boldsymbol U}_Y^{k+1} - \nabla_{\boldsymbol B} \tilde f_k(x^k, \mathbf 0) \|^2 \mid \mathcal F_k^+ ] &\leq \tilde{\sigma}_{f,2}^2,\\
\E [ \| \tilde{\boldsymbol V}^{k+1} - \nabla_{\boldsymbol B} \tilde g_k(x^k, \mathbf 0) \|^2 \mid \mathcal F_k^+ ] &\leq \tilde{\sigma}_{g,1}^2,\\
\E [ \| \tilde{\mathcal H}^{k+1}[W] - \nabla_{\boldsymbol B\boldsymbol B}^2 \tilde g_k(x^k, \mathbf 0)[W] \|^2 \mid \mathcal F_k^+ ] &\leq \tilde{\sigma}_{g,2}^2\|W\|^2, \quad \forall W\in\mathcal B,\\
\E [ \| \tilde{\mathcal J}^{k+1}[W] - \nabla_{x\boldsymbol B}^2 \tilde g_k(x^k, \mathbf 0)[W] \|^2 \mid \mathcal F_k^+ ] &\leq \tilde{\sigma}_{g,2}^2\|W\|^2, \quad \forall W\in\mathcal B.
\end{aligned}
\end{equation*}
In the algorithm, the layerwise probes $\{u_\ell^{k+1},\xi_\ell^{k+1}\}_{\ell=1}^L$ are sampled independently of $\mathcal F_k^+$ and independently of the random oracle objects above. Define $\mathcal F_k^{++}:=\sigma\!\left(\mathcal F_k^+,\{u_\ell^{k+1},\xi_\ell^{k+1}\}_{\ell=1}^L\right)$. Conditioned on $\mathcal F_k^+$, the random objects $\tilde U_x^{k+1},\tilde{\boldsymbol U}_Y^{k+1},\tilde{\boldsymbol V}^{k+1},\tilde{\mathcal J}^{k+1},\tilde{\mathcal H}^{k+1}$ are mutually independent. The operators $\tilde{\mathcal H}^{k+1}$ and $\tilde{\mathcal J}^{k+1}$ are sampled once at iteration $k$ and may be queried multiple times within the same iteration; no additional conditional independence is assumed between distinct queries to the same operator realization. Moreover, the displayed first- and second-moment relations for $\tilde{\mathcal H}^{k+1}$ and $\tilde{\mathcal J}^{k+1}$ extend to every $\mathcal F_k^{++}$-measurable random direction $W\in\mathcal B$ with finite conditional second moment.
\end{assumption}

\begin{assumption}[Random orthogonal projection]\label{as:A3-1}
Conditional on $\mathcal F_k$, the matrices $\{P_\ell^k\}_{\ell=1}^L$ are independent and satisfy
\begin{equation*}
P_\ell^k=\sqrt{\frac{m_\ell}{r_\ell}}\,O_{\ell,r_\ell}^k,
\end{equation*}
where $O_{\ell,r_\ell}^k$ is formed by the first $r_\ell$ columns of a Haar-uniform $O_\ell^k\in O(m_\ell)$. Consequently, ${(P_\ell^k)}^{\top} P_\ell^k=(m_\ell/r_\ell)I_{r_\ell}$ and $\E[P_\ell^k{(P_\ell^k)}^\top\mid\mathcal F_k]=I_{m_\ell}$. We assume $r_\ell\ge 2$ for all $\ell$.
\end{assumption}

\subsection{Proof roadmap}
This appendix proves the two claims stated in Section~\ref{sec:convergence}: the moment-preservation lemma, which justifies the bias-corrected projected directions, and the main convergence theorem, which gives the $\mathcal O(\varepsilon^{-2})$ sample complexity with projection-ratio dependent constants. For the convergence part, we use the Lyapunov quantity
\begin{equation*}
V_k:=\|x_+^k-x^k\|^2+\|h^k-\nabla\Phiobj(x^k)\|^2,
\end{equation*}
where $x_+^k:=x^k-h^k$. It is enough to prove $\frac{1}{K}\sum_{k=0}^{K-1}\E V_k\le C_0/(K\alpha)+C_1\alpha$, because $\|\nabla\Phiobj(x^k)\|^2\le 2V_k$. The proof is organized as follows.

\begin{itemize}[leftmargin=*]
\item \textbf{Step 1 (bias-corrected moment preservation). } Lemmas~\ref{lem:smoothness_consequences}--\ref{lem:unbiased_lifted_oracles} establish the technical bridge behind Lemma~\ref{lem:bros_moment_preservation}. The standard implicit-function identities are first extended to the product Frobenius space. The main new point is the auxiliary Hessian action: Lemma~\ref{lem:projection_moments_matrix} identifies the projection distortion in the self-layer blocks, Lemma~\ref{lem:probe_cancellation} shows that the Rademacher bi-probe does not introduce cross-layer mean errors, and Lemma~\ref{lem:unbiased_lifted_oracles} proves that the final \ours directions satisfy
\begin{equation*}
\E[\widehat{\boldsymbol V}^{k+1}\mid\mathcal F_k]=D_Y^k,\quad \E[\widehat{\boldsymbol S}^{k+1}\mid\mathcal F_k]=D_Z^k,\quad \E[\widehat W^{k+1}\mid\mathcal F_k]=D_x^k,
\end{equation*}
with conditional second moments controlled by the tracking errors. The leading projection-ratio dependence is tracked through $\omega_P=\max_\ell m_\ell/r_\ell$, while fixed dimension- and model-dependent constants are absorbed into the generic constants.

\item \textbf{Step 2 (tracking the lower and auxiliary solutions). } With the centered directions from Step 1, Lemma~\ref{lem:yz_tracking_rso} proves that the iterates $\boldsymbol Y^k$ and $\boldsymbol Z^k$ track the moving targets $\boldsymbol Y^*(x^k)$ and $\boldsymbol Z^*(x^k)$. In weighted-sum form, the tracking errors are bounded by the primal progress term $\sum_k\alpha\E\|x_+^k-x^k\|^2$, initialization constants, and stochastic terms of order $K\alpha^2$.

\item \textbf{Step 3 (hypergradient tracking). } Lemma~\ref{lem:hypergrad_bias} converts the lower/auxiliary tracking errors from Step 2 into a bound on the bias between the centered estimator $\E[\widehat W^{k+1}\mid\mathcal F_k]$ and the true hypergradient $\nabla\Phiobj(x^k)$. Lemma~\ref{lem:variance_rso} controls the variance of $\widehat W^{k+1}$, and Lemma~\ref{lem:h_tracking_rso} turns these estimates into a recursion for the moving-average error $S_h:=\sum_{k=0}^{K-1}\alpha\E\|h^k-\nabla\Phiobj(x^k)\|^2$:
\begin{equation*}
S_h\le A_0+\rho_0S_x+A_1K\alpha^2+A_2\alpha,
\end{equation*}
where $S_x:=\sum_{k=0}^{K-1}\alpha\E\|x_+^k-x^k\|^2$ and $\rho_0$ is defined in \eqref{eq:def_rho0_matrix_thm}.

\item \textbf{Step 4 (closing the descent recursion). } Lemma~\ref{lem:primal_sum_h} gives the descent relation $S_x\le 4(\Phiobj(x^0)-\Phiobj^*)+2S_h$. Substituting this into the Step 3 recursion and choosing the coupled stepsizes so that $\rho_0$ and the remaining $\alpha$-dependent coefficient are small allows the $S_h$ term to be absorbed. This yields
\begin{equation*}
S_h+S_x\le \mathcal O(1)+\mathcal O(K\alpha^2).
\end{equation*}
Dividing by $K\alpha$ gives the formal bound in Theorem~\ref{thm:rso_masoba_convergence}. After absorbing fixed dimension- and model-dependent quantities into generic constants, the displayed projection-ratio dependence in the main-text form of Theorem~\ref{thm:bros_main} is expressed through $\omega_P$: the transient and stochastic-error terms scale as $1+\omega_P^2$ and $1+\omega_P^4$, respectively, and the admissible stepsize is $\alpha\le \bar\alpha/(1+\omega_P^3)$. Choosing $\alpha=\Theta(K^{-1/2})$ then gives $\E\|\nabla\Phiobj(x^R)\|^2=\mathcal O((1+\omega_P^4)K^{-1/2})$, hence $K=\mathcal O((1+\omega_P^4)^2\varepsilon^{-2})$ for an $\varepsilon$-stationary point.
\end{itemize}

\subsection{Supporting lemmas}

\begin{lemma}[Product-space smoothness consequences]\label{lem:smoothness_consequences}
Suppose that \emph{(A1)} holds. Let $\boldsymbol Y^*(x)$ be the unique minimizer of $\boldsymbol Y\mapsto g(x,\boldsymbol Y)$ and define
\begin{equation}\label{eq:def_Zstar}
\boldsymbol Z^*(x):=\mathcal H(x,\boldsymbol Y^*(x))^{-1}[\nabla_{\boldsymbol Y}f(x,\boldsymbol Y^*(x))]\in\mathcal Y.
\end{equation}
Then the following statements hold.
\begin{enumerate}
\item \textbf{Implicit hypergradient formula.}
The bilevel objective $\Phiobj(x)=f(x,\boldsymbol Y^*(x))$ is differentiable and
\begin{equation}\label{eq:hypergrad_formula}
\nabla\Phiobj(x)=\nabla_x f(x,\boldsymbol Y^*(x))-\mathcal J(x,\boldsymbol Y^*(x))[\boldsymbol Z^*(x)].
\end{equation}

\item \textbf{Lipschitz continuity of $\boldsymbol Y^*(\cdot)$.}
The mapping $\boldsymbol Y^*(\cdot)$ is Lipschitz continuous with
\begin{equation}\label{eq:LY_star}
\|\boldsymbol Y^*(x)-\boldsymbol Y^*(\tilde x)\|\le L_{Y^*}\|x-\tilde x\|,
\quad
L_{Y^*}:=\frac{L_{g,1}}{\mu_g}\le \kappa.
\end{equation}

\item \textbf{Lipschitz continuity of $\boldsymbol Z^*(\cdot)$.}
The mapping $\boldsymbol Z^*(\cdot)$ is Lipschitz continuous with
\begin{equation}\label{eq:LZ_star}
\|\boldsymbol Z^*(x)-\boldsymbol Z^*(\tilde x)\|
\le L_{Z^*}\|x-\tilde x\|,
\quad
L_{Z^*}:=\sqrt{1+L_{Y^*}^2}\left(\frac{L_{f,1}}{\mu_g}+\frac{L_{f,0}L_{g,2}}{\mu_g^2}\right).
\end{equation}
In particular, if $\kappa\ge 1$ then $L_{Z^*}\le 2\sqrt{2}\,\kappa^3$.

\item \textbf{Lipschitz continuity of $\nabla\Phiobj$.}
The gradient $\nabla\Phiobj$ is Lipschitz continuous with
\begin{equation}\label{eq:Lnabla_Phi}
\|\nabla\Phiobj(x)-\nabla\Phiobj(\tilde x)\|
\le L_{\nabla\Phi}\|x-\tilde x\|,
\end{equation}
where one may take
\begin{equation}\label{eq:Lnabla_Phi_expr}
L_{\nabla\Phi}:=
\sqrt{1+L_{Y^*}^2}\left(
L_{f,1}+\frac{L_{f,0}L_{g,2}}{\mu_g}
+\frac{L_{g,1}L_{f,1}}{\mu_g}
+\frac{L_{g,1}L_{f,0}L_{g,2}}{\mu_g^2}
\right).
\end{equation}

\item \textbf{Uniform bound on $\|\boldsymbol Z^*(x)\|$.}
For all $x$,
\begin{equation}\label{eq:Zstar_bound}
\|\boldsymbol Z^*(x)\|\le \frac{L_{f,0}}{\mu_g}.
\end{equation}
\end{enumerate}
\end{lemma}

\begin{proof}
We endow $\R^{d_x}\times\mathcal Y$ with the product norm $\|(x,\boldsymbol Y)\|:=\sqrt{\|x\|^2+\|\boldsymbol Y\|^2}$.

\textbf{Step 1: implicit differentiation and \eqref{eq:hypergrad_formula}. }
Fix $x$ and write $\boldsymbol Y^*:=\boldsymbol Y^*(x)$. Define
\begin{equation*}
F(x,\boldsymbol Y):=\nabla_{\boldsymbol Y}g(x,\boldsymbol Y)\in\mathcal Y.
\end{equation*}
By Assumption~\ref{as:A1}, $g$ is twice continuously differentiable, so $F$ is continuously differentiable with
\begin{equation*}
\nabla_{\boldsymbol Y}F(x,\boldsymbol Y)=\mathcal H(x,\boldsymbol Y), \quad \nabla_x F(x,\boldsymbol Y)=\nabla_{\boldsymbol Y x}^2 g(x,\boldsymbol Y).
\end{equation*}
Since $\boldsymbol Y\mapsto g(x,\boldsymbol Y)$ is $\mu_g$-strongly convex in the product Frobenius norm, its Hessian is uniformly positive definite:
\begin{equation}\label{eq:H_sc_product}
\langle \boldsymbol W,\mathcal H(x,\boldsymbol Y)[\boldsymbol W]\rangle
\ge \mu_g \|\boldsymbol W\|^2,
\quad
\forall \boldsymbol W\in\mathcal Y,\ \forall (x,\boldsymbol Y).
\end{equation}
Hence $\mathcal H(x,\boldsymbol Y^*)$ is invertible. Applying the implicit function theorem to $F(x,\boldsymbol Y)=0$ at $(x,\boldsymbol Y^*)$ shows that $\boldsymbol Y^*(\cdot)$ is continuously differentiable locally, and differentiating $F(x,\boldsymbol Y^*(x))\equiv 0$ in direction $u\in\R^{d_x}$ gives
\begin{equation*}
0=\nabla_{\boldsymbol Y x}^2 g(x,\boldsymbol Y^*)[u] +\mathcal H(x,\boldsymbol Y^*)[D\boldsymbol Y^*(x)[u]],
\end{equation*}
so
\begin{equation}\label{eq:DYstar_formula_product}
D\boldsymbol Y^*(x)[u]
=-\mathcal H(x,\boldsymbol Y^*)^{-1}[\nabla_{\boldsymbol Y x}^2 g(x,\boldsymbol Y^*)[u]].
\end{equation}

Now compute the directional derivative of $\Phiobj(x)=f(x,\boldsymbol Y^*(x))$:
\begin{equation*}
D\Phiobj(x)[u] =\langle \nabla_x f(x,\boldsymbol Y^*),u\rangle +\langle \nabla_{\boldsymbol Y}f(x,\boldsymbol Y^*),D\boldsymbol Y^*(x)[u]\rangle.
\end{equation*}
Substituting \eqref{eq:DYstar_formula_product} and using that $\mathcal H(x,\boldsymbol Y^*)$ is self-adjoint on $\mathcal Y$ (hence so is its inverse) yields
\begin{equation*}
D\Phiobj(x)[u]
={}\langle \nabla_x f(x,\boldsymbol Y^*),u\rangle -\Bigl\langle
\mathcal H(x,\boldsymbol Y^*)^{-1}[\nabla_{\boldsymbol Y}f(x,\boldsymbol Y^*)],
\nabla_{\boldsymbol Y x}^2 g(x,\boldsymbol Y^*)[u]
\Bigr\rangle.
\end{equation*}
We define $\mathcal J(x,\boldsymbol Y):\mathcal Y\to\R^{d_x}$ as the adjoint of $u\mapsto \nabla_{\boldsymbol Y x}^2 g(x,\boldsymbol Y)[u]$, i.e., for all $u\in\R^{d_x}$ and $\boldsymbol W\in\mathcal Y$,
\begin{equation}\label{eq:adjoint_J_product}
\langle \nabla_{\boldsymbol Y x}^2 g(x,\boldsymbol Y)[u],\boldsymbol W\rangle
=\langle u,\mathcal J(x,\boldsymbol Y)[\boldsymbol W]\rangle.
\end{equation}
Recalling the definition \eqref{eq:def_Zstar}, we obtain
\begin{equation*}
D\Phiobj(x)[u] =\langle \nabla_x f(x,\boldsymbol Y^*)-\mathcal J(x,\boldsymbol Y^*)[\boldsymbol Z^*(x)],u\rangle,
\end{equation*}
which proves \eqref{eq:hypergrad_formula}.

\textbf{Step 2: Lipschitz continuity of $\boldsymbol Y^*(\cdot)$. }
Fix $x,\tilde x\in\R^{d_x}$. Using the optimality conditions
\begin{equation*}
\nabla_{\boldsymbol Y}g(x,\boldsymbol Y^*(x))=0, \quad \nabla_{\boldsymbol Y}g(\tilde x,\boldsymbol Y^*(\tilde x))=0,
\end{equation*}
and the strong monotonicity of $\boldsymbol Y\mapsto \nabla_{\boldsymbol Y}g(x,\boldsymbol Y)$, we get
\begin{equation*}
\begin{aligned}
\mu_g\|\boldsymbol Y^*(x)-\boldsymbol Y^*(\tilde x)\|
\le{}& \|\nabla_{\boldsymbol Y}g(x,\boldsymbol Y^*(\tilde x))-\nabla_{\boldsymbol Y}g(x,\boldsymbol Y^*(x))\| \\
={}& \|\nabla_{\boldsymbol Y}g(\tilde x,\boldsymbol Y^*(\tilde x))-\nabla_{\boldsymbol Y}g(x,\boldsymbol Y^*(\tilde x))\|.
\end{aligned}
\end{equation*}
By the $L_{g,1}$-Lipschitz continuity of $\nabla g$ in Assumption~\ref{as:A1}, the last term is at most $L_{g,1}\|x-\tilde x\|$, which proves \eqref{eq:LY_star}.

\textbf{Step 3: Lipschitz continuity and the bound of $\boldsymbol Z^*(\cdot)$. }
Define
\begin{equation*}
A(x):=\mathcal H(x,\boldsymbol Y^*(x)), \quad b(x):=\nabla_{\boldsymbol Y}f(x,\boldsymbol Y^*(x)), \quad \boldsymbol Z^*(x)=A(x)^{-1}[b(x)].
\end{equation*}
From \eqref{eq:H_sc_product}, $\|A(x)^{-1}\|\le 1/\mu_g$. Assumption~\ref{as:A1} gives $\|b(x)\|\le L_{f,0}$, so $\|\boldsymbol Z^*(x)\|\le \frac{L_{f,0}}{\mu_g}$.

Next, by Assumption~\ref{as:A1} and \eqref{eq:LY_star},
\begin{equation}\label{eq:b_Lipschitz_product}
\|b(x)-b(\tilde x)\|
\le
L_{f,1}\|(x,\boldsymbol Y^*(x))-(\tilde x,\boldsymbol Y^*(\tilde x))\|
\le
L_{f,1}\sqrt{1+L_{Y^*}^2}\|x-\tilde x\|.
\end{equation}
Similarly, since $\nabla^2 g$ is $L_{g,2}$-Lipschitz,
\begin{equation}\label{eq:A_Lipschitz_product}
\|A(x)-A(\tilde x)\|
\le
L_{g,2}\|(x,\boldsymbol Y^*(x))-(\tilde x,\boldsymbol Y^*(\tilde x))\|
\le
L_{g,2}\sqrt{1+L_{Y^*}^2}\|x-\tilde x\|.
\end{equation}
Using the resolvent identity $A(x)^{-1}-A(\tilde x)^{-1} = A(x)^{-1}(A(\tilde x)-A(x))A(\tilde x)^{-1}$, we have
\begin{equation}\label{eq:Ainv_Lipschitz_product}
\|A(x)^{-1}-A(\tilde x)^{-1}\|
\le
\frac{1}{\mu_g^2}\|A(x)-A(\tilde x)\|.
\end{equation}
Decomposing
\begin{equation*}
\boldsymbol Z^*(x)-\boldsymbol Z^*(\tilde x) = A(x)^{-1}[b(x)-b(\tilde x)] + (A(x)^{-1}-A(\tilde x)^{-1})[b(\tilde x)]
\end{equation*}
and combining \eqref{eq:b_Lipschitz_product}--\eqref{eq:Ainv_Lipschitz_product} with $\|b(\tilde x)\|\le L_{f,0}$ yields \eqref{eq:LZ_star}.

\textbf{Step 4: Lipschitz continuity of $\nabla\Phiobj$. }
From \eqref{eq:hypergrad_formula},
\begin{equation*}
\nabla\Phiobj(x) = \nabla_x f(x,\boldsymbol Y^*(x)) -\mathcal J(x,\boldsymbol Y^*(x))[\boldsymbol Z^*(x)].
\end{equation*}
Because $\nabla g$ is $L_{g,1}$-Lipschitz, the full Hessian of $g$ has operator norm at most $L_{g,1}$; in particular,
\begin{equation}\label{eq:J_bound}
\|\mathcal J(x,\boldsymbol Y)\|\le L_{g,1},
\quad
\forall (x,\boldsymbol Y)\in\R^{d_x}\times\mathcal Y.
\end{equation}
Likewise, the $L_{g,2}$-Lipschitz continuity of $\nabla^2 g$ implies
\begin{equation}\label{eq:J_Lipschitz}
\|\mathcal J(x,\boldsymbol Y)-\mathcal J(\tilde x,\tilde{\boldsymbol Y})\|
\le
L_{g,2}\|(x,\boldsymbol Y)-(\tilde x,\tilde{\boldsymbol Y})\|.
\end{equation}
Therefore,
\begin{equation*}
\begin{aligned}
\|\nabla\Phiobj(x)-\nabla\Phiobj(\tilde x)\|
\le{}&
\|\nabla_x f(x,\boldsymbol Y^*(x))-\nabla_x f(\tilde x,\boldsymbol Y^*(\tilde x))\| \\
&+
\|\mathcal J(x,\boldsymbol Y^*(x))[\boldsymbol Z^*(x)]
-\mathcal J(\tilde x,\boldsymbol Y^*(\tilde x))[\boldsymbol Z^*(\tilde x)]\|.
\end{aligned}
\end{equation*}
The first term is bounded by $L_{f,1}\sqrt{1+L_{Y^*}^2}\|x-\tilde x\|$. For the second term, add and subtract $\mathcal J(\tilde x,\boldsymbol Y^*(\tilde x))[\boldsymbol Z^*(x)]$:
\begin{equation*}
\begin{aligned}
\|\mathcal J(x,\boldsymbol Y^*(x))[\boldsymbol Z^*(x)]
-\mathcal J(\tilde x,\boldsymbol Y^*(\tilde x))[\boldsymbol Z^*(\tilde x)]\| 
\le{}&
\|\mathcal J(x,\boldsymbol Y^*(x))-\mathcal J(\tilde x,\boldsymbol Y^*(\tilde x))\|\,
\|\boldsymbol Z^*(x)\| \\
&+\|\mathcal J(\tilde x,\boldsymbol Y^*(\tilde x))\|\,
\|\boldsymbol Z^*(x)-\boldsymbol Z^*(\tilde x)\|.
\end{aligned}
\end{equation*}
Using \eqref{eq:J_Lipschitz}, \eqref{eq:J_bound}, and \eqref{eq:LZ_star} yields \eqref{eq:Lnabla_Phi} with the constant \eqref{eq:Lnabla_Phi_expr}.
\end{proof}

\begin{lemma}[Layerwise projection identities and induced derivatives]\label{lem:projection_moments_matrix}
For each layer $\ell\in\{1,\ldots,L\}$, let
\begin{equation*}
P_\ell:=\sqrt{\frac{m_\ell}{r_\ell}}\,O_{\ell,r_\ell}\in\R^{m_\ell\times r_\ell}, \quad Q_\ell:=P_\ell P_\ell^\top,
\end{equation*}
where $O_{\ell,r_\ell}$ contains the first $r_\ell$ columns of a Haar-uniform matrix in $O(m_\ell)$. Assume the projectors are independent across layers. Then:
\begin{enumerate}
\item[(i)] For every $\ell$, we have $P_\ell^\top P_\ell=\frac{m_\ell}{r_\ell}I_{r_\ell}$ and $\E[Q_\ell]=I_{m_\ell}$.
\item[(ii)] For every layer $\ell$, there exist constants
\begin{equation}\label{eq:ab_proj_matrix}
a_\ell=\frac{m_\ell(m_\ell r_\ell+m_\ell-2)}{r_\ell(m_\ell-1)(m_\ell+2)},
\quad
b_\ell=\frac{m_\ell(m_\ell-r_\ell)}{r_\ell(m_\ell-1)(m_\ell+2)}
\end{equation}
such that for any deterministic $A\in\R^{m_\ell\times m_\ell}$,
\begin{equation}\label{eq:Weingarten_matrix_form}
\E[Q_\ell A Q_\ell]
=(a_\ell-b_\ell)A+b_\ell A^\top+b_\ell\tr(A)I_{m_\ell},
\end{equation}
and
\begin{equation}\label{eq:a_minus_2b}
a_\ell-2b_\ell=\frac{m_\ell(r_\ell-1)}{r_\ell(m_\ell-1)}>0.
\end{equation}
\item[(iii)] For any pair $\ell\neq t$ and any deterministic $A\in\R^{m_\ell\times m_t}$,
\begin{equation}\label{eq:cross_layer_proj_identity}
\E[Q_\ell A Q_t]=A.
\end{equation}
\item[(iv)] Fix an iteration $k$ and define
\begin{equation*}
\tilde f_k(x,\boldsymbol B):=f(x,\boldsymbol Y^k+\boldsymbol P^k\boldsymbol B), \quad \tilde g_k(x,\boldsymbol B):=g(x,\boldsymbol Y^k+\boldsymbol P^k\boldsymbol B).
\end{equation*}
Then at $\boldsymbol B=\mathbf 0$,
\begin{equation*}
\nabla_{\boldsymbol B}\tilde f_k(x,\mathbf 0)=\boldsymbol P^{k\top}\nabla_{\boldsymbol Y}f(x,\boldsymbol Y^k), \quad \nabla_{\boldsymbol B}\tilde g_k(x,\mathbf 0)=\boldsymbol P^{k\top}\nabla_{\boldsymbol Y}g(x,\boldsymbol Y^k),
\end{equation*}
and for every $\boldsymbol W\in\mathcal B$,
\begin{equation*}
\begin{aligned}
\nabla_{\boldsymbol B\boldsymbol B}^2\tilde g_k(x,\mathbf 0)[\boldsymbol W] &=\boldsymbol P^{k\top}\mathcal H(x,\boldsymbol Y^k)[\boldsymbol P^k\boldsymbol W], \\
\nabla_{x\boldsymbol B}^2\tilde g_k(x,\mathbf 0)[\boldsymbol W] &=\mathcal J(x,\boldsymbol Y^k)[\boldsymbol P^k\boldsymbol W].
\end{aligned}
\end{equation*}
\end{enumerate}
\end{lemma}

\begin{proof}
\textbf{Proof of (i). }
Fix a layer $\ell$. Since the columns of $O_{\ell,r_\ell}$ are orthonormal,
\begin{equation*}
P_\ell^\top P_\ell = \frac{m_\ell}{r_\ell}O_{\ell,r_\ell}^\top O_{\ell,r_\ell} = \frac{m_\ell}{r_\ell}I_{r_\ell}.
\end{equation*}
To prove $\E[Q_\ell]=I_{m_\ell}$, let $V\in O(m_\ell)$ be deterministic. By left-invariance of Haar measure, $VO_{\ell,r_\ell}$ has the same distribution as $O_{\ell,r_\ell}$, hence
\begin{equation*}
V\,\E[Q_\ell]\,V^\top=\E[VQ_\ell V^\top]=\E[Q_\ell].
\end{equation*}
Therefore $\E[Q_\ell]$ commutes with every orthogonal matrix and must equal $cI_{m_\ell}$ for some scalar $c$. Taking traces and using $P_\ell^\top P_\ell=(m_\ell/r_\ell)I_{r_\ell}$ gives
\begin{equation*}
cm_\ell = \tr(\E[Q_\ell]) = \E[\tr(P_\ell^\top P_\ell)] = \tr\!\left(\frac{m_\ell}{r_\ell}I_{r_\ell}\right) =m_\ell,
\end{equation*}
so $c=1$.

\textbf{Proof of (ii). }
Again fix $\ell$ and abbreviate
\begin{equation*}
m:=m_\ell,\quad r:=r_\ell,\quad P:=P_\ell,\quad Q:=Q_\ell.
\end{equation*}
Let $Q_0:=O_rO_r^\top$, so $Q=(m/r)Q_0$. By conjugation invariance of $Q_0$, the fourth-order tensor
\begin{equation*}
T_{iabj}:=\E[(Q_0)_{ia}(Q_0)_{bj}]
\end{equation*}
is $O(m)$-invariant and hence has the form
\begin{equation*}
\E[(Q_0)_{ia}(Q_0)_{bj}] = c_1\delta_{ia}\delta_{bj} +c_2\delta_{ib}\delta_{aj} +c_3\delta_{ij}\delta_{ab}.
\end{equation*}
To identify the coefficients, note first that for a Haar row vector $u\in\mathbb S^{m-1}$,
\begin{equation*}
(Q_0)_{11}=\sum_{j=1}^r u_j^2,
\end{equation*}
so the standard sphere moments give $\E[(Q_0)_{11}^2]=\frac{r(r+2)}{m(m+2)}$. Using $\tr(Q_0^2)=\tr(Q_0)=r$ together with rotational invariance yields $\E[(Q_0)_{12}^2]=\frac{r(m-r)}{m(m-1)(m+2)}$. Evaluating the invariant-tensor form at $(i,a,b,j)=(1,2,1,2)$, $(1,2,2,1)$, and $(1,1,1,1)$ shows that
\begin{equation*}
c_2=c_3=\frac{r(m-r)}{m(m-1)(m+2)}, \quad c_1=\frac{r(r(m+1)-2)}{m(m-1)(m+2)}.
\end{equation*}
Hence for any $A\in\R^{m\times m}$,
\begin{equation*}
\E[Q_0AQ_0] = c_1A+c_2A^\top+c_2\tr(A)I_m.
\end{equation*}
Multiplying by $(m/r)^2$ and simplifying gives
\begin{equation*}
\E[QAQ] = (a_\ell-b_\ell)A+b_\ell A^\top+b_\ell\tr(A)I_{m_\ell},
\end{equation*}
with
\begin{equation*}
a_\ell=\frac{m_\ell(m_\ell r_\ell+m_\ell-2)}{r_\ell(m_\ell-1)(m_\ell+2)}, \quad b_\ell=\frac{m_\ell(m_\ell-r_\ell)}{r_\ell(m_\ell-1)(m_\ell+2)}.
\end{equation*}
The formula $a_\ell-2b_\ell = \frac{m_\ell(r_\ell-1)}{r_\ell(m_\ell-1)}>0$ follows by direct algebra.

\textbf{Proof of (iii). }
For $\ell\neq t$ and deterministic $A\in\R^{m_\ell\times m_t}$, independence across layers gives
\begin{equation*}
\E[Q_\ell A Q_t] = \E[Q_\ell]\,A\,\E[Q_t] = A,
\end{equation*}
where we used (i).

\textbf{Proof of (iv). }
Fix $k$ and define $\Psi(\boldsymbol B):=\boldsymbol Y^k+\boldsymbol P^k\boldsymbol B$. For any $\boldsymbol H\in\mathcal B$,
\begin{equation*}
\begin{aligned}
\langle \nabla_{\boldsymbol B}\tilde g_k(x,\boldsymbol B),\boldsymbol H\rangle
&=
\left.\frac{\mathrm d}{\mathrm dt}g\bigl(x,\Psi(\boldsymbol B+t\boldsymbol H)\bigr)\right|_{t=0} \\
&=
\langle \nabla_{\boldsymbol Y}g(x,\Psi(\boldsymbol B)),\boldsymbol P^k\boldsymbol H\rangle =
\langle \boldsymbol P^{k\top}\nabla_{\boldsymbol Y}g(x,\Psi(\boldsymbol B)),\boldsymbol H\rangle,
\end{aligned}
\end{equation*}
so
\begin{equation*}
\nabla_{\boldsymbol B}\tilde g_k(x,\boldsymbol B)=\boldsymbol P^{k\top}\nabla_{\boldsymbol Y}g(x,\Psi(\boldsymbol B)).
\end{equation*}
Evaluating at $\boldsymbol B=\mathbf 0$ proves the formula for $\tilde g_k$; the proof for $\tilde f_k$ is identical.

For the Hessian, the directional derivative satisfies
\begin{equation*}
\nabla_{\boldsymbol B\boldsymbol B}^2\tilde g_k(x,\boldsymbol B)[\boldsymbol W] = \left.\frac{\mathrm d}{\mathrm dt}\nabla_{\boldsymbol B}\tilde g_k(x,\boldsymbol B+t\boldsymbol W)\right|_{t=0} = \boldsymbol P^{k\top}\mathcal H(x,\Psi(\boldsymbol B))[\boldsymbol P^k\boldsymbol W],
\end{equation*}
which gives the first identity at $\boldsymbol B=\mathbf 0$. Similarly,
\begin{equation*}
\nabla_{x\boldsymbol B}^2\tilde g_k(x,\boldsymbol B)[\boldsymbol W] = \left.\frac{\mathrm d}{\mathrm dt}\nabla_x g\bigl(x,\Psi(\boldsymbol B+t\boldsymbol W)\bigr)\right|_{t=0} = \mathcal J(x,\Psi(\boldsymbol B))[\boldsymbol P^k\boldsymbol W],
\end{equation*}
and evaluating at $\boldsymbol B=\mathbf 0$ yields the claim.
\end{proof}

\begin{lemma}[Cross-layer cancellation in the global probe]\label{lem:probe_cancellation}
Fix an iteration $k$ and a layer $t\in\{1,\ldots,L\}$. Let $E_j:\R^{r_j\times n_j}\to\mathcal B$ denote the block-insertion operator into block $j$, define
\begin{equation*}
\Delta B_j^{k+1}:=\xi_j^{k+1}(u_j^{k+1})^\top, \quad \Delta\boldsymbol B^{k+1}:=(\Delta B_1^{k+1},\ldots,\Delta B_L^{k+1}),
\end{equation*}
and let $\boldsymbol M^{k+1}:=\tilde{\mathcal H}^{k+1}[\Delta\boldsymbol B^{k+1}]$. Write the $t$-th output block as
\begin{equation*}
M_t^{k+1}=M_t^{\mathrm{self},k+1}+M_t^{\mathrm{cross},k+1},
\end{equation*}
where $M_t^{\mathrm{self},k+1}:=\bigl(\tilde{\mathcal H}^{k+1}[E_t(\Delta B_t^{k+1})]\bigr)_t$ and $M_t^{\mathrm{cross},k+1}:=\sum_{j\neq t}\bigl(\tilde{\mathcal H}^{k+1}[E_j(\Delta B_j^{k+1})]\bigr)_t$. Let
\begin{equation*}
\tilde Z_t^k:=(P_t^k)^\top Z_t^k, \quad v_t^{\mathrm{self},k+1}:=(M_t^{\mathrm{self},k+1})^\top\xi_t^{k+1}, \quad v_t^{\mathrm{cross},k+1}:=(M_t^{\mathrm{cross},k+1})^\top\xi_t^{k+1},
\end{equation*}
\begin{equation*}
w_t^{\mathrm{self},k+1}:=(M_t^{\mathrm{self},k+1})^\top(\tilde Z_t^k u_t^{k+1}), \quad w_t^{\mathrm{cross},k+1}:=(M_t^{\mathrm{cross},k+1})^\top(\tilde Z_t^k u_t^{k+1}),
\end{equation*}
and define
\begin{equation*}
\mathcal G_{k,t}:=\sigma\!\left(\mathcal F_k^+,\tilde{\mathcal H}^{k+1},\boldsymbol Z^k,\{u_j^{k+1},\xi_j^{k+1}\}_{j\neq t}\right).
\end{equation*}
Then
\begin{subequations}\label{eq:cross_probe_cancel}
\begin{align}
\E\!\left[v_t^{\mathrm{cross},k+1}(u_t^{k+1})^\top\mid \mathcal G_{k,t}\right]&=0, \label{eq:cross_probe_cancel_trace}\\
\E\!\left[\xi_t^{k+1}(w_t^{\mathrm{cross},k+1})^\top\mid \mathcal G_{k,t}\right]&=0. \label{eq:cross_probe_cancel_sharp}
\end{align}
\end{subequations}
Consequently, the cross-layer contamination carried by the global probe contributes zero conditional mean to both
\begin{equation*}
\widehat C_t^{k+1}:=Z_t^k\bigl(v_t^{k+1}(u_t^{k+1})^\top\bigr), \quad \widehat B_{\sharp,t}^{k+1}:=P_t^k\bigl(\xi_t^{k+1}(w_t^{k+1})^\top\bigr).
\end{equation*}
\end{lemma}

\begin{proof}
By construction, $M_t^{\mathrm{cross},k+1}$ depends only on the shared operator realization $\tilde{\mathcal H}^{k+1}$ and on the probes $\{u_j^{k+1},\xi_j^{k+1}\}_{j\neq t}$, so it does not depend on $u_t^{k+1}$. Hence, conditioning first on $(\mathcal G_{k,t},\xi_t^{k+1})$ and using the zero mean of the Rademacher vector $u_t^{k+1}$ gives
\begin{equation*}
\E\!\left[v_t^{\mathrm{cross},k+1}(u_t^{k+1})^\top\mid \mathcal G_{k,t},\xi_t^{k+1}\right] = v_t^{\mathrm{cross},k+1}\,\E\!\left[(u_t^{k+1})^\top\mid \mathcal G_{k,t},\xi_t^{k+1}\right] =0.
\end{equation*}
Taking the tower expectation yields \eqref{eq:cross_probe_cancel_trace}.

Similarly, $M_t^{\mathrm{cross},k+1}$ and therefore $w_t^{\mathrm{cross},k+1}$ depend on $u_t^{k+1}$ and on the off-layer probes, but not on $\xi_t^{k+1}$. Conditioning first on $(\mathcal G_{k,t},u_t^{k+1})$ and using the zero mean of the Rademacher vector $\xi_t^{k+1}$ gives
\begin{equation*}
\E\!\left[\xi_t^{k+1}(w_t^{\mathrm{cross},k+1})^\top\mid \mathcal G_{k,t},u_t^{k+1}\right] = \E\!\left[\xi_t^{k+1}\mid \mathcal G_{k,t},u_t^{k+1}\right](w_t^{\mathrm{cross},k+1})^\top =0.
\end{equation*}
Taking the tower expectation yields \eqref{eq:cross_probe_cancel_sharp}. The final claim follows by linearity from the decompositions
\begin{equation*}
v_t^{k+1}=v_t^{\mathrm{self},k+1}+v_t^{\mathrm{cross},k+1}, \quad w_t^{k+1}=w_t^{\mathrm{self},k+1}+w_t^{\mathrm{cross},k+1}.
\end{equation*}
\end{proof}

\begin{lemma}[Unbiasedness and second moments of the projected stochastic directions]\label{lem:unbiased_lifted_oracles}
Assume \emph{(A1)}--\emph{(A3)} and the sampling conditions stated in the main text. Define $\widehat{\boldsymbol V}^{k+1}:=\boldsymbol P^k\tilde{\boldsymbol V}^{k+1}\in\mathcal Y$, $\widehat W^{k+1}:=\tilde U_x^{k+1}-\tilde{\mathcal J}^{k+1}[\boldsymbol P^{k\top}\boldsymbol Z^k]\in\R^{d_x}$ and $\widehat{\boldsymbol S}^{k+1}:=\widehat{\mathcal H[\boldsymbol Z]}^{k+1}-\boldsymbol P^k\tilde{\boldsymbol U}_Y^{k+1}\in\mathcal Y$, where $\widehat{\mathcal H[\boldsymbol Z]}^{k+1}$ is the assembled multilayer HVP estimator described in Appendix~\ref{app:multilayer_algorithm_derivation}. Let $\omega_P:=\max_{1\le \ell\le L}\frac{m_\ell}{r_\ell}$ and $\Lambda_Y:=(\omega_P-1)L_{g,1}^2$. Then:
\begin{enumerate}
\item[(i)] \textbf{Conditional unbiasedness.}
\begin{equation*}
\begin{aligned}
\E[\widehat{\boldsymbol V}^{k+1}\mid\mathcal F_k]&=\nabla_{\boldsymbol Y}g(x^k,\boldsymbol Y^k),\\
\E[\widehat W^{k+1}\mid\mathcal F_k] &=\nabla_x f(x^k,\boldsymbol Y^k)-\mathcal J(x^k,\boldsymbol Y^k)[\boldsymbol Z^k],\\
\E[\widehat{\boldsymbol S}^{k+1}\mid\mathcal F_k] &=\mathcal H(x^k,\boldsymbol Y^k)[\boldsymbol Z^k]-\nabla_{\boldsymbol Y}f(x^k,\boldsymbol Y^k).
\end{aligned}
\end{equation*}
\item[(ii)] \textbf{Conditional second moments.}
Define
\begin{equation*}
\Delta_S^{k+1} := \widehat{\boldsymbol S}^{k+1} -\bigl(\mathcal H(x^k,\boldsymbol Y^k)[\boldsymbol Z^k] -\nabla_{\boldsymbol Y}f(x^k,\boldsymbol Y^k)\bigr).
\end{equation*}
There exist constants $\sigma_{S,0},\sigma_{S,2},\Lambda_S<\infty$ such that
\begin{subequations}\label{eq:lifted_oracle_second_moments}
\begin{align}
\E[\|\widehat{\boldsymbol V}^{k+1}-\nabla_{\boldsymbol Y}g(x^k,\boldsymbol Y^k)\|^2\mid\mathcal F_k]
&\le \omega_P\tilde\sigma_{g,1}^2+\Lambda_Y\|\boldsymbol Y^k-\boldsymbol Y^*(x^k)\|^2, \label{eq:V_hat_var_matrix}\\
\E[\|\widehat W^{k+1}-\E[\widehat W^{k+1}\mid\mathcal F_k]\|^2\mid\mathcal F_k]
&\le \tilde\sigma_{f,1}^2+(\tilde\sigma_{g,2}^2+\Lambda_Y)\|\boldsymbol Z^k\|^2, \label{eq:W_hat_var_matrix}\\
\E\!\bigl[\|\Delta_S^{k+1}\|^2 \mid \mathcal F_k \bigr]
&\le \sigma_{S,0}^2 +\sigma_{S,2}^2\|\boldsymbol Y^k-\boldsymbol Y^*(x^k)\|^2+\Lambda_S\|\boldsymbol Z^k\|^2. \label{eq:S_hat_var_matrix}
\end{align}
\end{subequations}
\end{enumerate}
\end{lemma}

\begin{proof}
We condition on $\mathcal F_k$ throughout and abbreviate
\begin{equation*}
\boldsymbol P:=\boldsymbol P^k,\quad P_\ell:=P_\ell^k,\quad Q_\ell:=P_\ell(P_\ell)^\top,
\end{equation*}
\begin{equation*}
\boldsymbol Y:=\boldsymbol Y^k,\quad \boldsymbol Z:=\boldsymbol Z^k,\quad f_Y:=\nabla_{\boldsymbol Y}f(x^k,\boldsymbol Y),\quad H:=\mathcal H(x^k,\boldsymbol Y).
\end{equation*}
We also write
\begin{equation*}
\mathcal F_k^+=\sigma(\mathcal F_k\cup\{\boldsymbol P\}), \quad \mathcal F_k^{++} = \sigma\!\left(\mathcal F_k^+,\{u_\ell^{k+1},\xi_\ell^{k+1}\}_{\ell=1}^L\right).
\end{equation*}

\textbf{Part (i): unbiasedness of $\widehat{\boldsymbol V}^{k+1}$. }
By Assumption~\ref{as:A2-1}, $\E[\tilde{\boldsymbol V}^{k+1}\mid \mathcal F_k^+] = \nabla_{\boldsymbol B}\tilde g_k(x^k,\mathbf 0)$. Lemma~\ref{lem:projection_moments_matrix}(iv) gives $\nabla_{\boldsymbol B}\tilde g_k(x^k,\mathbf 0) = \boldsymbol P^\top \nabla_{\boldsymbol Y}g(x^k,\boldsymbol Y)$. Therefore,
\begin{equation*}
\begin{aligned}
\E[\widehat{\boldsymbol V}^{k+1}\mid \mathcal F_k]
&=
\E\!\left[\boldsymbol P\,\E[\tilde{\boldsymbol V}^{k+1}\mid \mathcal F_k^+]\mid \mathcal F_k\right] \\
&=
\E[\boldsymbol P\boldsymbol P^\top\mid \mathcal F_k]\nabla_{\boldsymbol Y}g(x^k,\boldsymbol Y)
=
\nabla_{\boldsymbol Y}g(x^k,\boldsymbol Y),
\end{aligned}
\end{equation*}
because each block projector satisfies $\E[Q_\ell\mid\mathcal F_k]=I_{m_\ell}$.

\textbf{Part (i): unbiasedness of $\widehat W^{k+1}$. }
Again by Assumption~\ref{as:A2-1},
\begin{equation*}
\E[\tilde U_x^{k+1}\mid \mathcal F_k^+] = \nabla_x \tilde f_k(x^k,\mathbf 0) = \nabla_x f(x^k,\boldsymbol Y).
\end{equation*}
Moreover, for every $W\in\mathcal B$,
\begin{equation*}
\E[\tilde{\mathcal J}^{k+1}[W]\mid \mathcal F_k^+] = \nabla_{x\boldsymbol B}^2 \tilde g_k(x^k,\mathbf 0)[W].
\end{equation*}
Applying Lemma~\ref{lem:projection_moments_matrix}(iv) with $W=\boldsymbol P^\top\boldsymbol Z$ gives
\begin{equation*}
\E[\tilde{\mathcal J}^{k+1}[\boldsymbol P^\top\boldsymbol Z]\mid \mathcal F_k^+] = \mathcal J(x^k,\boldsymbol Y)[\boldsymbol P\boldsymbol P^\top\boldsymbol Z],
\end{equation*}
where $(\boldsymbol P\boldsymbol P^\top\boldsymbol Z)_\ell=Q_\ell Z_\ell$. Taking conditional expectation with respect to $\mathcal F_k$ and using $\E[Q_\ell\mid\mathcal F_k]=I_{m_\ell}$ yields
\begin{equation*}
\E[\widehat W^{k+1}\mid \mathcal F_k] = \nabla_x f(x^k,\boldsymbol Y)-\mathcal J(x^k,\boldsymbol Y)[\boldsymbol Z].
\end{equation*}

\textbf{Part (i): unbiasedness of $\widehat{\boldsymbol S}^{k+1}$. }
For each pair of layers $(\ell,t)$, let $H_{\ell t}:\R^{m_t\times n_t}\to\R^{m_\ell\times n_\ell}$ denote the $(\ell,t)$ block of $H$, so that
\begin{equation*}
(H[\boldsymbol Z])_\ell=\sum_{t=1}^L H_{\ell t}[Z_t].
\end{equation*}
For each $(\ell,t)$ and columns $i\in\{1,\ldots,n_\ell\}$, $j\in\{1,\ldots,n_t\}$, there exist matrices $H_{ij}^{\ell t}\in\R^{m_\ell\times m_t}$ such that for $U_t=[u_{t,1},\ldots,u_{t,n_t}]$,
\begin{equation*}
(H_{\ell t}[U_t])_{:,i} = \sum_{j=1}^{n_t} H_{ij}^{\ell t}u_{t,j}.
\end{equation*}
Self-adjointness of $H$ on $\mathcal Y$ implies $H_{ji}^{t\ell}=(H_{ij}^{\ell t})^\top$.

\emph{Step 1: deterministic expectation of the masked main HVPs.} Condition on $\mathcal F_k^+$ and define the deterministic projected Hessian operator $\bar{\mathcal H}:=\nabla_{\boldsymbol B\boldsymbol B}^2\tilde g_k(x^k,\mathbf 0)$. For each input layer $t$, set $\tilde Z_t:=(P_t)^\top Z_t$ and $\bar{\boldsymbol A}^{(t)}:=\boldsymbol P\,\bar{\mathcal H}[E_t(\tilde Z_t)]$. By Lemma~\ref{lem:projection_moments_matrix}(iv),
\begin{equation*}
(\bar A_\ell^{(t)})_{:,i} = \sum_{j=1}^{n_t}Q_\ell H_{ij}^{\ell t}Q_t z_{t,j}.
\end{equation*}
If $\ell\neq t$, Lemma~\ref{lem:projection_moments_matrix}(iii) gives
\begin{equation*}
\E[\bar A_\ell^{(t)}\mid \mathcal F_k] = H_{\ell t}[Z_t].
\end{equation*}
For $\ell=t$, define the self-layer partial-transpose operator and trace matrix by
\begin{equation*}
(H_{tt}^\sharp[Z_t])_{:,i} := \sum_{j=1}^{n_t}(H_{ij}^{tt})^\top z_{t,j}, \quad (T_t)_{ij}:=\tr(H_{ij}^{tt}).
\end{equation*}
Applying Lemma~\ref{lem:projection_moments_matrix}(ii) blockwise to $H_{ij}^{tt}$ yields
\begin{equation}\label{eq:self_main_expectation_active}
\E[\bar A_t^{(t)}\mid \mathcal F_k]
=
(a_t-b_t)H_{tt}[Z_t]
+b_t H_{tt}^\sharp[Z_t]
+b_t Z_t T_t.
\end{equation}

\emph{Step 2: deterministic expectation of the global probe correction terms.} Condition now on $\mathcal F_k^+$ and define
\begin{equation*}
\bar{\boldsymbol M}:=\bar{\mathcal H}[\Delta\boldsymbol B^{k+1}], \quad \Delta B_\ell^{k+1}:=\xi_\ell^{k+1}(u_\ell^{k+1})^\top.
\end{equation*}
Fix a layer $t$. The same cancellation argument as in Lemma~\ref{lem:probe_cancellation} applies to the deterministic operator $\bar{\mathcal H}$, so the cross-layer part of $\bar M_t$ contributes zero conditional mean to both
\begin{equation*}
\widehat C_t^{k+1} = Z_t\bigl(v_t^{k+1}(u_t^{k+1})^\top\bigr), \quad \widehat B_{\sharp,t}^{k+1} = P_t\bigl(\xi_t^{k+1}(w_t^{k+1})^\top\bigr).
\end{equation*}
It is therefore enough to compute the expectation of the self part.

For the trace-matrix correction, define the projected trace matrix $(\tilde T_t)_{ij}:=\tr\!\left((P_t)^\top H_{ij}^{tt}P_t\right)$. Writing out the self contribution columnwise and using $\E[\xi_t^{k+1}(\xi_t^{k+1})^\top]=I_{r_t}$, $\E[u_t^{k+1}(u_t^{k+1})^\top]=I_{n_t}$ gives
\begin{equation}\label{eq:ECt_given_P_active}
\E[\widehat C_t^{k+1}\mid \mathcal F_k^+]=Z_t\tilde T_t.
\end{equation}
Taking expectation over $P_t$ and using Lemma~\ref{lem:projection_moments_matrix}(i) gives:
\begin{equation}\label{eq:ECt_total_active}
\E[\widehat C_t^{k+1}\mid \mathcal F_k]=Z_tT_t.
\end{equation}

For the partial-transpose correction, define the projected self-layer operator
\begin{equation*}
(\bar H_{tt}^\sharp[W_t])_{:,i} := \sum_{j=1}^{n_t}\bigl((P_t)^\top H_{ij}^{tt}P_t\bigr)^\top (w_t)_j.
\end{equation*}
Write $\tilde Z_t=[\tilde z_{t,1},\ldots,\tilde z_{t,n_t}]$. For every output column $i$,
\begin{equation*}
(M_t^{\mathrm{self},k+1})_{:,i} = \sum_{j=1}^{n_t}u_{t,j}^{k+1}\bigl((P_t)^\top H_{ij}^{tt}P_t\bigr)\xi_t^{k+1}.
\end{equation*}
Hence
\begin{equation*}
w_{t,i}^{\mathrm{self},k+1}=
\bigl((M_t^{\mathrm{self},k+1})_{:,i}\bigr)^\top(\tilde Z_t u_t^{k+1}) =\sum_{j=1}^{n_t}u_{t,j}^{k+1}
(\xi_t^{k+1})^\top
\bigl((P_t)^\top H_{ij}^{tt}P_t\bigr)^\top
(\tilde Z_t u_t^{k+1}),
\end{equation*}
and therefore
\begin{equation*}
\xi_t^{k+1}w_{t,i}^{\mathrm{self},k+1} = \sum_{j=1}^{n_t}u_{t,j}^{k+1} \bigl(\xi_t^{k+1}(\xi_t^{k+1})^\top\bigr) \bigl((P_t)^\top H_{ij}^{tt}P_t\bigr)^\top (\tilde Z_t u_t^{k+1}).
\end{equation*}
Taking $\E_{\xi_t}[\cdot\mid \mathcal F_k^+,u_t^{k+1}]$ and using $\E[\xi_t^{k+1}(\xi_t^{k+1})^\top]=I_{r_t}$ gives
\begin{equation*}
\E_{\xi_t}\!\left[\xi_t^{k+1}w_{t,i}^{\mathrm{self},k+1}\mid \mathcal F_k^+,u_t^{k+1}\right] = \sum_{j=1}^{n_t}u_{t,j}^{k+1} \bigl((P_t)^\top H_{ij}^{tt}P_t\bigr)^\top (\tilde Z_t u_t^{k+1}).
\end{equation*}
Now take $\E_{u_t}[\cdot\mid \mathcal F_k^+]$ and use
\begin{equation*}
\E\!\left[u_{t,j}^{k+1}(\tilde Z_t u_t^{k+1})\mid \mathcal F_k^+\right] = \tilde z_{t,j}
\end{equation*}
to obtain
\begin{equation*}
\E\!\left[\xi_t^{k+1}w_{t,i}^{\mathrm{self},k+1}\mid \mathcal F_k^+\right] = \sum_{j=1}^{n_t} \bigl((P_t)^\top H_{ij}^{tt}P_t\bigr)^\top \tilde z_{t,j} = (\bar H_{tt}^\sharp[\tilde Z_t])_{:,i}.
\end{equation*}
Therefore the entire rank-one matrix satisfies
\begin{equation}\label{eq:EBt_given_P_active}
\E[\widehat B_{\sharp,t}^{k+1}\mid \mathcal F_k^+]
=
P_t\,\bar H_{tt}^\sharp[\tilde Z_t].
\end{equation}
Applying Lemma~\ref{lem:projection_moments_matrix}(ii) to the matrices $(H_{ij}^{tt})^\top$ gives
\begin{equation}\label{eq:EBt_total_active}
\E[\widehat B_{\sharp,t}^{k+1}\mid \mathcal F_k]
=
b_t H_{tt}[Z_t]
+(a_t-b_t)H_{tt}^\sharp[Z_t]
+b_t Z_tT_t.
\end{equation}

\emph{Step 3: self-layer linear inversion.} Let
\begin{equation*}
X_t:=H_{tt}[Z_t],\quad X_t^\sharp:=H_{tt}^\sharp[Z_t],\quad C_t:=Z_tT_t.
\end{equation*}
By \eqref{eq:self_main_expectation_active}, \eqref{eq:ECt_total_active}, and \eqref{eq:EBt_total_active},
\begin{equation*}
\E[\bar A_t^{(t)}\mid \mathcal F_k] = (a_t-b_t)X_t+b_t X_t^\sharp+b_t C_t,
\end{equation*}
\begin{equation*}
\E[\widehat B_{\sharp,t}^{k+1}\mid \mathcal F_k] = b_t X_t+(a_t-b_t)X_t^\sharp+b_t C_t, \quad \E[\widehat C_t^{k+1}\mid \mathcal F_k]=C_t.
\end{equation*}
Therefore the diagonal-block correction used in $\widehat{\mathcal H_{tt}[Z_t]}^{k+1}$ satisfies
\begin{equation}\label{eq:self_corrected_unbiased_active}
\E[\widehat{\mathcal H_{tt}[Z_t]}^{k+1}\mid \mathcal F_k]=H_{tt}[Z_t].
\end{equation}

\emph{Step 4: pass from $\bar{\mathcal H}$ to the stochastic operator $\tilde{\mathcal H}^{k+1}$.} The actual algorithm uses the shared random operator $\tilde{\mathcal H}^{k+1}$. All queried directions are $\mathcal F_k^{++}$-measurable:
\begin{equation*}
E_t(\tilde Z_t)\in\mathcal B\quad\text{is }\mathcal F_k^+\text{-measurable}, \quad \Delta\boldsymbol B^{k+1}\in\mathcal B\quad\text{is }\mathcal F_k^{++}\text{-measurable}.
\end{equation*}
By Assumption~\ref{as:A2-1},
\begin{equation*}
\E[\tilde{\mathcal H}^{k+1}[W]\mid \mathcal F_k^{++}] = \bar{\mathcal H}[W] \quad \forall\ \mathcal F_k^{++}\text{-measurable }W\in\mathcal B.
\end{equation*}
Because $E_t(\tilde Z_t)$ is $\mathcal F_k^+$-measurable and therefore $\mathcal F_k^{++}$-measurable, we have
\begin{equation*}
\E[\tilde{\mathcal H}^{k+1}[E_t(\tilde Z_t)]\mid \mathcal F_k^{++}] = \bar{\mathcal H}[E_t(\tilde Z_t)].
\end{equation*}
Multiplying by the deterministic lift $\boldsymbol P$ and then taking $\E[\cdot\mid \mathcal F_k]$ yields
\begin{equation*}
\E[\boldsymbol A^{(t),k+1}\mid \mathcal F_k] = \E[\bar{\boldsymbol A}^{(t)}\mid \mathcal F_k].
\end{equation*}
Likewise, once the probes are fixed, the maps
\begin{equation*}
M_t\longmapsto Z_t\bigl((M_t^\top\xi_t^{k+1})(u_t^{k+1})^\top\bigr), \quad M_t\longmapsto P_t\bigl(\xi_t^{k+1}((M_t^\top(\tilde Z_tu_t^{k+1}))^\top)\bigr)
\end{equation*}
are linear in the queried output block $M_t$. Since $\Delta\boldsymbol B^{k+1}$ is $\mathcal F_k^{++}$-measurable,
\begin{equation*}
\E[\tilde{\mathcal H}^{k+1}[\Delta\boldsymbol B^{k+1}]\mid \mathcal F_k^{++}] = \bar{\mathcal H}[\Delta\boldsymbol B^{k+1}],
\end{equation*}
and therefore
\begin{equation*}
\E[\widehat C_t^{k+1}\mid \mathcal F_k] = \E[\bar{\widehat C}_t^{k+1}\mid \mathcal F_k], \quad \E[\widehat B_{\sharp,t}^{k+1}\mid \mathcal F_k] = \E[\bar{\widehat B}_{\sharp,t}^{k+1}\mid \mathcal F_k],
\end{equation*}
where the barred quantities are computed with the deterministic operator $\bar{\mathcal H}$. Consequently the deterministic first-moment identities derived in Steps 1--3 pass unchanged to the stochastic estimator built from $\tilde{\mathcal H}^{k+1}$. In particular,
\begin{equation*}
\E[(A_\ell^{(t),k+1})_{\ell\neq t}\mid \mathcal F_k]=H_{\ell t}[Z_t], \quad \E[\widehat{\mathcal H_{tt}[Z_t]}^{k+1}\mid \mathcal F_k]=H_{tt}[Z_t].
\end{equation*}
Summing the corrected diagonal and lifted off-diagonal block contributions yields
\begin{equation}\label{eq:EHhat_multilayer_active}
\E[\widehat{\mathcal H[\boldsymbol Z]}^{k+1}\mid \mathcal F_k]
=
H[\boldsymbol Z].
\end{equation}
Finally, using Assumption~\ref{as:A2-1} and Lemma~\ref{lem:projection_moments_matrix}(iv),
\begin{equation*}
\E[\boldsymbol P\tilde{\boldsymbol U}_Y^{k+1}\mid \mathcal F_k] = \nabla_{\boldsymbol Y}f(x^k,\boldsymbol Y),
\end{equation*}
so
\begin{equation*}
\E[\widehat{\boldsymbol S}^{k+1}\mid \mathcal F_k] = H[\boldsymbol Z]-f_Y.
\end{equation*}

\textbf{Part (ii): second moments. }
We first prove \eqref{eq:V_hat_var_matrix} and \eqref{eq:W_hat_var_matrix}, and then derive \eqref{eq:S_hat_var_matrix}.

\emph{Step 1: bound for $\widehat{\boldsymbol V}^{k+1}$.} Let $\boldsymbol G:=\nabla_{\boldsymbol Y}g(x^k,\boldsymbol Y)$. Define the block-diagonal lifted projector
\begin{equation*}
(\boldsymbol Q\boldsymbol U)_\ell:=Q_\ell U_\ell,\quad \boldsymbol U\in\mathcal Y.
\end{equation*}
Then
\begin{equation*}
\widehat{\boldsymbol V}^{k+1}-\boldsymbol G = \boldsymbol P\bigl(\tilde{\boldsymbol V}^{k+1}-\boldsymbol P^\top\boldsymbol G\bigr) +(\boldsymbol Q-I)\boldsymbol G.
\end{equation*}
Conditioned on $\mathcal F_k^+$, the first term has mean zero. Using
\begin{equation*}
\|\boldsymbol P\boldsymbol U\|^2 = \sum_{\ell=1}^L \|P_\ell U_\ell\|^2 \le \omega_P\sum_{\ell=1}^L \|U_\ell\|^2 = \omega_P\|\boldsymbol U\|^2,
\end{equation*}
Assumption~\ref{as:A2-1} gives
\begin{equation*}
\E\!\left[\left\|\boldsymbol P\bigl(\tilde{\boldsymbol V}^{k+1}-\boldsymbol P^\top\boldsymbol G\bigr)\right\|^2\mid \mathcal F_k^+\right] \le \omega_P \tilde\sigma_{g,1}^2.
\end{equation*}
Also,
\begin{equation*}
\E[\|(\boldsymbol Q-I)\boldsymbol G\|^2\mid \mathcal F_k] = \sum_{\ell=1}^L \E[\|(Q_\ell-I)G_\ell\|^2\mid \mathcal F_k] \le (\omega_P-1)\|\boldsymbol G\|^2.
\end{equation*}
Since $\nabla_{\boldsymbol Y}g(x^k,\boldsymbol Y^*(x^k))=0$ and $\nabla g$ is $L_{g,1}$-Lipschitz,
\begin{equation*}
\|\boldsymbol G\| \le L_{g,1}\|\boldsymbol Y-\boldsymbol Y^*(x^k)\|.
\end{equation*}
This proves \eqref{eq:V_hat_var_matrix}.

\emph{Step 2: bound for $\widehat W^{k+1}$.} Write
\begin{equation*}
R_x:=\tilde U_x^{k+1}, \quad R_J:=\tilde{\mathcal J}^{k+1}[\boldsymbol P^\top\boldsymbol Z].
\end{equation*}
By the law of total variance,
\begin{equation*}
\begin{aligned}
\E[\|R_J-\E[R_J\mid \mathcal F_k]\|^2\mid \mathcal F_k]
={}&
\E[\|R_J-\E[R_J\mid \mathcal F_k^+]\|^2\mid \mathcal F_k] \\
&+
\E[\|\E[R_J\mid \mathcal F_k^+]-\E[R_J\mid \mathcal F_k]\|^2\mid \mathcal F_k].
\end{aligned}
\end{equation*}
For the first term, Assumption~\ref{as:A2-1} gives
\begin{equation*}
\E[\|R_J-\E[R_J\mid \mathcal F_k^+]\|^2\mid \mathcal F_k^+] \le \tilde\sigma_{g,2}^2\|\boldsymbol P^\top\boldsymbol Z\|^2.
\end{equation*}
Taking expectation over $\mathcal F_k^+$ and using $\E[\|\boldsymbol P^\top\boldsymbol Z\|^2\mid\mathcal F_k]=\|\boldsymbol Z\|^2$ yields
\begin{equation*}
\E[\|R_J-\E[R_J\mid \mathcal F_k^+]\|^2\mid \mathcal F_k] \le \tilde\sigma_{g,2}^2\|\boldsymbol Z\|^2.
\end{equation*}
For the second term,
\begin{equation*}
\E[R_J\mid \mathcal F_k^+]=\mathcal J(x^k,\boldsymbol Y)[\boldsymbol Q\boldsymbol Z], \quad \E[R_J\mid \mathcal F_k]=\mathcal J(x^k,\boldsymbol Y)[\boldsymbol Z].
\end{equation*}
Hence
\begin{equation*}
\begin{aligned}
\E[\|\E[R_J\mid \mathcal F_k^+]-\E[R_J\mid \mathcal F_k]\|^2\mid \mathcal F_k]
&=
\E[\|\mathcal J(x^k,\boldsymbol Y)[(\boldsymbol Q-I)\boldsymbol Z]\|^2\mid \mathcal F_k] \\
&\le
L_{g,1}^2\E[\|(\boldsymbol Q-I)\boldsymbol Z\|^2\mid \mathcal F_k] \le \Lambda_Y\|\boldsymbol Z\|^2.
\end{aligned}
\end{equation*}
The oracle $\tilde U_x^{k+1}$ is independent of $\tilde{\mathcal J}^{k+1}$ given $\mathcal F_k^+$, so
\begin{equation*}
\E[\|R_x-\E[R_x\mid \mathcal F_k]\|^2\mid \mathcal F_k]\le \tilde\sigma_{f,1}^2.
\end{equation*}
Combining the three displays yields \eqref{eq:W_hat_var_matrix}.

\emph{Step 3: bound for $\widehat{\boldsymbol S}^{k+1}$.} Define
\begin{equation*}
\Delta_H^{k+1} := \widehat{\mathcal H[\boldsymbol Z]}^{k+1}-H[\boldsymbol Z], \quad \Delta_U^{k+1} := \boldsymbol P\tilde{\boldsymbol U}_Y^{k+1}-f_Y.
\end{equation*}
Then
\begin{equation*}
\Delta_S^{k+1}=\Delta_H^{k+1}-\Delta_U^{k+1},
\end{equation*}
so
\begin{equation}\label{eq:DeltaS_split_active}
\|\Delta_S^{k+1}\|^2
\le
2\|\Delta_H^{k+1}\|^2+2\|\Delta_U^{k+1}\|^2.
\end{equation}

\emph{Step 3a: bound for $\Delta_U^{k+1}$.} Exactly as in Step 1,
\begin{equation*}
\Delta_U^{k+1} = \boldsymbol P(\tilde{\boldsymbol U}_Y^{k+1}-\boldsymbol P^\top f_Y) +(\boldsymbol Q-I)f_Y.
\end{equation*}
Conditioned on $\mathcal F_k^+$, the first term has mean zero, and therefore
\begin{equation*}
\E[\|\boldsymbol P(\tilde{\boldsymbol U}_Y^{k+1}-\boldsymbol P^\top f_Y)\|^2\mid \mathcal F_k] \le \omega_P\tilde\sigma_{f,2}^2.
\end{equation*}
Also,
\begin{equation*}
\E[\|(\boldsymbol Q-I)f_Y\|^2\mid \mathcal F_k] \le (\omega_P-1)\|f_Y\|^2.
\end{equation*}
Since
\begin{equation*}
\|f_Y\| \le \|\nabla_{\boldsymbol Y}f(x^k,\boldsymbol Y^*(x^k))\| + L_{f,1}\|\boldsymbol Y-\boldsymbol Y^*(x^k)\| \le L_{f,0}+L_{f,1}\|\boldsymbol Y-\boldsymbol Y^*(x^k)\|,
\end{equation*}
we obtain
\begin{equation*}
\|f_Y\|^2 \le 2L_{f,0}^2+2L_{f,1}^2\|\boldsymbol Y-\boldsymbol Y^*(x^k)\|^2.
\end{equation*}
Consequently,
\begin{equation}\label{eq:DeltaU_bound_active}
\E[\|\Delta_U^{k+1}\|^2\mid \mathcal F_k]
\le
\omega_P\tilde\sigma_{f,2}^2
+2(\omega_P-1)L_{f,0}^2
+2(\omega_P-1)L_{f,1}^2\|\boldsymbol Y-\boldsymbol Y^*(x^k)\|^2.
\end{equation}

\emph{Step 3b: bound for $\Delta_H^{k+1}$.} For each layer $t$, define the centered contribution $\delta_t^{k+1}\in\mathcal Y$ by
\begin{equation*}
(\delta_t^{k+1})_\ell
:=
\begin{cases}
A_\ell^{(t),k+1}-H_{\ell t}[Z_t], & \ell\neq t,\\[1mm]
\widehat{\mathcal H_{tt}[Z_t]}^{k+1}-H_{tt}[Z_t], & \ell=t.
\end{cases}
\end{equation*}
Then $\Delta_H^{k+1}=\sum_{t=1}^L \delta_t^{k+1}$. Hence
\begin{equation}\label{eq:DeltaH_sum_active}
\|\Delta_H^{k+1}\|^2
\le
L\sum_{t=1}^L \|\delta_t^{k+1}\|^2.
\end{equation}

For each layer $t$, define the self-layer correction coefficients
\begin{equation*}
c_{A,t}:=\frac{a_t-b_t}{a_t(a_t-2b_t)}, \quad c_{B,t}:=\frac{b_t}{a_t(a_t-2b_t)}, \quad c_{C,t}:=\frac{b_t}{a_t},
\end{equation*}
and let
\begin{equation*}
\bar c_A:=\max_{1\le t\le L} c_{A,t}, \quad \bar c_B:=\max_{1\le t\le L} c_{B,t}, \quad \bar c_C:=\max_{1\le t\le L} c_{C,t}.
\end{equation*}
Because $\E[\widehat{\mathcal H_{tt}[Z_t]}^{k+1}\mid \mathcal F_k]=H_{tt}[Z_t]$, we may write
\begin{equation*}
\begin{aligned}
\widehat{\mathcal H_{tt}[Z_t]}^{k+1}\hspace{-1.0em}-H_{tt}[Z_t]
=&\ c_{A,t}\bigl(A_t^{(t),k+1}-\E[A_t^{(t),k+1}\mid \mathcal F_k]\bigr)-c_{B,t}\bigl(\widehat B_{\sharp,t}^{k+1}-\E[\widehat B_{\sharp,t}^{k+1}\mid \mathcal F_k]\bigr)\\
&-c_{C,t}\bigl(\widehat C_t^{k+1}-\E[\widehat C_t^{k+1}\mid \mathcal F_k]\bigr).
\end{aligned}
\end{equation*}
Using $\E[\|X-\E[X]\|^2]\le \E[\|X\|^2]$ and $\|a+b+c\|^2\le 3(\|a\|^2+\|b\|^2+\|c\|^2)$, we obtain
\begin{equation*}
\E[\|\delta_t^{k+1}\|^2\mid \mathcal F_k] \le (1+3\bar c_A^2)\E[\|\boldsymbol A^{(t),k+1}\|^2\mid \mathcal F_k] +3\bar c_B^2\E[\|\widehat B_{\sharp,t}^{k+1}\|^2\mid \mathcal F_k] +3\bar c_C^2\E[\|\widehat C_t^{k+1}\|^2\mid \mathcal F_k].
\end{equation*}

We now bound the three second moments. First, conditioned on $\mathcal F_k^+$,
\begin{equation*}
\|\bar{\mathcal H}[W]\| \le \omega_P L_{g,1}\|W\|, \quad \forall W\in\mathcal B,
\end{equation*}
because $\|\boldsymbol P\|\le \sqrt{\omega_P}$ and $\|\boldsymbol P^\top\|\le \sqrt{\omega_P}$. Using Assumption~\ref{as:A2-1} and $\|\boldsymbol P U\|^2\le \omega_P\|U\|^2$,
\begin{equation*}
\begin{aligned}
\E[\|\boldsymbol A^{(t),k+1}\|^2\mid \mathcal F_k^+]
&=
\E[\|\boldsymbol P\tilde{\mathcal H}^{k+1}[E_t(\tilde Z_t)]\|^2\mid \mathcal F_k^+] \\
&\le
\omega_P\E[\|\tilde{\mathcal H}^{k+1}[E_t(\tilde Z_t)]\|^2\mid \mathcal F_k^+] \\
&\le
2\omega_P\bigl(\tilde\sigma_{g,2}^2+\omega_P^2L_{g,1}^2\bigr)\|\tilde Z_t\|^2.
\end{aligned}
\end{equation*}
Taking expectation over $\mathcal F_k^+$ and using $\E[\|\tilde Z_t\|^2\mid \mathcal F_k]=\|Z_t\|^2$ yields
\begin{equation*}
\E[\|\boldsymbol A^{(t),k+1}\|^2\mid \mathcal F_k] \le \Lambda_A\|Z_t\|^2, \quad \Lambda_A:=2\omega_P\bigl(\tilde\sigma_{g,2}^2+\omega_P^2L_{g,1}^2\bigr).
\end{equation*}

Next define $D_{rn}:=\sum_{\ell=1}^L r_\ell n_\ell$ and $\eta_{rn}:=\max_{1\le \ell\le L} r_\ell n_\ell$. Since each probe block is rank one with Rademacher entries,
\begin{equation*}
\|\Delta\boldsymbol B^{k+1}\|^2=\sum_{\ell=1}^L \|\xi_\ell^{k+1}(u_\ell^{k+1})^\top\|_F^2=D_{rn}.
\end{equation*}
Therefore, conditioned on $\mathcal F_k^+$,
\begin{equation*}
\E[\|\boldsymbol M^{k+1}\|^2\mid \mathcal F_k^+] \le 2\bigl(\tilde\sigma_{g,2}^2+\omega_P^2L_{g,1}^2\bigr)D_{rn},
\end{equation*}
and hence, for each layer $t$,
\begin{equation*}
\E[\|M_t^{k+1}\|^2\mid \mathcal F_k^+] \le 2\bigl(\tilde\sigma_{g,2}^2+\omega_P^2L_{g,1}^2\bigr)D_{rn}.
\end{equation*}

For $\widehat C_t^{k+1}$, since
\begin{equation*}
\widehat C_t^{k+1}=Z_t\bigl(v_t^{k+1}(u_t^{k+1})^\top\bigr)
\end{equation*}
and $\|u_t^{k+1}\|^2=n_t$, we have
\begin{equation*}
\|\widehat C_t^{k+1}\|^2 = \|Z_t v_t^{k+1}\|^2\|u_t^{k+1}\|^2 \le n_t\|Z_t\|^2\|v_t^{k+1}\|^2.
\end{equation*}
Moreover,
\begin{equation*}
v_t^{k+1}=(M_t^{k+1})^\top\xi_t^{k+1}, \quad \|v_t^{k+1}\|^2\le r_t\|M_t^{k+1}\|^2.
\end{equation*}
Hence
\begin{equation*}
\E[\|\widehat C_t^{k+1}\|^2\mid \mathcal F_k] \le \Lambda_C\|Z_t\|^2, \quad \Lambda_C:=2\eta_{rn}D_{rn}\bigl(\tilde\sigma_{g,2}^2+\omega_P^2L_{g,1}^2\bigr).
\end{equation*}

For $\widehat B_{\sharp,t}^{k+1}$,
\begin{equation*}
\widehat B_{\sharp,t}^{k+1}=P_t\bigl(\xi_t^{k+1}(w_t^{k+1})^\top\bigr),
\end{equation*}
so
\begin{equation*}
\|\widehat B_{\sharp,t}^{k+1}\|^2 \le \omega_P\|\xi_t^{k+1}(w_t^{k+1})^\top\|^2 = \omega_P r_t\|w_t^{k+1}\|^2.
\end{equation*}
Since
\begin{equation*}
w_t^{k+1}=(M_t^{k+1})^\top(\tilde Z_t u_t^{k+1}),
\end{equation*}
we get
\begin{equation*}
\|w_t^{k+1}\|^2 \le \|M_t^{k+1}\|^2\|\tilde Z_t u_t^{k+1}\|^2 \le n_t\|M_t^{k+1}\|^2\|\tilde Z_t\|^2.
\end{equation*}
Conditioning on $\mathcal F_k^+$ and using the bound for $\E[\|M_t^{k+1}\|^2\mid \mathcal F_k^+]$ gives
\begin{equation*}
\E[\|\widehat B_{\sharp,t}^{k+1}\|^2\mid \mathcal F_k^+] \le 2\omega_P \eta_{rn} D_{rn}\bigl(\tilde\sigma_{g,2}^2+\omega_P^2L_{g,1}^2\bigr)\|\tilde Z_t\|^2.
\end{equation*}
Taking expectation over $\mathcal F_k^+$ and using $\E[\|\tilde Z_t\|^2\mid \mathcal F_k]=\|Z_t\|^2$ yields
\begin{equation*}
\E[\|\widehat B_{\sharp,t}^{k+1}\|^2\mid \mathcal F_k] \le \Lambda_B\|Z_t\|^2,
\end{equation*}
where $\Lambda_B:=2\omega_P \eta_{rn} D_{rn}\bigl(\tilde\sigma_{g,2}^2+\omega_P^2L_{g,1}^2\bigr)$.

Combining the last three displays with \eqref{eq:DeltaH_sum_active} shows
\begin{equation*}
\E[\|\Delta_H^{k+1}\|^2\mid \mathcal F_k] \le \Lambda_H\|\boldsymbol Z\|^2,
\end{equation*}
where one may take
\begin{equation}\label{eq:LambdaH_active}
\Lambda_H
:=
L\Bigl((1+3\bar c_A^2)\Lambda_A+3\bar c_B^2\Lambda_B+3\bar c_C^2\Lambda_C\Bigr).
\end{equation}

Finally, combining \eqref{eq:DeltaS_split_active}, \eqref{eq:DeltaU_bound_active}, and the bound on $\Delta_H^{k+1}$ gives
\begin{equation*}
\E[\|\Delta_S^{k+1}\|^2\mid \mathcal F_k] \le 2\Lambda_H\|\boldsymbol Z\|^2 +2\omega_P\tilde\sigma_{f,2}^2 +4(\omega_P-1)L_{f,0}^2 +4(\omega_P-1)L_{f,1}^2\|\boldsymbol Y-\boldsymbol Y^*(x^k)\|^2.
\end{equation*}
Therefore \eqref{eq:S_hat_var_matrix} holds with
\begin{equation*}
\sigma_{S,0}^2:=2\omega_P\tilde\sigma_{f,2}^2+4(\omega_P-1)L_{f,0}^2, \quad \sigma_{S,2}^2:=4(\omega_P-1)L_{f,1}^2, \quad \Lambda_S:=2\Lambda_H.
\end{equation*}
\end{proof}

\begin{lemma}[Tracking bounds for $\boldsymbol Y^k$ and $\boldsymbol Z^k$]\label{lem:yz_tracking_rso}
Assume \emph{(A1)}--\emph{(A3)} and the sampling conditions of Lemma~\ref{lem:unbiased_lifted_oracles}. Let
\begin{equation*}
\boldsymbol Y_*^k:=\boldsymbol Y^*(x^k), \quad \boldsymbol Z_*^k:=\boldsymbol Z^*(x^k).
\end{equation*}
Define $x_+^k:=x^k-h^k$ so that
\begin{equation}\label{eq:xplus_relation_YZ}
x^{k+1}-x^k=\alpha_k(x_+^k-x^k),
\end{equation}
and couple the stepsizes as
\begin{equation}\label{eq:beta_gamma_coupling_YZ}
\beta_k=c_1\alpha_k,\quad \gamma_k=c_2\alpha_k\quad (\forall k\ge 0).
\end{equation}
Assume further that
\begin{equation}\label{eq:beta_cond_YZ}
\beta_k\le \bar\beta,
\end{equation}
\begin{equation}\label{eq:gamma_cond_YZ}
\gamma_k\le \bar\gamma,
\end{equation}
for sufficiently small constants $\bar\beta,\bar\gamma$ depending only on the constants in Lemmas~\ref{lem:smoothness_consequences} and \ref{lem:unbiased_lifted_oracles}. Then there exist finite constants $C_{Yx},C_{Y,0},C_{Y,1},C_{Zx},C_{Z,0},C_{Z,1}$ such that for all $K\ge 0$,
\begin{align}
\sum_{k=0}^{K}\alpha_k\,\E\|\boldsymbol Y^k-\boldsymbol Y_*^k\|^2
&\le C_{Yx}\sum_{k=0}^{K}\alpha_k\,\E\|x_+^k-x^k\|^2 + C_{Y,0}+C_{Y,1}\sum_{k=0}^{K}\alpha_k^2,\label{eq:Y_track_sum}\\
\sum_{k=0}^{K}\alpha_k\,\E\|\boldsymbol Z^k-\boldsymbol Z_*^k\|^2
&\le C_{Zx}\sum_{k=0}^{K}\alpha_k\,\E\|x_+^k-x^k\|^2 + C_{Z,0}+C_{Z,1}\sum_{k=0}^{K}\alpha_k^2.\label{eq:Z_track_sum}
\end{align}
\end{lemma}

\begin{proof}
Define the tracking errors
\begin{equation*}
E_Y^k:=\boldsymbol Y^k-\boldsymbol Y_*^k, \quad E_Z^k:=\boldsymbol Z^k-\boldsymbol Z_*^k.
\end{equation*}
We also introduce the constants
\begin{equation*}
\sigma_Y^2:=\omega_P\tilde\sigma_{g,1}^2, \quad \Lambda_Y:=(\omega_P-1)L_{g,1}^2,
\end{equation*}
\begin{equation*}
B_{ZY}:=L_{f,1}+\frac{L_{g,2}L_{f,0}}{\mu_g}, \quad \Lambda_Z:=2\Lambda_S, \quad \bar\sigma_{S,0}^2:=\sigma_{S,0}^2+\frac{2\Lambda_S L_{f,0}^2}{\mu_g^2},
\end{equation*}
where $\sigma_{S,0},\sigma_{S,2},\Lambda_S$ are any valid constants from Lemma~\ref{lem:unbiased_lifted_oracles}. We repeatedly condition on $\mathcal F_k$ and work in the product Frobenius space $\mathcal Y$.

\textbf{Part I: tracking of $\boldsymbol Y^k$. }
The lower update is $\boldsymbol Y^{k+1}=\boldsymbol Y^k-\beta_k\widehat{\boldsymbol V}^{k+1}$ and Lemma~\ref{lem:unbiased_lifted_oracles} gives
\begin{equation*}
\E[\widehat{\boldsymbol V}^{k+1}\mid \mathcal F_k] = \nabla_{\boldsymbol Y}g(x^k,\boldsymbol Y^k).
\end{equation*}

\emph{Step 1: drift of the target point $\boldsymbol Y_*^k$.} Using $\|U+V\|^2\le (1+c)\|U\|^2+(1+1/c)\|V\|^2$ with
\begin{equation*}
U:=\boldsymbol Y^{k+1}-\boldsymbol Y_*^k, \quad V:=\boldsymbol Y_*^k-\boldsymbol Y_*^{k+1}, \quad c:=\beta_k\mu_g,
\end{equation*}
we obtain
\begin{equation}\label{eq:Y_shift_1_product}
\|E_Y^{k+1}\|^2
\le
(1+\beta_k\mu_g)\|\boldsymbol Y^{k+1}-\boldsymbol Y_*^k\|^2
+\left(1+\frac{1}{\beta_k\mu_g}\right)\|\boldsymbol Y_*^{k+1}-\boldsymbol Y_*^k\|^2.
\end{equation}
By \eqref{eq:LY_star} and \eqref{eq:xplus_relation_YZ},
\begin{equation*}
\|\boldsymbol Y_*^{k+1}-\boldsymbol Y_*^k\| \le L_{Y^*}\|x^{k+1}-x^k\| = \alpha_k L_{Y^*}\|x_+^k-x^k\|.
\end{equation*}
Substituting into \eqref{eq:Y_shift_1_product} yields
\begin{equation}\label{eq:Y_shift_2_product}
\|E_Y^{k+1}\|^2
\le
(1+\beta_k\mu_g)\|\boldsymbol Y^{k+1}-\boldsymbol Y_*^k\|^2
+\left(\alpha_k^2+\frac{\alpha_k^2}{\beta_k\mu_g}\right)L_{Y^*}^2\|x_+^k-x^k\|^2.
\end{equation}

\emph{Step 2: recursion for $\|\boldsymbol Y^{k+1}-\boldsymbol Y_*^k\|^2$.} Define
\begin{equation*}
\Xi_V^{k+1} := \widehat{\boldsymbol V}^{k+1}-\nabla_{\boldsymbol Y}g(x^k,\boldsymbol Y^k), \quad \E[\Xi_V^{k+1}\mid \mathcal F_k]=0.
\end{equation*}
Then
\begin{equation}\label{eq:Y_basic_decomp_product}
\E[\|\boldsymbol Y^{k+1}-\boldsymbol Y_*^k\|^2\mid \mathcal F_k]
=
\|\boldsymbol Y^k-\boldsymbol Y_*^k-\beta_k\nabla_{\boldsymbol Y}g(x^k,\boldsymbol Y^k)\|^2
+\beta_k^2\E[\|\Xi_V^{k+1}\|^2\mid \mathcal F_k].
\end{equation}
Because $\boldsymbol Y\mapsto g(x^k,\boldsymbol Y)$ is $\mu_g$-strongly convex with $L_{g,1}$-Lipschitz gradient on $\mathcal Y$, the deterministic gradient step contracts whenever $\beta_k\le \frac{2}{\mu_g+L_{g,1}}$:
\begin{equation}\label{eq:gd_contract_product_Y}
\|\boldsymbol Y^k-\boldsymbol Y_*^k-\beta_k\nabla_{\boldsymbol Y}g(x^k,\boldsymbol Y^k)\|^2
\le
(1-\beta_k\mu_g)^2\|E_Y^k\|^2.
\end{equation}
Also, Lemma~\ref{lem:unbiased_lifted_oracles} gives
\begin{equation}\label{eq:Y_var_bound_product}
\E[\|\Xi_V^{k+1}\|^2\mid \mathcal F_k]
\le
\sigma_Y^2+\Lambda_Y\|E_Y^k\|^2.
\end{equation}
Combining \eqref{eq:Y_basic_decomp_product}--\eqref{eq:Y_var_bound_product} yields
\begin{equation}\label{eq:Y_track_to_Yk_product}
\E[\|\boldsymbol Y^{k+1}-\boldsymbol Y_*^k\|^2\mid \mathcal F_k]
\le
\bigl((1-\beta_k\mu_g)^2+\beta_k^2\Lambda_Y\bigr)\|E_Y^k\|^2
+\beta_k^2\sigma_Y^2.
\end{equation}

\emph{Step 3: combine with the drift bound.} Choose $\bar\beta$ small enough that $\bar\beta\le \min\left\{\frac{2}{\mu_g+L_{g,1}},\frac{1}{4\mu_g},\frac{\mu_g}{8\Lambda_Y}\right\}$ with the convention $\mu_g/(8\Lambda_Y)=+\infty$ when $\Lambda_Y=0$. Then for every $\beta_k\le \bar\beta$,
\begin{equation*}
(1+\beta_k\mu_g)\bigl((1-\beta_k\mu_g)^2+\beta_k^2\Lambda_Y\bigr) \le 1-\frac{\beta_k\mu_g}{4},
\end{equation*}
\begin{equation*}
1+\beta_k\mu_g\le 2, \quad \alpha_k^2+\frac{\alpha_k^2}{\beta_k\mu_g}\le \frac{2\alpha_k^2}{\beta_k\mu_g}.
\end{equation*}
Taking $\E[\cdot\mid \mathcal F_k]$ in \eqref{eq:Y_shift_2_product} and using \eqref{eq:Y_track_to_Yk_product} gives
\begin{equation}\label{eq:Y_one_step_final_product}
\E[\|E_Y^{k+1}\|^2\mid \mathcal F_k]
\le
\left(1-\frac{\beta_k\mu_g}{4}\right)\|E_Y^k\|^2
+\frac{2\alpha_k^2L_{Y^*}^2}{\beta_k\mu_g}\|x_+^k-x^k\|^2
+2\beta_k^2\sigma_Y^2.
\end{equation}
Rearranging, taking total expectation, and summing over $k=0,\ldots,K$ yield
\begin{equation}\label{eq:Y_sum_beta_product}
\frac{\mu_g}{4}\sum_{k=0}^{K}\beta_k\,\E\|E_Y^k\|^2
\le
\E\|E_Y^0\|^2
+\frac{2L_{Y^*}^2}{\mu_g}\sum_{k=0}^{K}\frac{\alpha_k^2}{\beta_k}\E\|x_+^k-x^k\|^2
+2\sigma_Y^2\sum_{k=0}^{K}\beta_k^2.
\end{equation}
Using $\beta_k=c_1\alpha_k$ gives
\begin{equation*}
\sum_{k=0}^{K}\alpha_k\,\E\|E_Y^k\|^2 \le \frac{8L_{Y^*}^2}{c_1^2\mu_g^2}\sum_{k=0}^{K}\alpha_k\,\E\|x_+^k-x^k\|^2 +\frac{4}{c_1\mu_g}\E\|E_Y^0\|^2 +\frac{8c_1}{\mu_g}\sigma_Y^2\sum_{k=0}^{K}\alpha_k^2.
\end{equation*}
This proves \eqref{eq:Y_track_sum}; in particular one may take
\begin{equation*}
C_{Yx}:=\frac{8L_{Y^*}^2}{c_1^2\mu_g^2}, \quad C_{Y,0}:=\frac{4}{c_1\mu_g}\E\|E_Y^0\|^2, \quad C_{Y,1}:=\frac{8c_1}{\mu_g}\sigma_Y^2.
\end{equation*}

\textbf{Part II: tracking of $\boldsymbol Z^k$. }
The auxiliary update is $\boldsymbol Z^{k+1}=\boldsymbol Z^k-\gamma_k\widehat{\boldsymbol S}^{k+1}$ and Lemma~\ref{lem:unbiased_lifted_oracles} gives
\begin{equation*}
\E[\widehat{\boldsymbol S}^{k+1}\mid \mathcal F_k] = \mathcal H(x^k,\boldsymbol Y^k)[\boldsymbol Z^k]-\nabla_{\boldsymbol Y}f(x^k,\boldsymbol Y^k).
\end{equation*}
Write
\begin{equation*}
H^k:=\mathcal H(x^k,\boldsymbol Y^k), \quad f_Y^k:=\nabla_{\boldsymbol Y}f(x^k,\boldsymbol Y^k).
\end{equation*}

\emph{Step 1: drift of the target point $\boldsymbol Z_*^k$.} Using $\|U+V\|^2\le (1+c)\|U\|^2+(1+1/c)\|V\|^2$ with
\begin{equation*}
U:=\boldsymbol Z^{k+1}-\boldsymbol Z_*^k, \quad V:=\boldsymbol Z_*^k-\boldsymbol Z_*^{k+1}, \quad c:=\frac{\gamma_k\mu_g}{3},
\end{equation*}
we obtain
\begin{equation}\label{eq:Z_shift_1_product}
\|E_Z^{k+1}\|^2
\le
\left(1+\frac{\gamma_k\mu_g}{3}\right)\|\boldsymbol Z^{k+1}-\boldsymbol Z_*^k\|^2
+\left(1+\frac{3}{\gamma_k\mu_g}\right)\|\boldsymbol Z_*^{k+1}-\boldsymbol Z_*^k\|^2.
\end{equation}
By \eqref{eq:LZ_star} and \eqref{eq:xplus_relation_YZ},
\begin{equation*}
\|\boldsymbol Z_*^{k+1}-\boldsymbol Z_*^k\| \le L_{Z^*}\|x^{k+1}-x^k\| = \alpha_k L_{Z^*}\|x_+^k-x^k\|.
\end{equation*}
Thus, provided $\gamma_k\mu_g\le 1/4$,
\begin{equation}\label{eq:Z_shift_2_product}
\E[\|E_Z^{k+1}\|^2\mid \mathcal F_k]
\le
\left(1+\frac{\gamma_k\mu_g}{3}\right)\E[\|\boldsymbol Z^{k+1}-\boldsymbol Z_*^k\|^2\mid \mathcal F_k]
+\frac{4\alpha_k^2L_{Z^*}^2}{\gamma_k\mu_g}\|x_+^k-x^k\|^2.
\end{equation}

\emph{Step 2: recursion for $\|\boldsymbol Z^{k+1}-\boldsymbol Z_*^k\|^2$.} Let
\begin{equation*}
\Xi_S^{k+1} := \widehat{\boldsymbol S}^{k+1}-(H^k[\boldsymbol Z^k]-f_Y^k), \quad \E[\Xi_S^{k+1}\mid \mathcal F_k]=0.
\end{equation*}
Then
\begin{equation}\label{eq:Z_basic_decomp_product}
\E[\|\boldsymbol Z^{k+1}-\boldsymbol Z_*^k\|^2\mid \mathcal F_k]
=
\|\boldsymbol Z^k-\boldsymbol Z_*^k-\gamma_k(H^k[\boldsymbol Z^k]-f_Y^k)\|^2
+\gamma_k^2\E[\|\Xi_S^{k+1}\|^2\mid \mathcal F_k].
\end{equation}
Write $H_*^k:=\mathcal H(x^k,\boldsymbol Y_*^k)$ and $f_{Y,*}^k:=\nabla_{\boldsymbol Y}f(x^k,\boldsymbol Y_*^k)$, so that
\begin{equation*}
H_*^k[\boldsymbol Z_*^k]-f_{Y,*}^k=0.
\end{equation*}
Define $B_Z^k:=(H^k-H_*^k)[\boldsymbol Z_*^k]-(f_Y^k-f_{Y,*}^k)$. Then
\begin{equation*}
H^k[\boldsymbol Z^k]-f_Y^k=H^k[E_Z^k]+B_Z^k,
\end{equation*}
hence
\begin{equation*}
\boldsymbol Z^k-\boldsymbol Z_*^k-\gamma_k(H^k[\boldsymbol Z^k]-f_Y^k) =(I-\gamma_k H^k)[E_Z^k]-\gamma_k B_Z^k.
\end{equation*}
Using $\|U+V\|^2\le (1+c)\|U\|^2+(1+1/c)\|V\|^2$ with $c=\gamma_k\mu_g/2$,
\begin{equation}\label{eq:Z_split_bias_product}
\|(I-\gamma_k H^k)[E_Z^k]-\gamma_k B_Z^k\|^2
\le
\left(1+\frac{\gamma_k\mu_g}{2}\right)\|(I-\gamma_k H^k)[E_Z^k]\|^2
+\left(1+\frac{2}{\gamma_k\mu_g}\right)\gamma_k^2\|B_Z^k\|^2.
\end{equation}

Because $H^k$ is self-adjoint on $\mathcal Y$, $\mu_g I\preceq H^k\preceq L_{g,1}I$, and $\gamma_k\le 1/L_{g,1}$, we have
\begin{equation}\label{eq:Z_contract_product}
\|(I-\gamma_k H^k)[E_Z^k]\|^2
\le
(1-\gamma_k\mu_g)^2\|E_Z^k\|^2.
\end{equation}
Moreover, by Lipschitz continuity of $\nabla^2 g$ and $\nabla f$ and the bound \eqref{eq:Zstar_bound},
\begin{equation}\label{eq:BZ_bound_product}
\|B_Z^k\|
\le
\left(L_{g,2}\|\boldsymbol Z_*^k\|+L_{f,1}\right)\|E_Y^k\|
\le
B_{ZY}\|E_Y^k\|.
\end{equation}

Lemma~\ref{lem:unbiased_lifted_oracles} gives
\begin{equation}\label{eq:S_var_raw_product}
\E[\|\Xi_S^{k+1}\|^2\mid \mathcal F_k]
\le
\sigma_{S,0}^2+\sigma_{S,2}^2\|E_Y^k\|^2+\Lambda_S\|\boldsymbol Z^k\|^2.
\end{equation}
Since
\begin{equation*}
\|\boldsymbol Z^k\|^2 \le 2\|E_Z^k\|^2+2\|\boldsymbol Z_*^k\|^2 \le 2\|E_Z^k\|^2+\frac{2L_{f,0}^2}{\mu_g^2},
\end{equation*}
we deduce
\begin{equation}\label{eq:S_var_final_product}
\E[\|\Xi_S^{k+1}\|^2\mid \mathcal F_k]
\le
\Lambda_Z\|E_Z^k\|^2+\sigma_{S,2}^2\|E_Y^k\|^2+\bar\sigma_{S,0}^2.
\end{equation}

\emph{Step 3: collect terms.} Substituting \eqref{eq:Z_split_bias_product}--\eqref{eq:S_var_final_product} into \eqref{eq:Z_basic_decomp_product} yields
\begin{equation}\label{eq:Z_rec_pre_product}
\begin{aligned}
\E[\|\boldsymbol Z^{k+1}-\boldsymbol Z_*^k\|^2\mid \mathcal F_k]
\le{}&
\left(\left(1+\frac{\gamma_k\mu_g}{2}\right)(1-\gamma_k\mu_g)^2+\gamma_k^2\Lambda_Z\right)\|E_Z^k\|^2 \\
&+\left(\left(1+\frac{2}{\gamma_k\mu_g}\right)\gamma_k^2 B_{ZY}^2+\gamma_k^2\sigma_{S,2}^2\right)\|E_Y^k\|^2+\gamma_k^2\bar\sigma_{S,0}^2.
\end{aligned}
\end{equation}
Choose $\bar\gamma$ small enough that $\bar\gamma\le \min\left\{\frac{1}{4\mu_g},\frac{1}{L_{g,1}},\frac{\mu_g}{16\Lambda_Z}\right\}$, with the convention $\mu_g/(16\Lambda_Z)=+\infty$ when $\Lambda_Z=0$. Then for every $\gamma_k\le \bar\gamma$,
\begin{equation*}
\left(1+\frac{\gamma_k\mu_g}{2}\right)(1-\gamma_k\mu_g)^2+\gamma_k^2\Lambda_Z \le 1-\frac{7\gamma_k\mu_g}{8},
\end{equation*}
and
\begin{equation*}
\left(1+\frac{2}{\gamma_k\mu_g}\right)\gamma_k^2 \le \frac{9\gamma_k}{4\mu_g}, \quad \gamma_k^2\sigma_{S,2}^2\le \frac{\gamma_k}{4\mu_g}\sigma_{S,2}^2.
\end{equation*}
Therefore
\begin{equation}\label{eq:Z_rec_simplified_product}
\E[\|\boldsymbol Z^{k+1}-\boldsymbol Z_*^k\|^2\mid \mathcal F_k]
\le
\left(1-\frac{7\gamma_k\mu_g}{8}\right)\|E_Z^k\|^2
+\frac{9B_{ZY}^2+\sigma_{S,2}^2}{4\mu_g}\gamma_k\|E_Y^k\|^2
+\gamma_k^2\bar\sigma_{S,0}^2.
\end{equation}

\emph{Step 4: combine with the drift bound.} Combining \eqref{eq:Z_shift_2_product} and \eqref{eq:Z_rec_simplified_product}, and using $\gamma_k\mu_g\le 1/4$, gives
{\small
\begin{equation}\label{eq:Z_one_step_final_product}
\E[\|E_Z^{k+1}\|^2\mid \mathcal F_k]
\le
\left(1-\frac{\gamma_k\mu_g}{4}\right)\|E_Z^k\|^2+\frac{9B_{ZY}^2+\sigma_{S,2}^2}{2\mu_g}\gamma_k\|E_Y^k\|^2 +\frac{4\alpha_k^2L_{Z^*}^2}{\gamma_k\mu_g}\|x_+^k-x^k\|^2
+2\gamma_k^2\bar\sigma_{S,0}^2.
\end{equation}
}
Rearranging, taking expectation, and summing over $k=0,\ldots,K$ yield
\begin{equation}\label{eq:Z_sum_gamma_product}
\begin{aligned}
\frac{\mu_g}{4}\sum_{k=0}^{K}\gamma_k\,\E\|E_Z^k\|^2
\le{}&
\E\|E_Z^0\|^2
+\frac{9B_{ZY}^2+\sigma_{S,2}^2}{2\mu_g}\sum_{k=0}^{K}\gamma_k\,\E\|E_Y^k\|^2 \\
&+\frac{4L_{Z^*}^2}{\mu_g}\sum_{k=0}^{K}\frac{\alpha_k^2}{\gamma_k}\E\|x_+^k-x^k\|^2
+2\bar\sigma_{S,0}^2\sum_{k=0}^{K}\gamma_k^2.
\end{aligned}
\end{equation}
Using the coupling $\gamma_k=c_2\alpha_k$ and the bound \eqref{eq:Y_track_sum}, we obtain
\begin{equation*}
\sum_{k=0}^{K}\alpha_k\,\E\|E_Z^k\|^2 \le C_{Zx}\sum_{k=0}^{K}\alpha_k\,\E\|x_+^k-x^k\|^2 +C_{Z,0} +C_{Z,1}\sum_{k=0}^{K}\alpha_k^2,
\end{equation*}
where one may take
\begin{equation*}
C_{Zx}:=\frac{16L_{Z^*}^2}{c_2^2\mu_g^2} +\frac{2(9B_{ZY}^2+\sigma_{S,2}^2)}{\mu_g^2}C_{Yx},
\end{equation*}
\begin{equation*}
C_{Z,0}:=\frac{4}{c_2\mu_g}\E\|E_Z^0\|^2 +\frac{2(9B_{ZY}^2+\sigma_{S,2}^2)}{\mu_g^2}C_{Y,0},
\end{equation*}
\begin{equation*}
C_{Z,1}:=\frac{8c_2}{\mu_g}\bar\sigma_{S,0}^2 +\frac{2(9B_{ZY}^2+\sigma_{S,2}^2)}{\mu_g^2}C_{Y,1}.
\end{equation*}
This proves \eqref{eq:Z_track_sum}.
\end{proof}

\begin{lemma}[Hypergradient bias bound]\label{lem:hypergrad_bias}
Suppose that \emph{(A1)} holds and define
\begin{equation}\label{eq:def_Dx_matrix}
D_x(x,\boldsymbol Y,\boldsymbol Z):=\nabla_x f(x,\boldsymbol Y)-\mathcal J(x,\boldsymbol Y)[\boldsymbol Z].
\end{equation}
Then for any $x\in\R^{d_x}$ and any $\boldsymbol Y,\boldsymbol Z\in\mathcal Y$,
\begin{equation}\label{eq:hypergrad_bias_bound}
\|D_x(x,\boldsymbol Y,\boldsymbol Z)-\nabla\Phiobj(x)\|^2
\le C_{xY}\|\boldsymbol Y-\boldsymbol Y^*(x)\|^2+C_{xZ}\|\boldsymbol Z-\boldsymbol Z^*(x)\|^2,
\end{equation}
where one may take
\begin{equation}\label{eq:hypergrad_bias_constants}
C_{xY}:=3\left(L_{f,1}^2+\frac{L_{f,0}^2L_{g,2}^2}{\mu_g^2}\right),
\quad
C_{xZ}:=3L_{g,1}^2.
\end{equation}
\end{lemma}

\begin{proof}
Write $\boldsymbol Y^*:=\boldsymbol Y^*(x)$ and $\boldsymbol Z^*:=\boldsymbol Z^*(x)$. Then
\begin{equation*}
\begin{aligned}
D_x(x,\boldsymbol Y,\boldsymbol Z)-\nabla\Phiobj(x)
={}&\bigl(\nabla_x f(x,\boldsymbol Y)-\nabla_x f(x,\boldsymbol Y^*)\bigr)-\mathcal J(x,\boldsymbol Y)[\boldsymbol Z-\boldsymbol Z^*] \\
&-\bigl(\mathcal J(x,\boldsymbol Y)-\mathcal J(x,\boldsymbol Y^*)\bigr)[\boldsymbol Z^*].
\end{aligned}
\end{equation*}
The first term is controlled by $L_{f,1}$, the second by \eqref{eq:J_bound}, and the third by \eqref{eq:J_Lipschitz} together with \eqref{eq:Zstar_bound}. Applying $(a+b+c)^2\le 3(a^2+b^2+c^2)$ yields the claimed constants.
\end{proof}

\begin{lemma}[Variance of the hypergradient sample and a bound for $\|h^{k+1}-h^k\|$]\label{lem:variance_rso}
Assume \emph{(A1)}--\emph{(A3)} and the sampling conditions in Lemma~\ref{lem:unbiased_lifted_oracles}. Recall
\begin{equation*}
\widehat W^{k+1}:=\tilde U_x^{k+1}-\tilde{\mathcal J}^{k+1}[\boldsymbol P^{k\top}\boldsymbol Z^k]\in\R^{d_x}.
\end{equation*}
Then there exist constants $\sigma_w^2,c_{wZ}\ge 0$ such that
\begin{equation}\label{eq:w_var_rso}
\E[\|\widehat W^{k+1}-\E[\widehat W^{k+1}\mid\mathcal F_k]\|^2]
\le \sigma_w^2+c_{wZ}\,\E\|\boldsymbol Z^k-\boldsymbol Z^*(x^k)\|^2.
\end{equation}
Moreover, if $h^{k+1}=(1-\theta_k)h^k+\theta_k\widehat W^{k+1}$, then
\begin{equation}\label{eq:h_step_rso}
\E\|h^{k+1}-h^k\|^2\le \sigma_{h,k}^2,
\end{equation}
where
\begin{equation}\label{eq:h_step_rso_def}
\begin{aligned}
\sigma_{h,k}^2:=
&\ 2\theta_k^2\E\!\left[
\|h^k-\nabla\Phiobj(x^k)\|^2
+\|\E[\widehat W^{k+1}\mid\mathcal F_k]-\nabla\Phiobj(x^k)\|^2
\right] \\
&+\theta_k^2\Bigl(\sigma_w^2+c_{wZ}\E\|\boldsymbol Z^k-\boldsymbol Z^*(x^k)\|^2\Bigr).
\end{aligned}
\end{equation}
\end{lemma}

\begin{proof}
Fix $k\ge 0$. Lemma~\ref{lem:unbiased_lifted_oracles} gives
\begin{equation*}
\E[\widehat W^{k+1}\mid \mathcal F_k] = \nabla_x f(x^k,\boldsymbol Y^k)-\mathcal J(x^k,\boldsymbol Y^k)[\boldsymbol Z^k].
\end{equation*}
Moreover, \eqref{eq:W_hat_var_matrix} yields
\begin{equation*}
\E[\|\widehat W^{k+1}-\E[\widehat W^{k+1}\mid\mathcal F_k]\|^2\mid \mathcal F_k] \le \tilde\sigma_{f,1}^2+(\tilde\sigma_{g,2}^2+\Lambda_Y)\|\boldsymbol Z^k\|^2.
\end{equation*}
Using the elementary inequality
\begin{equation*}
\|\boldsymbol Z^k\|^2 \le 2\|\boldsymbol Z^k-\boldsymbol Z^*(x^k)\|^2+2\|\boldsymbol Z^*(x^k)\|^2
\end{equation*}
and the uniform bound \eqref{eq:Zstar_bound}, we obtain
\begin{equation*}
\|\boldsymbol Z^k\|^2 \le 2\|\boldsymbol Z^k-\boldsymbol Z^*(x^k)\|^2+\frac{2L_{f,0}^2}{\mu_g^2}.
\end{equation*}
Taking expectations gives \eqref{eq:w_var_rso} with
\begin{equation*}
c_{wZ}:=2(\tilde\sigma_{g,2}^2+\Lambda_Y), \quad \sigma_w^2:=\tilde\sigma_{f,1}^2+c_{wZ}\frac{L_{f,0}^2}{\mu_g^2}.
\end{equation*}

For the increment bound, use
\begin{equation*}
h^{k+1}-h^k=\theta_k(\widehat W^{k+1}-h^k).
\end{equation*}
Conditioning on $\mathcal F_k$ and using the zero mean of $\widehat W^{k+1}-\E[\widehat W^{k+1}\mid \mathcal F_k]$ gives
\begin{equation*}
\E[\|h^{k+1}-h^k\|^2\mid \mathcal F_k]
=\theta_k^2\|\E[\widehat W^{k+1}\mid \mathcal F_k]-h^k\|^2+\theta_k^2\E[\|\widehat W^{k+1}-\E[\widehat W^{k+1}\mid \mathcal F_k]\|^2\mid \mathcal F_k].
\end{equation*}
Using
\begin{equation*}
\|\E[\widehat W^{k+1}\mid \mathcal F_k]-h^k\|^2 \le 2\|h^k-\nabla\Phiobj(x^k)\|^2 +2\|\E[\widehat W^{k+1}\mid \mathcal F_k]-\nabla\Phiobj(x^k)\|^2
\end{equation*}
and then taking expectations proves \eqref{eq:h_step_rso} and \eqref{eq:h_step_rso_def}.
\end{proof}

\begin{lemma}[Moving-average tracking bound for $h^k$]\label{lem:h_tracking_rso}
Suppose that \emph{(A1)}--\emph{(A3)} and the sampling conditions in Lemma~\ref{lem:unbiased_lifted_oracles} hold. Let $x_+^k:=x^k-h^k$, so that $x^{k+1}-x^k=\alpha_k(x_+^k-x^k)$, and assume
\begin{equation*}
0<\theta_k\le 1, \quad \theta_k=c_3\alpha_k
\end{equation*}
for some $c_3>0$. Then for any $K\ge 0$,
{\small
\begin{equation*}
\begin{aligned}
\sum_{k=0}^{K}\alpha_k\E\|h^k\hspace{-0.5em}-\nabla\Phiobj(x^k)\|^2
\le &\frac{1}{c_3}\E\|h^0-\nabla\Phiobj(x^0)\|^2
+2\sum_{k=0}^{K}\alpha_k\,\E\|\E[\widehat W^{k+1}\mid\mathcal F_k]-\nabla\Phiobj(x^k)\|^2\\
&+\hspace{-0.2em}\frac{2L_{\nabla\Phi}^2}{c_3^2}\hspace{-0.2em}\sum_{k=0}^{K}\alpha_k\E\|x_+^k\hspace{-0.5em}-x^k\|^2\hspace{-0.5em}
+\hspace{-0.2em}c_3c_{wZ}\hspace{-0.2em}\sum_{k=0}^{K}\alpha_k^2\E\|\boldsymbol Z^k\hspace{-0.5em}-\hspace{-0.2em}\boldsymbol Z^*(x^k)\|^2\hspace{-0.5em}+\hspace{-0.2em}c_3\sigma_w^2\hspace{-0.2em}\sum_{k=0}^{K}\alpha_k^2.
\end{aligned}
\end{equation*}
}
\end{lemma}

\begin{proof}
Fix $k\ge 0$ and define $e^k:=h^k-\nabla\Phiobj(x^k)$. Using the update $h^{k+1}=(1-\theta_k)h^k+\theta_k\widehat W^{k+1}$ and adding/subtracting $\E[\widehat W^{k+1}\mid \mathcal F_k]$ and $\nabla\Phiobj(x^k)$, we obtain
\begin{equation*}
\begin{aligned}
e^{k+1}
={}&(1-\theta_k)e^k
+\theta_k\bigl(\E[\widehat W^{k+1}\mid \mathcal F_k]-\nabla\Phiobj(x^k)\bigr)+\bigl(\nabla\Phiobj(x^k)-\nabla\Phiobj(x^{k+1})\bigr) \\
&+\theta_k\bigl(\widehat W^{k+1}-\E[\widehat W^{k+1}\mid \mathcal F_k]\bigr).
\end{aligned}
\end{equation*}
Conditioning on $\mathcal F_k$, the last term has zero mean and is orthogonal in expectation to the remaining $\mathcal F_k$-measurable terms. Hence
\begin{equation*}
\begin{aligned}
\E[\|e^{k+1}\|^2\mid \mathcal F_k]
={}&
\Bigl\|(1-\theta_k)e^k
+\theta_k(\E[\widehat W^{k+1}\mid \mathcal F_k]-\nabla\Phiobj(x^k))+\nabla\Phiobj(x^k)-\nabla\Phiobj(x^{k+1})\Bigr\|^2 \\
&+\theta_k^2\E[\|\widehat W^{k+1}-\E[\widehat W^{k+1}\mid \mathcal F_k]\|^2\mid \mathcal F_k].
\end{aligned}
\end{equation*}
By convexity of $\|\cdot\|^2$ and then $\|a+b\|^2\le 2\|a\|^2+2\|b\|^2$,
\begin{equation*}
\begin{aligned}
\E[\|e^{k+1}\|^2\mid \mathcal F_k]
\le{}&
(1-\theta_k)\|e^k\|^2+\theta_k\left\|
\E[\widehat W^{k+1}\mid \mathcal F_k]-\nabla\Phiobj(x^k)
+\frac{1}{\theta_k}(\nabla\Phiobj(x^k)-\nabla\Phiobj(x^{k+1}))
\right\|^2 \\
&+\theta_k^2\E[\|\widehat W^{k+1}-\E[\widehat W^{k+1}\mid \mathcal F_k]\|^2\mid \mathcal F_k] \\
\le{}&
(1-\theta_k)\|e^k\|^2
+2\theta_k\|\E[\widehat W^{k+1}\mid \mathcal F_k]-\nabla\Phiobj(x^k)\|^2+\frac{2}{\theta_k}\|\nabla\Phiobj(x^{k+1})-\nabla\Phiobj(x^k)\|^2 \\
&+\theta_k^2\E[\|\widehat W^{k+1}-\E[\widehat W^{k+1}\mid \mathcal F_k]\|^2\mid \mathcal F_k].
\end{aligned}
\end{equation*}
By the $L_{\nabla\Phi}$-Lipschitz continuity of $\nabla\Phiobj$ and the identity $x^{k+1}-x^k=\alpha_k(x_+^k-x^k)$,
\begin{equation*}
\|\nabla\Phiobj(x^{k+1})-\nabla\Phiobj(x^k)\|^2 \le \alpha_k^2L_{\nabla\Phi}^2\|x_+^k-x^k\|^2.
\end{equation*}
Substituting this bound and using Lemma~\ref{lem:variance_rso} give
\begin{equation*}
\begin{aligned}
\E\|e^{k+1}\|^2
\le{}&
(1-\theta_k)\E\|e^k\|^2
+2\theta_k\E\|\E[\widehat W^{k+1}\mid \mathcal F_k]-\nabla\Phiobj(x^k)\|^2+\frac{2\alpha_k^2L_{\nabla\Phi}^2}{\theta_k}\E\|x_+^k-x^k\|^2 \\
&+\theta_k^2\sigma_w^2
+\theta_k^2c_{wZ}\E\|\boldsymbol Z^k-\boldsymbol Z^*(x^k)\|^2.
\end{aligned}
\end{equation*}
Rearranging and summing from $k=0$ to $K$ yields
{\small
\begin{equation*}
\begin{aligned}
\sum_{k=0}^{K}\theta_k\,\E\|e^k\|^2
\le{}&
\E\|e^0\|^2
+2\sum_{k=0}^{K}\theta_k\E\|\E[\widehat W^{k+1}\mid \mathcal F_k]-\nabla\Phiobj(x^k)\|^2+2L_{\nabla\Phi}^2\sum_{k=0}^{K}\frac{\alpha_k^2}{\theta_k}\E\|x_+^k-x^k\|^2 \\
&+\sigma_w^2\sum_{k=0}^{K}\theta_k^2
+c_{wZ}\sum_{k=0}^{K}\theta_k^2\E\|\boldsymbol Z^k-\boldsymbol Z^*(x^k)\|^2.
\end{aligned}
\end{equation*}
}
Substituting $\theta_k=c_3\alpha_k$ and dividing by $c_3$ proves the claim.
\end{proof}

\begin{lemma}[Primal progress: summability of $\|h^k\|^2$]\label{lem:primal_sum_h}
Assume that $\nabla\Phiobj$ is $L_{\nabla\Phi}$-Lipschitz continuous and that $\Phiobj^*:=\inf_{x\in\R^{d_x}}\Phiobj(x)>-\infty$. Consider the update $x^{k+1}=x^k-\alpha_k h^k$ and define $x_+^k:=x^k-h^k$. Then $x^{k+1}-x^k=\alpha_k(x_+^k-x^k)$ and $\|x_+^k-x^k\|=\|h^k\|$. Moreover, if $\alpha_k\le \frac{1}{2L_{\nabla\Phi}}$ for all $k$, then for any integer $K\ge 0$,
\begin{equation}\label{eq:sum_alpha_hk_bound}
\sum_{k=0}^{K}\alpha_k\,\E\|x_+^k-x^k\|^2
\le 4\bigl(\Phiobj(x^0)-\Phiobj^*\bigr)+2\sum_{k=0}^{K}\alpha_k\,\E\|h^k-\nabla\Phiobj(x^k)\|^2.
\end{equation}
\end{lemma}

\begin{proof}
The identities
\begin{equation*}
x^{k+1}-x^k=\alpha_k(x_+^k-x^k), \quad \|x_+^k-x^k\|=\|h^k\|
\end{equation*}
follow directly from the definition $x_+^k=x^k-h^k$ and the update $x^{k+1}=x^k-\alpha_k h^k$.

Fix $k\ge 0$. By the $L_{\nabla\Phi}$-smoothness of $\Phiobj$,
\begin{equation*}
\Phiobj(x^{k+1}) \le \Phiobj(x^k) +\langle \nabla\Phiobj(x^k),x^{k+1}-x^k\rangle +\frac{L_{\nabla\Phi}}{2}\|x^{k+1}-x^k\|^2.
\end{equation*}
Substituting $x^{k+1}-x^k=-\alpha_k h^k$ gives
\begin{equation*}
\Phiobj(x^{k+1}) \le \Phiobj(x^k)-\alpha_k\langle \nabla\Phiobj(x^k),h^k\rangle +\frac{L_{\nabla\Phi}}{2}\alpha_k^2\|h^k\|^2.
\end{equation*}
Using $2\langle a,b\rangle=\|a\|^2+\|b\|^2-\|a-b\|^2$ with $a=\nabla\Phiobj(x^k)$ and $b=h^k$, we obtain
\begin{equation*}
\langle \nabla\Phiobj(x^k),h^k\rangle \ge \frac12\|h^k\|^2-\frac12\|h^k-\nabla\Phiobj(x^k)\|^2.
\end{equation*}
Therefore
\begin{equation*}
\Phiobj(x^{k+1}) \le \Phiobj(x^k) -\frac{\alpha_k}{2}\|h^k\|^2 +\frac{\alpha_k}{2}\|h^k-\nabla\Phiobj(x^k)\|^2 +\frac{L_{\nabla\Phi}}{2}\alpha_k^2\|h^k\|^2.
\end{equation*}
Rearranging,
\begin{equation*}
\frac{\alpha_k}{2}(1-L_{\nabla\Phi}\alpha_k)\|h^k\|^2 \le \Phiobj(x^k)-\Phiobj(x^{k+1}) +\frac{\alpha_k}{2}\|h^k-\nabla\Phiobj(x^k)\|^2.
\end{equation*}
If $\alpha_k\le \frac{1}{2L_{\nabla\Phi}}$, then $1-L_{\nabla\Phi}\alpha_k\ge 1/2$, and hence
\begin{equation*}
\frac{\alpha_k}{4}\|h^k\|^2 \le \Phiobj(x^k)-\Phiobj(x^{k+1}) +\frac{\alpha_k}{2}\|h^k-\nabla\Phiobj(x^k)\|^2.
\end{equation*}
Taking expectations and summing over $k=0,\ldots,K$ yield
\begin{equation*}
\frac14\sum_{k=0}^{K}\alpha_k\,\E\|h^k\|^2 \le \Phiobj(x^0)-\Phiobj^* +\frac12\sum_{k=0}^{K}\alpha_k\,\E\|h^k-\nabla\Phiobj(x^k)\|^2.
\end{equation*}
Multiplying by $4$ and using $\|x_+^k-x^k\|=\|h^k\|$ proves \eqref{eq:sum_alpha_hk_bound}.
\end{proof}

\subsection{Main convergence theorem}

\begin{theorem}[Convergence rate of \ourslast]\label{thm:rso_masoba_convergence}
Suppose that \emph{(A1)}--\emph{(A3)} hold and that the sampling conditions in Lemma~\ref{lem:unbiased_lifted_oracles} hold. Assume $\Phiobj^*:=\inf_{x\in\R^{d_x}}\Phiobj(x)>-\infty$. Consider the unconstrained \ours updates
\begin{equation*}
\begin{aligned}
x^{k+1}&=x^k-\alpha_k h^k,\\
\boldsymbol Y^{k+1}&=\boldsymbol Y^k-\beta_k\,\widehat{\boldsymbol V}^{k+1},\\
\boldsymbol Z^{k+1}&=\boldsymbol Z^k-\gamma_k\,\widehat{\boldsymbol S}^{k+1},\\
h^{k+1}&=(1-\theta_k)h^k+\theta_k\,\widehat W^{k+1}.
\end{aligned}
\end{equation*}
Define
\begin{equation*}
x_+^k:=x^k-h^k, \quad V_k:=\|x_+^k-x^k\|^2+\|h^k-\nabla\Phiobj(x^k)\|^2.
\end{equation*}
If the stepsizes are constant and coupled as
\begin{equation*}
\alpha_k\equiv \alpha,\quad \beta_k\equiv c_1\alpha,\quad \gamma_k\equiv c_2\alpha,\quad \theta_k\equiv c_3\alpha,
\end{equation*}
with $\alpha>0$ sufficiently small so that Lemmas~\ref{lem:yz_tracking_rso}, \ref{lem:h_tracking_rso}, and \ref{lem:primal_sum_h} apply, define
\begin{equation}\label{eq:def_rho0_matrix_thm}
\rho_0:=\frac{2L_{\nabla\Phi}^2}{c_3^2}+2(C_{xY}C_{Yx}+C_{xZ}C_{Zx}).
\end{equation}
Assume
\begin{equation}\label{eq:rho0_condition_matrix_thm}
\rho_0\le 1/8
\end{equation}
and
\begin{equation}\label{eq:thm_extra_stepsize_cond}
\alpha\le \min\left\{1,\frac{1}{c_3},\frac{1}{8c_3c_{wZ}C_{Zx}}\right\}.
\end{equation}
Then there exist finite constants $C_0,C_1>0$, independent of $K$, such that
\begin{equation}\label{eq:thm_rso_masoba_V_rate_generic}
\frac{1}{K}\sum_{k=0}^{K-1}\E V_k
\le \frac{C_0}{K\alpha}+C_1\alpha.
\end{equation}
In particular, if $\alpha=\Theta(K^{-1/2})$, then
\begin{equation*}
\frac{1}{K}\sum_{k=0}^{K-1}\E\|\nabla\Phiobj(x^k)\|^2=\mathcal O(K^{-1/2}).
\end{equation*}
Moreover, if $R\sim\mathrm{Unif}\{0,\ldots,K-1\}$ is independent of the algorithmic randomness, then
\begin{equation}\label{eq:thm_rso_masoba_random_iterate}
\E\|\nabla\Phiobj(x^R)\|^2\le \frac{2}{K}\sum_{k=0}^{K-1}\E V_k=\mathcal O(K^{-1/2}),
\end{equation}
so to ensure $\E\|\nabla\Phiobj(x^R)\|^2\le \varepsilon$ it suffices to take $K=\mathcal O(\varepsilon^{-2})$.
\end{theorem}

\begin{remark}[Feasibility of the coupling constants]\label{rem:rho0_feasible}
The condition \eqref{eq:rho0_condition_matrix_thm} is non-vacuous. Indeed, Lemma~\ref{lem:yz_tracking_rso} gives
\begin{equation*}
C_{Yx}=\frac{8L_{Y^*}^2}{c_1^2\mu_g^2}, \quad C_{Zx}=\frac{16L_{Z^*}^2}{c_2^2\mu_g^2} +\frac{2(9B_{ZY}^2+\sigma_{S,2}^2)}{\mu_g^2}C_{Yx},
\end{equation*}
so
\begin{equation*}
\rho_0=\frac{2L_{\nabla\Phi}^2}{c_3^2}+2(C_{xY}C_{Yx}+C_{xZ}C_{Zx})
\end{equation*}
can be made arbitrarily small by choosing $c_1,c_2,c_3$ sufficiently large. Hence there always exist finite coupling constants for which \eqref{eq:rho0_condition_matrix_thm} holds.
\end{remark}

\begin{proof}
Let
\begin{equation*}
\Delta_0:=\Phiobj(x^0)-\Phiobj^*, \quad S_x:=\sum_{k=0}^{K-1}\alpha\,\E\|x_+^k-x^k\|^2, \quad S_h:=\sum_{k=0}^{K-1}\alpha\,\E\|h^k-\nabla\Phiobj(x^k)\|^2.
\end{equation*}

\textbf{Step 1: start from the moving-average tracking inequality. }
Lemma~\ref{lem:h_tracking_rso} yields
\begin{equation}\label{eq:thm_step1_h_track_product}
\begin{aligned}
S_h
\le{}&
\frac{1}{c_3}\E\|h^0-\nabla\Phiobj(x^0)\|^2+2\sum_{k=0}^{K-1}\alpha\,\E\|\E[\widehat W^{k+1}\mid\mathcal F_k]-\nabla\Phiobj(x^k)\|^2 \\
&+\frac{2L_{\nabla\Phi}^2}{c_3^2}S_x+c_3c_{wZ}\sum_{k=0}^{K-1}\alpha^2\E\|\boldsymbol Z^k-\boldsymbol Z^*(x^k)\|^2
+c_3\sigma_w^2K\alpha^2.
\end{aligned}
\end{equation}

\textbf{Step 2: bound the hypergradient bias term via $\boldsymbol Y/\boldsymbol Z$ tracking. }
By Lemma~\ref{lem:unbiased_lifted_oracles} and the definition \eqref{eq:def_Dx_matrix},
\begin{equation*}
\E[\widehat W^{k+1}\mid \mathcal F_k]=D_x(x^k,\boldsymbol Y^k,\boldsymbol Z^k).
\end{equation*}
Lemma~\ref{lem:hypergrad_bias} therefore implies
\begin{equation*}
\E\|\E[\widehat W^{k+1}\mid \mathcal F_k]-\nabla\Phiobj(x^k)\|^2 \le C_{xY}\E\|\boldsymbol Y^k-\boldsymbol Y^*(x^k)\|^2 +C_{xZ}\E\|\boldsymbol Z^k-\boldsymbol Z^*(x^k)\|^2.
\end{equation*}
Summing with weight $\alpha$ and using Lemma~\ref{lem:yz_tracking_rso} give
\begin{equation}\label{eq:bias_sum_product}
\sum_{k=0}^{K-1}\alpha\,\E\|\E[\widehat W^{k+1}\mid \mathcal F_k]-\nabla\Phiobj(x^k)\|^2
\le
(C_{xY}C_{Yx}+C_{xZ}C_{Zx})S_x+B_0+B_1K\alpha^2,
\end{equation}
where
\begin{equation*}
B_0:=C_{xY}C_{Y,0}+C_{xZ}C_{Z,0}, \quad B_1:=C_{xY}C_{Y,1}+C_{xZ}C_{Z,1}.
\end{equation*}

\textbf{Step 3: control the $\alpha^2\|\boldsymbol Z^k-\boldsymbol Z^*(x^k)\|^2$ term. }
Since $\alpha_k\equiv \alpha$,
\begin{equation*}
\sum_{k=0}^{K-1}\alpha^2\,\E\|\boldsymbol Z^k-\boldsymbol Z^*(x^k)\|^2 = \alpha\sum_{k=0}^{K-1}\alpha\,\E\|\boldsymbol Z^k-\boldsymbol Z^*(x^k)\|^2.
\end{equation*}
Applying the second inequality in Lemma~\ref{lem:yz_tracking_rso} yields
\begin{equation}\label{eq:alpha2_Z_bound_product_thm}
\sum_{k=0}^{K-1}\alpha^2\,\E\|\boldsymbol Z^k-\boldsymbol Z^*(x^k)\|^2
\le
\alpha\bigl(C_{Zx}S_x+C_{Z,0}+C_{Z,1}K\alpha^2\bigr).
\end{equation}
Since \eqref{eq:thm_extra_stepsize_cond} includes $\alpha\le 1$, the last term on the right satisfies $K\alpha^3\le K\alpha^2$.

\textbf{Step 4: close the recursion using primal descent. }
Substituting \eqref{eq:bias_sum_product} and \eqref{eq:alpha2_Z_bound_product_thm} into \eqref{eq:thm_step1_h_track_product} gives
\begin{equation}\label{eq:Sh_vs_Sx_product_thm}
S_h
\le
A_{h,0}
+\bigl(\rho_0+c_3c_{wZ}C_{Zx}\alpha\bigr)S_x+A_{h,1}K\alpha^2+A_{h,2}\alpha,
\end{equation}
where
\begin{equation*}
A_{h,0}:=\frac{1}{c_3}\E\|h^0-\nabla\Phiobj(x^0)\|^2+2B_0,
\end{equation*}
\begin{equation*}
A_{h,1}:=2B_1+c_3c_{wZ}C_{Z,1}+c_3\sigma_w^2, \quad A_{h,2}:=c_3c_{wZ}C_{Z,0}.
\end{equation*}
On the other hand, Lemma~\ref{lem:primal_sum_h} gives
\begin{equation}\label{eq:primal_Sx_product_thm}
S_x\le 4\Delta_0+2S_h.
\end{equation}
Combining \eqref{eq:Sh_vs_Sx_product_thm} and \eqref{eq:primal_Sx_product_thm} yields
\begin{equation*}
S_h \le A_{h,0} +\bigl(\rho_0+c_3c_{wZ}C_{Zx}\alpha\bigr)(4\Delta_0+2S_h) +A_{h,1}K\alpha^2+A_{h,2}\alpha.
\end{equation*}
By \eqref{eq:rho0_condition_matrix_thm} and \eqref{eq:thm_extra_stepsize_cond},
\begin{equation*}
\rho_0+c_3c_{wZ}C_{Zx}\alpha\le \frac18+\frac18=\frac14,
\end{equation*}
and therefore
\begin{equation*}
S_h \le A_{h,0}+\Delta_0+\frac12 S_h+A_{h,1}K\alpha^2+A_{h,2}\alpha.
\end{equation*}
Rearranging gives
\begin{equation}\label{eq:Sh_final_product_thm}
S_h
\le
2A_{h,0}+2\Delta_0+2A_{h,1}K\alpha^2+2A_{h,2}\alpha.
\end{equation}
Substituting \eqref{eq:Sh_final_product_thm} into \eqref{eq:primal_Sx_product_thm} yields
\begin{equation}\label{eq:Sx_final_product_thm}
S_x
\le
4\Delta_0+4A_{h,0}+4\Delta_0+4A_{h,1}K\alpha^2+4A_{h,2}\alpha.
\end{equation}

\textbf{Step 5: divide by $K\alpha$. }
Dividing \eqref{eq:Sh_final_product_thm} and \eqref{eq:Sx_final_product_thm} by $K\alpha$ shows that
\begin{equation*}
\frac{S_h}{K\alpha} \le \frac{\widetilde C_{h,0}}{K\alpha} +\widetilde C_{h,1}\alpha +\frac{\widetilde C_{h,2}}{K} \quad \text{and} \quad \frac{S_x}{K\alpha} \le \frac{\widetilde C_{x,0}}{K\alpha} +\widetilde C_{x,1}\alpha +\frac{\widetilde C_{x,2}}{K},
\end{equation*}
for suitable finite constants $\widetilde C_{\cdot,\cdot}$ independent of $K$. Because $\alpha\le 1$, we have $1/K\le 1/(K\alpha)$, so the $1/K$ terms can be absorbed into the $1/(K\alpha)$ term. Summing the last two inequalities proves \eqref{eq:thm_rso_masoba_V_rate_generic} for some finite constants $C_0,C_1>0$.

If $\alpha=\Theta(K^{-1/2})$, then both right-hand sides are $\mathcal O(K^{-1/2})$, which proves the displayed rate for
\begin{equation*}
\frac{1}{K}\sum_{k=0}^{K-1}\E\|h^k-\nabla\Phiobj(x^k)\|^2 \quad\text{and}\quad \frac{1}{K}\sum_{k=0}^{K-1}\E\|x_+^k-x^k\|^2.
\end{equation*}
Finally,
\begin{equation*}
\|\nabla\Phiobj(x^k)\|^2 \le 2\|x_+^k-x^k\|^2+2\|h^k-\nabla\Phiobj(x^k)\|^2 = 2V_k,
\end{equation*}
which yields \eqref{eq:thm_rso_masoba_random_iterate} and the $\mathcal O(\varepsilon^{-2})$ complexity claim.
\end{proof}

\begin{corollary}[An explicit constant-stepsize choice]\label{cor:rso_masoba_rate}
Under the assumptions of Theorem~\ref{thm:rso_masoba_convergence}, fix any constants $c_1,c_2,c_3>0$ such that the step-size ratio conditions of Lemma~\ref{lem:yz_tracking_rso} hold and \eqref{eq:rho0_condition_matrix_thm} is satisfied; such constants exist by Remark~\ref{rem:rho0_feasible}. Let $\bar\alpha>0$ be any sufficiently small constant so that all hypotheses of Theorem~\ref{thm:rso_masoba_convergence} hold whenever $\alpha\le \bar\alpha$. If we choose
\begin{equation*}
\alpha_k\equiv\alpha:=\min\{\bar\alpha,1/\sqrt K\}, \quad \beta_k\equiv c_1\alpha,\quad \gamma_k\equiv c_2\alpha,\quad \theta_k\equiv c_3\alpha,
\end{equation*}
then
\begin{equation*}
\E\|\nabla\Phiobj(x^R)\|^2\le 2\left(\frac{C_0}{K\alpha}+C_1\alpha\right).
\end{equation*}
In particular, if $K\ge 1/\bar\alpha^2$, then $\alpha=1/\sqrt K$ and
\begin{equation*}
\E\|\nabla\Phiobj(x^R)\|^2\le \frac{2(C_0+C_1)}{\sqrt K}.
\end{equation*}
Consequently, to guarantee $\E\|\nabla\Phiobj(x^R)\|^2\le \varepsilon$ it suffices to take $K=\mathcal O(\varepsilon^{-2})$.
\end{corollary}

\subsection{A quadratic counterexample for naive projected auxiliary HVPs}
\label{app:counterexample}

The correction in \ours is needed because a naive projected auxiliary HVP can change the fixed point of the auxiliary linear system. The following deterministic quadratic example isolates this issue. There is no stochastic objective noise and no lower-level tracking error; the only error comes from replacing the full HVP $Hz$ by the uncorrected projected HVP $QHQz$.

\begin{example}[Naive projected auxiliary HVP can give a nonstationary limit]
\label{ex:naive_projected_hvp_counterexample}
Consider a deterministic bilevel problem with scalar upper variable $x\in\R$ and one lower-level matrix block $Y\in\R^{3\times 1}$, which we identify with a vector $y\in\R^3$:
\begin{equation*}
g(x,y)=\frac{1}{2}y^\top H y-xc^\top y, \quad f(x,y)=\frac{1}{2}x^2+d^\top y,
\end{equation*}
where $H=\operatorname{diag}(1,2,3)$ and $c=d=e_1=(1,0,0)^\top$. The lower problem is strongly convex, and its solution is
\begin{equation*}
y^*(x)=H^{-1}cx=xe_1.
\end{equation*}
The auxiliary linear system for the true hypergradient is
\begin{equation*}
Hz^*=d, \quad z^*=H^{-1}d=e_1.
\end{equation*}
Moreover, since $\nabla_{xy}^2g[z]=-c^\top z$, the true hypergradient is
\begin{equation*}
\nabla\Phi(x)=\nabla_x f(x,y^*(x))-\nabla_{xy}^2g(x,y^*(x))[z^*]=x+c^\top z^*=x+1.
\end{equation*}
Thus the unique stationary point of the true bilevel objective is $x^*=-1$.
Now sample a scaled subspace lifting matrix
\begin{equation*}
P=\sqrt{\frac{3}{2}}\,U_2\in\R^{3\times 2}, \quad Q=PP^\top,
\end{equation*}
where $U_2$ contains the first two columns of a Haar-uniform orthogonal matrix in $O(3)$. Then $\E[Q]=I$, so the lifted lower gradient remains unbiased. However, the naive projected auxiliary HVP has mean $\bar H z:=\E[QHQ]z$. By Lemma~\ref{lem:projection_moments_matrix}(ii), with $m=3$ and $r=2$, we have $a=\frac{21}{20}$ and $b=\frac{3}{20}$, and therefore
\begin{equation*}
\bar H=aH+b\operatorname{tr}(H)I=\operatorname{diag}\left(\frac{39}{20},3,\frac{81}{20}\right).
\end{equation*}
Hence the mean field of the naive projected auxiliary recursion is
\begin{equation*}
z^{k+1}=z^k-\gamma(\bar H z^k-d).
\end{equation*}
For any $0<\gamma<2/\lambda_{\max}(\bar H)$, this recursion converges to $\bar z=\bar H^{-1}d=\frac{20}{39}e_1$, which differs from the true auxiliary solution $z^*=e_1$.

Even if the lower variable is kept exactly at $y^*(x)$, an upper update using this limiting biased auxiliary variable follows the direction
\begin{equation*}
\bar h(x)=x+c^\top \bar z=x+\frac{20}{39}.
\end{equation*}
Thus the recursion
\begin{equation*}
x^{k+1}=x^k-\eta \bar h(x^k)
\end{equation*}
converges, for any $0<\eta<2$, to $\bar x=-\frac{20}{39}$. However, this point is not stationary for the true bilevel objective, since
\begin{equation*}
\nabla\Phi(\bar x)=\bar x+1=\frac{19}{39}\neq 0.
\end{equation*}
Consequently, for any $\varepsilon<\left(\frac{19}{39}\right)^2$, the naive projected auxiliary update cannot guarantee $\E\|\nabla\Phi(x)\|^2\le \varepsilon$ even in this noiseless quadratic problem.
\end{example}

This example shows that correcting the projected auxiliary HVP is necessary for targeting the original bilevel stationary point. The issue is not caused by stochastic variance or imperfect lower-level optimization; it is caused by the biased mean of the uncorrected projected auxiliary HVP.

\section{Memory analysis for Non-GQA Transformer blocks}
\label{sec:memory_compute_analysis}

\subsection{Transformer layer and memory components}
\label{subsec:memory_complexity_analysis}

We follow the LLaMA-style Transformer layer used in the CR-Net memory discussion~\citep{kong2025cr}. Let $b$ be the micro-batch size, $s$ the sequence length, $n$ the hidden size, and $h$ the number of attention heads. In the non-GQA case, the query, key, and value projections all have output width $n$. Folding the batch and sequence dimensions into the row dimension, a layer input can be viewed as $X_l\in\mathbb R^{bs\times n}$. Omitting layer normalization and RoPE, the attention part is
\begin{equation*}
\begin{aligned}
Y_l^Q&:=X_lW_l^Q,\quad Y_l^K:=X_lW_l^K,\quad Y_l^V:=X_lW_l^V,\\
\mathrm{Att}_l&:=\operatorname{softmax}\!\left(\frac{Y_l^Q(Y_l^K)^\top}{\sqrt{n/h}}\right)Y_l^VW_l^O+X_l.
\end{aligned}
\end{equation*}
With SwiGLU and intermediate width $m=\frac{8}{3}n$, the feed-forward part uses $W_l^g,W_l^u\in\mathbb R^{n\times m}$ and $W_l^d\in\mathbb R^{m\times n}$:
\begin{equation*}
X_{l+1}:=\left(\operatorname{SwiGLU}(\mathrm{Att}_lW_l^g)\odot(\mathrm{Att}_lW_l^u)\right)W_l^d+\mathrm{Att}_l.
\end{equation*}

We count scalar memory slots for one trainable decoder block and omit datatype constants. For a full-space method, the non-GQA trainable tuple is
\begin{equation*}
\boldsymbol Y=(Y_{qkv},Y_o,Y_{ff},Y_d),
\end{equation*}
where $Y_{qkv}:=[Y_q,Y_k,Y_v]\in\mathbb R^{n\times3n}$, $Y_o\in\mathbb R^{n\times n}$, $Y_{ff}:=[Y_g,Y_u]\in\mathbb R^{n\times\frac{16}{3}n}$, and $Y_d\in\mathbb R^{\frac{8}{3}n\times n}$. Hence
\begin{equation*}
S_{\rm full}:=|\boldsymbol Y|=3n^2\;(\text{QKV})+n^2\;(\text{O})+\frac{16}{3}n^2\;(\text{FFN up/gate})+\frac{8}{3}n^2\;(\text{FFN down})=12n^2.
\end{equation*}
The trainable-side activations stored by a full-space backward pass have size
\begin{equation*}
A_{\rm full}:=15bsn\;(\text{saved hidden})+2bhs^2\;(\text{attention}),
\end{equation*}
where the $2bhs^2$ term comes from attention-score and attention-probability tensors.

\subsection{BROS memory in the non-GQA case}
\label{subsec:bros_memory_non_gqa}

For \ourslast, we use randomized projectors $P_{qkv},P_o,P_{ff}\in\mathbb R^{n\times r}$ and $P_d\in\mathbb R^{\frac{8}{3}n\times r}$. The projected lower and auxiliary variables have the same block shapes
\begin{equation*}
B_{qkv},C_{qkv}\in\mathbb R^{r\times3n},\quad B_o,C_o\in\mathbb R^{r\times n},\quad B_{ff},C_{ff}\in\mathbb R^{r\times\frac{16}{3}n},\quad B_d,C_d\in\mathbb R^{r\times n}.
\end{equation*}
Thus one projected tuple has size
\begin{equation*}
S_{\rm proj}:=3rn\;(\text{QKV})+rn\;(\text{O})+\frac{16}{3}rn\;(\text{FFN up/gate})+rn\;(\text{FFN down})=\frac{31}{3}rn.
\end{equation*}
\ours stores only random seeds for the projectors and reconstructs them when needed, so there is no persistent projector-storage term. The projected trainable-side activation footprint is
\begin{equation*}
A_{\rm BROS}:=\frac{28}{3}bsn\;(\text{hidden})+2bhs^2\;(\text{attention})+4bsr\;(\text{projected inputs}).
\end{equation*}

The peak-memory proxy below focuses on trainable-side lower-network memory and omits upper-variable slots. We decompose the counted peak-memory proxy as
\begin{equation*}
M_{\rm peak}=M_{\rm state}\;(\text{states})+M_{\rm act}\;(\text{activations})+M_{\rm dir}\;(\text{gradients/directions}).
\end{equation*}
Here $M_{\rm state}$ counts persistent lower and auxiliary variables, $M_{\rm act}$ counts saved trainable-side activations, and $M_{\rm dir}$ counts online lower-gradient and auxiliary-gradient directions. The HVP/JVP autodiff tapes are implementation-dependent, so we report their query scale separately rather than hiding them inside a fixed constant. For \ourslast, the theorem-aligned recursion stores the full lower/auxiliary iterates $(\boldsymbol Y^k,\boldsymbol Z^k)$, so $M_{\rm state}^{\text{\ourslast}}=2S_{\rm full}=24n^2$, while the projected lower and auxiliary gradient directions have size $M_{\rm dir}^{\text{\ourslast}}=2S_{\rm proj}=\frac{62}{3}rn$. Therefore
\begin{equation}
\label{eq:peak_mabros_sgd_final}
M_{\rm peak}^{\text{\ourslast}}=24n^2\;(\text{state})+\frac{28}{3}bsn\;(\text{hidden})+2bhs^2\;(\text{attn})+4bsr\;(\text{proj-act})+\frac{62}{3}rn\;(\text{grad/dir}).
\end{equation}
For the full-space exact single-loop baseline, the persistent lower/auxiliary states are the same, the trainable-side activations remain $A_{\rm full}$, and the lower/auxiliary gradient directions contribute $2S_{\rm full}=24n^2$. Table~\ref{tab:peak_memory_no_gqa} compares the resulting peak-memory proxy.

\begin{table}[t!]
\centering
\caption{\small Non-GQA lower-network peak-memory proxy for full-space exact recursion and \ours on one trainable LLaMA-style decoder block. The HVP/JVP column reports derivative-query scale rather than a backend-specific peak constant.}
\label{tab:peak_memory_no_gqa}
{\small
\setlength{\tabcolsep}{4.5pt}
\setlength{\extrarowheight}{1pt}
\renewcommand{\arraystretch}{1.15}
\resizebox{\textwidth}{!}{%
\begin{tabular}{lccccc}
\toprule
\textbf{Method} & \textbf{Lower/aux. states} & \textbf{Activations} & \textbf{Grad./directions} & \textbf{HVP/JVP scale} & \textbf{Peak proxy} \\
\midrule
Full-space exact & $24n^2$ & $15bsn+2bhs^2$ & $24n^2$ & full-space HVP/JVP & $48n^2+15bsn+2bhs^2$ \\
\rowcolor{boxbg}
\ourslast & $24n^2$ & $\frac{28}{3}bsn+2bhs^2+4bsr$ & $\frac{62}{3}rn$ & projected HVP/JVP & $24n^2+\frac{28}{3}bsn+2bhs^2+4bsr+\frac{62}{3}rn$ \\
\bottomrule
\end{tabular}}}
\end{table}

\subsection{Memory of other bilevel baselines}
\label{subsec:baseline_memory_non_gqa}

We next place representative baselines from the references into the same non-GQA block model under the same peak-memory proxy. MA-SOBA keeps the lower iterate and an auxiliary linear-system iterate, and its lower/auxiliary gradient directions are full-space~\citep{chen2024optimal}. FdeHBO keeps the lower iterate $y_t$, the linear-system auxiliary iterate $v_t$, one-step-lag iterates, and full-space recursive momentum estimators for the lower and auxiliary recursions~\citep{yang2023achieving}. Penalty methods avoid the auxiliary linear system, so the favorable finite-penalty row only counts the lower state and full-space lower-gradient direction~\citep{mehra2021penalty}. Bilevel-ZOFO is omitted because its memory is governed by a PEFT lower-level parameterization and a zeroth-order upper update, rather than the full trainable lower-network model considered in this table~\citep{shirkavand2025bilevel}.

\begin{table}[t!]
\centering
\caption{\small Baseline lower-network peak-memory proxy under the same non-GQA decoder-block model. The HVP/JVP column reports derivative-query scale rather than a backend-specific peak constant.}
\label{tab:additional_baseline_memory}
{\small
\setlength{\tabcolsep}{4.5pt}
\setlength{\extrarowheight}{1pt}
\renewcommand{\arraystretch}{1.15}
\resizebox{\textwidth}{!}{%
\begin{tabular}{lccccc}
\toprule
\textbf{Method} & \textbf{Lower/aux. states} & \textbf{Activations} & \textbf{Grad./directions} & \textbf{HVP/JVP scale} & \textbf{Peak proxy} \\
\midrule
MA-SOBA~\citep{chen2024optimal} & $24n^2$ & $15bsn+2bhs^2$ & $24n^2$ & full-space HVP/JVP & $48n^2+15bsn+2bhs^2$ \\
FdeHBO~\citep{yang2023achieving} & $48n^2$ & $15bsn+2bhs^2$ & $24n^2$ & full-space finite-diff. directions & $72n^2+15bsn+2bhs^2$ \\
Penalty~\citep{mehra2021penalty} & $12n^2$ & $15bsn+2bhs^2$ & $12n^2$ & no auxiliary HVP/JVP & $24n^2+15bsn+2bhs^2$ \\
\rowcolor{boxbg}
\ourslast & $24n^2$ & $\frac{28}{3}bsn+2bhs^2+4bsr$ & $\frac{62}{3}rn$ & projected HVP/JVP & $24n^2+\frac{28}{3}bsn+2bhs^2+4bsr+\frac{62}{3}rn$ \\
\bottomrule
\end{tabular}}}
\end{table}

\section{Additional experiments}
\label{app:additional_experiments}

Throughout this appendix, the experimental shorthand $r/n$ denotes the chosen subspace rank $r$ divided by the relevant layer width or hidden size $n$. In the theoretical notation, the corresponding projector ratio is $r_\ell/m_\ell$ for each block $Y_\ell$.

\subsection{MNIST hyper-data cleaning}
We consider the convex MNIST hyper-data-cleaning problem~\citep{shaban2019truncated}, where the upper-level variable assigns sample weights and the lower-level problem trains a classifier on corrupted labels. The MNIST training split is divided into $20$k noisy training examples and $5$k clean validation examples, and the standard $10$k clean test set is used for final evaluation. The training labels are corrupted by $50\%$ random label noise, while validation and test labels remain clean. The lower model is a convex logistic-regression classifier with per-sample hyper-weights. All methods use the same data split, batch size $128$, seed $42$, $6000$ outer steps, and evaluation interval $100$ steps.
All MNIST runs are conducted on a single NVIDIA A100 GPU ($80$ GB).

The common regularization and optimization settings are $C_r=0.05$, gradient clipping $1.0$, exact HVP/JVP evaluation, power learning-rate decay with inner decay $0.4$ and outer decay $0.6$, and $\delta_{\epsilon}=0.01$ unless specified otherwise. Learning rates are selected under the same validation-based tuning protocol for all methods. We search inner learning rates in $\{0.03,0.06,0.1,0.2\}$ and outer learning rates in $\{0.03,0.05,0.1,0.3,0.6,1.0,2.0,3.0\}$, and report the best-performing values for each method. MA-SOBA uses inner learning rate $0.1$, outer learning rate $3.0$, $\theta=0.08$, $z$ learning rate $0.02$, and $z$ clipping $1000$. \ours uses inner learning rate $0.1$, outer learning rate $2.0$, $\delta_{\epsilon}=0.005$, $\theta=0.08$, $z$ learning rate $0.02$, $z$ clipping $1000$, and experimental rank ratio $r/n=0.25$ in the main run. The rank-ratio ablation uses $r/n\in\{0.01,0.05,0.10,0.25,0.33,0.50\}$. Penalty uses inner learning rate $0.06$, outer learning rate $0.3$, penalty coefficient $0.03$, and inner clean coefficient $0.0$; FdeHBO uses inner learning rate $0.06$, outer learning rate $0.6$, $\delta_{\epsilon}=0.03$, and momentum $0.4$; ZOFO uses inner learning rate $0.06$, outer learning rate $0.05$, perturbation scale $0.02$, and gradient clipping $0.25$.

\begin{table}[t!]
\centering
\caption{Model-configuration hyperparameters for MNIST hyper-data cleaning.}
\label{tab:mnist_hyperparams}
\small
\setlength{\tabcolsep}{4pt}
\renewcommand{\arraystretch}{1.12}
\begin{threeparttable}
\begin{tabular}{cccccc}
\toprule
Lower model & Input dim. & Classes & Lower weights & Upper weights & Layers \\
\midrule
Logistic regression & $784$ & $10$ & $10\times 784$ & $20000$ & $1$ \\
\bottomrule
\end{tabular}
\end{threeparttable}
\end{table}

\subsection{Data-mixture learning}
We evaluate data-mixture learning on $17$ copyright-free Pile domains. The experiment follows a proxy-to-main pipeline: a $280$M GPT-style proxy model first learns domain weights, and a separate $280$M main model is then trained from scratch using the learned mixture. The tokenizer is EleutherAI/gpt-neox-20b~\cite{black2022gpt}, the sequence length is $256$, and the proxy, reference, and main models all use $16$ layers, hidden size $1024$, $16$ attention heads, vocabulary size $50304$, and dropout $0.0$. The reference model is trained for $100000$ steps with the initial data mixture and is used for DoReMi excess-loss evaluation and proxy-stage reference losses.
All data-mixture runs are conducted on $8$ NVIDIA RTX 4090 GPUs ($24$ GB each) with DDP.

Proxy training uses global batch size $128$, micro-batch size $16$, gradient accumulation $1$, bfloat16 precision, gradient checkpointing, seed $42$, warmup $2000$ steps, and $100000$ total steps. All bilevel data-mixture methods start from the uniform weight vector over the $17$ domains and use inner learning rate $0.15$, outer learning rate $0.01$, $C_r=5\times10^{-4}$, gradient clipping $5.0$, smoothing constant $10^{-3}$, minimum weight floor $5\times10^{-5}$, cosine schedule, and Neumann depth $2$. MA-SOBA uses $\delta_{\epsilon}=0.01$, $z$ learning rate $0.02$, $z$ clipping $1000$, $\theta=0.7$, and $\theta_{\min}=0.05$; \ours uses the same moving-average settings with experimental rank ratio $r/n=0.25$. Penalty uses penalty coefficient $0.1$ and $\delta_{\epsilon}=0.001$, FdeHBO uses $\delta_{\epsilon}=0.01$ and momentum $0.9$, and ZOFO uses perturbation scale $0.02$, one perturbation, and upper update interval $10$. Main-model training is from scratch with the learned final domain weights, total steps $200000$, global batch size $128$, micro-batch size $16$, learning rate $3\times10^{-4}$, minimum learning rate $3\times10^{-5}$, warmup ratio $0.06$, weight decay $0.1$, gradient clipping $1.0$, and bfloat16 precision.

\begin{table}[t!]
\centering
\caption{Model-configuration hyperparameters for data-mixture learning.}
\label{tab:data_mixture_hyperparams}
\small
\setlength{\tabcolsep}{4pt}
\renewcommand{\arraystretch}{1.12}
\begin{threeparttable}
\begin{tabular}{cccccc}
\toprule
Parameters & Hidden & Intermediate & KV heads & Heads & Layers \\
\midrule
$280$M & $1024$ & $4096$ & $16$ & $16$ & $16$ \\
\bottomrule
\end{tabular}
\end{threeparttable}
\end{table}

In addition to the main-model loss curves in Figure~\ref{fig:data_mixture_main}, we report proxy-stage system profiles for the data-mixture baselines. Table~\ref{tab:data_mixture_memory} summarizes the peak memory and average wall-clock time per proxy-training step. \ours has the smallest memory footprint among all compared methods, reducing peak memory to $55.1\%$ of MA-SOBA while keeping the step time close to MA-SOBA, Penalty, and FdeHBO. Compared with ZOFO, which attains a similar final loss but relies on slower zeroth-order updates, \ours uses less memory and runs substantially faster per step. These results show that the randomized-subspace recursion improves memory efficiency without materially compromising proxy-stage training efficiency.

\begin{table}[t!]
\centering
\caption{Memory and per-step time for different algorithms in data-mixture learning.}
\label{tab:data_mixture_memory}
\small
\setlength{\tabcolsep}{7pt}
\renewcommand{\arraystretch}{1.12}
\begin{tabular}{lccc}
\toprule
Method & Peak memory & Memory vs. MA-SOBA & Per-step time (s) \\
\midrule
DoReMi~\citep{xie2023doremi} & $9447.0$ MB & $70.0\%$ & $0.149$ \\
MA-SOBA~\citep{chen2024optimal} & $13496.9$ MB & $100.0\%$ & $0.278$ \\
Penalty~\citep{mehra2021penalty} & $10190.5$ MB & $75.5\%$ & $0.272$ \\
FdeHBO~\citep{yang2023achieving} & $12896.3$ MB & $95.6\%$ & $0.260$ \\
ZOFO~\citep{shirkavand2025bilevel} & $10047.1$ MB & $74.4\%$ & $0.556$ \\
\rowcolor{boxbg}
\ourslast & $7436.0$ MB & $55.1\%$ & $0.301$ \\
\bottomrule
\end{tabular}
\end{table}

\subsection{Hyper-representation learning}
We evaluate CIFAR-10 hyper-representation learning~\citep{krizhevsky2009learning}, where the upper-level objective evaluates the representation and the lower-level problem trains a classifier under the current representation. The data split is $45$k training examples, $5$k validation examples, and $10$k test examples. The lower model is a ResNet-18 trained from scratch. All methods are trained for $30000$ steps with batch size $128$, seed $42$, CUDA with TF32 enabled, cosine decay, gradient clipping $20.0$, and exact HVP/JVP evaluation; we record wall-clock time and test accuracy at each evaluation point.
All CIFAR-10 runs are conducted on a single NVIDIA A100 GPU ($80$ GB).

The common regularization is $C_r=5\times10^{-4}$, with inner decay $0.4$ and outer decay $0.6$. Learning rates are selected under the same validation-based tuning protocol for all methods. We search inner learning rates in $\{0.05,0.08,0.1,0.2,0.4\}$ and outer learning rates in $\{0.005,0.008,0.01,0.02,0.05,0.1\}$, and report the best-performing values for each method. MA-SOBA uses inner learning rate $0.4$, outer learning rate $0.1$, $\delta_{\epsilon}=0.005$, $\theta=0.1$, $\theta_{\min}=0.05$, $z$ learning rate $0.1$, and $z$ clipping $4000$. \ours uses the same inner and outer learning rates, $\delta_{\epsilon}=0.005$, $\theta=0.1$, $z$ learning rate $0.1$, $z$ clipping $4000$, and experimental rank ratio $r/n=0.25$ in the main run. The ablation compares $r/n\in\{0.50,0.25,0.10\}$. Penalty uses inner learning rate $0.08$, outer learning rate $0.008$, penalty coefficient $0.1$, and $\delta_{\epsilon}=0.01$; FdeHBO uses inner learning rate $0.1$, outer learning rate $0.01$, $\delta_{\epsilon}=0.01$, and momentum $0.95$; ZOFO uses inner learning rate $0.1$, outer learning rate $0.02$, and perturbation scale $0.02$.

\begin{table}[t!]
\centering
\caption{Model-configuration hyperparameters for CIFAR-10 hyper-representation learning.}
\label{tab:hyper_representation_hyperparams}
\small
\setlength{\tabcolsep}{4pt}
\renewcommand{\arraystretch}{1.12}
\begin{threeparttable}
\begin{tabular}{cccccc}
\toprule
Parameters & Input size & Stem width & Stage widths & Classes & Layers \\
\midrule
$\approx 11.2$M & $32\times32$ & $64$ & $64/128/256/512$ & $10$ & $18$ \\
\bottomrule
\end{tabular}
\end{threeparttable}
\end{table}

Figure~\ref{fig:hyper_representation_appendix} and Table~\ref{tab:hyper_representation_appendix} show that \ours with experimental rank ratio $r/n=0.25$ obtains $78.03\%$ final test accuracy and $78.28\%$ best test accuracy, closely matching MA-SOBA ($78.57\%$ final, $78.87\%$ best) and substantially outperforming FdeHBO, Penalty, and ZOFO. The rank-ratio ablation shows that $r/n=0.50$ reaches $79.22\%$, while the more aggressive $r/n=0.10$ still maintains $77.68\%$ final accuracy.

\begin{figure}[t!]
\centering
\begin{subfigure}[t]{0.45\linewidth}
\centering
\includegraphics[width=\linewidth]{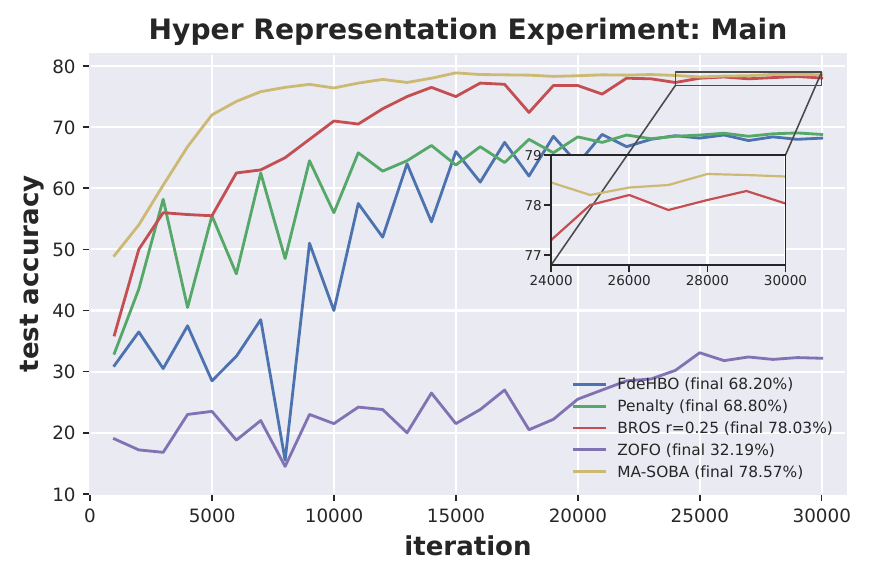}
\caption{Main comparison.}
\end{subfigure}
\hfill
\begin{subfigure}[t]{0.45\linewidth}
\centering
\includegraphics[width=\linewidth]{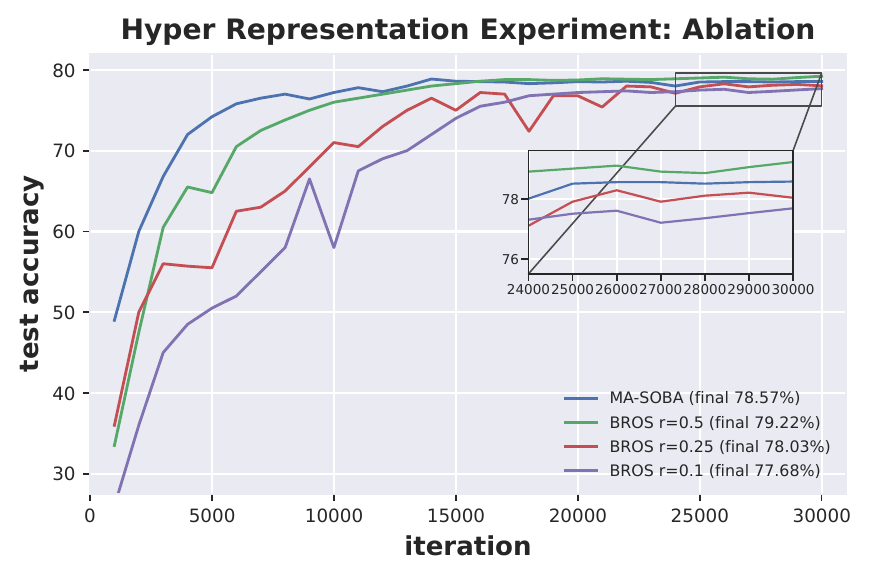}
\caption{Rank-ratio ablation.}
\end{subfigure}
\caption{\textbf{Hyper-representation learning.} CIFAR-10 test accuracy for main baselines and \ours rank-ratio ablations.}
\label{fig:hyper_representation_appendix}
\end{figure}

\begin{table}[t!]
\centering
\caption{CIFAR-10 hyper-representation results.}
\label{tab:hyper_representation_appendix}
\small
\setlength{\tabcolsep}{6pt}
\renewcommand{\arraystretch}{1.18}
\begin{tabular}{lcccc}
\toprule
Method & Final step & Final test acc. & Best test acc. & Best step \\
\midrule
FdeHBO~\citep{yang2023achieving} & $30000$ & $68.20\%$ & $68.71\%$ & $26000$ \\
Penalty~\citep{mehra2021penalty} & $30000$ & $68.80\%$ & $69.05\%$ & $29500$ \\
MA-SOBA~\citep{chen2024optimal} & $30000$ & $78.57\%$ & $78.87\%$ & $15000$ \\
ZOFO~\citep{shirkavand2025bilevel} & $30000$ & $32.19\%$ & $33.09\%$ & $25000$ \\
\rowcolor{boxbg}
\ourslast & $30000$ & $78.03\%$ & $78.28\%$ & $29000$ \\
\bottomrule
\end{tabular}
\end{table}

\subsection{CIFAR-100 ViT sample reweighting}
We also evaluate a larger bilevel sample-reweighting task on CIFAR-100 with $20\%$ label noise. The original $50$k training images are split into $45$k noisy training examples and $5$k clean validation examples; the test set remains clean. The lower model is a small ViT trained from scratch with patch size $4$, convolutional stem, embedding dimension $256$, depth $8$, and $4$ heads. The upper-level objective selects sample weights using the clean validation split.
All ViT runs are conducted on a single NVIDIA RTX PRO 6000 Blackwell GPU ($96$ GB).

All methods use the same AdamW training recipe with batch size $512$, model learning rate $10^{-3}$, upper learning rate $10^{-4}$, cosine decay, $1500$ warmup steps, weight decay $0.08$, and standard image augmentations including random crop, horizontal flip, RandAugment, random erasing, mixup/cutmix, and label smoothing. Low-rank methods use experimental rank ratio $r/n=0.25$. Evaluation is performed every $500$ steps. Model selection and early stopping use validation accuracy, with maximum $50000$ steps, minimum $12000$ steps, patience $6000$, and minimum improvement $0.001$.

\begin{table}[t!]
\centering
\caption{Model-configuration hyperparameters for CIFAR-100 ViT sample reweighting.}
\label{tab:vit_hyperparams}
\small
\setlength{\tabcolsep}{4pt}
\renewcommand{\arraystretch}{1.12}
\begin{threeparttable}
\begin{tabular}{cccccc}
\toprule
Parameters & Patch size & Embed dim. & MLP dim. & Heads & Depth \\
\midrule
$\approx 6.4$M & $4$ & $256$ & $1024$ & $4$ & $8$ \\
\bottomrule
\end{tabular}
\end{threeparttable}
\end{table}

Table~\ref{tab:vit_appendix} shows that \ours attains the best validation and test accuracy while using much less peak memory than full-space hypergradient baselines. In particular, \ours reaches $49.88\%$ best validation accuracy and $49.05\%$ test accuracy at the best-validation checkpoint, while using $6.42$ GB peak memory compared with more than $23$ GB for MA-SOBA, FdeHBO, and Penalty. ZOFO uses less memory than full-space baselines, but its test accuracy is substantially lower than \ourslast.

\begin{table}[t!]
\centering
\caption{CIFAR-100 ViT sample-reweighting results.}
\label{tab:vit_appendix}
\small
\setlength{\tabcolsep}{4pt}
\renewcommand{\arraystretch}{1.12}
\begin{tabular}{lccccc}
\toprule
Method & Best step & Best val. acc. & Test acc. at best val. & Peak memory & Memory vs. MA-SOBA \\
\midrule
Penalty~\citep{mehra2021penalty} & $26000$ & $37.56\%$ & $37.04\%$ & $23.59$ GB & $99.9\%$ \\
MA-SOBA~\citep{chen2024optimal} & $4500$ & $47.36\%$ & $48.07\%$ & $23.62$ GB & $100.0\%$ \\
FdeHBO~\citep{yang2023achieving} & $4500$ & $48.04\%$ & $47.76\%$ & $23.16$ GB & $98.1\%$ \\
ZOFO~\citep{shirkavand2025bilevel} & $17000$ & $42.68\%$ & $41.99\%$ & $11.20$ GB & $47.4\%$ \\
\rowcolor{boxbg}
\ourslast & $6000$ & $49.88\%$ & $49.05\%$ & $6.42$ GB & $27.2\%$ \\
\bottomrule
\end{tabular}
\end{table}

\subsection{Licenses of used datasets and models}
\label{subsec:asset_licenses}

We use publicly available datasets, tokenizer assets, and model architectures for research experiments. We summarize their sources and associated licenses or terms below:
\begin{itemize}
    \item \textbf{MNIST}~\citep{lecun2010mnist}: The MNIST dataset is publicly available from the original MNIST database at \url{http://yann.lecun.com/exdb/mnist/}. The original public page does not explicitly state a separate dataset license; we use MNIST only as a standard research benchmark for hyper-data cleaning and do not redistribute it.
    \item \textbf{CIFAR-10 and CIFAR-100}~\citep{krizhevsky2009learning}: The CIFAR datasets are publicly available from the University of Toronto CIFAR page at \url{https://www.cs.toronto.edu/~kriz/cifar.html}. The UCI Machine Learning Repository entry for CIFAR-10 lists the Creative Commons Attribution 4.0 International (CC BY 4.0) license; the original CIFAR-100 source page does not state a separate explicit license. We use both datasets only as standard research benchmarks and do not redistribute them.
    \item \textbf{Copyright-free Pile domains}~\citep{gao2020pile}: The data-mixture experiment follows the copyright-free Pile-domain setting used in DoReMi~\citep{xie2023doremi}. We use the text domains only for research training and evaluation, do not redistribute the data, and follow the terms of the corresponding source components.
    \item \textbf{GPT-NeoX tokenizer}~\citep{black2022gpt}: We use the tokenizer associated with EleutherAI/gpt-neox-20b, available from the \href{https://huggingface.co/EleutherAI/gpt-neox-20b}{Hugging Face model card}. The model card lists the tokenizer and model assets under the Apache License 2.0.
    \item \textbf{ResNet and ViT architectures}~\citep{he2016deep,dosovitskiy2020image}: The CIFAR-10 and CIFAR-100 experiments train ResNet-18 and small ViT models from scratch. We do not use or release pretrained ResNet or ViT checkpoints.
\end{itemize}

\section{Broader impacts}
\label{app:broader_impacts}

This work focuses on improving the memory efficiency of stochastic bilevel optimization for neural and language-model training. By reducing the lower-level and auxiliary memory footprint of hypergradient computation, \ours may help run bilevel tuning, data-mixture learning, and sample-reweighting experiments on devices with lower memory capacity, thereby reducing resource consumption and carbon emissions associated with large-scale training. These efficiency gains may also make bilevel optimization techniques more accessible to smaller institutions and research teams operating under constrained compute budgets. Aiming to advance the field of bilevel optimization, we do not identify additional societal consequences that require further discussion here.

\end{document}